\documentclass[11pt]{article}
\usepackage[a4paper,top=3cm,bottom=2cm,left=3cm,right=3cm,marginparwidth=1.75cm]{geometry}
\usepackage{xcolor}
\usepackage[colorlinks = true,
            linkcolor = blue,
            urlcolor  = blue,
            citecolor = blue,
            anchorcolor = blue]{hyperref}
\usepackage[round]{natbib}

\usepackage[normalem]{ulem}
\usepackage[none]{hyphenat}
\usepackage{breakcites}
\usepackage[utf8]{inputenc} % allow utf-8 input
\usepackage[T1]{fontenc}    % use 8-bit T1 fonts
\usepackage{hyperref}       % hyperlinks
\usepackage{url}            % simple URL typesetting
\usepackage{booktabs}       % professional-quality tables
\usepackage{amsfonts}       % blackboard math symbols
\usepackage{nicefrac}       % compact symbols for 1/2, etc.
\usepackage{microtype}      % microtypography
\usepackage{lipsum}
\usepackage[linesnumbered,ruled,vlined]{algorithm2e}

\SetCommentSty{mycommfont}
\SetKwInput{KwInput}{Input}                % Set the Input
\SetKwInput{KwOutput}{Output}  
\SetKwInput{KwInitialization}{Initialization}
\usepackage{float}
\usepackage{caption}
\usepackage{subcaption}
\usepackage{mathtools}
\usepackage{soul}
\mathtoolsset{showonlyrefs}
\usepackage{enumitem}
\usepackage{amssymb}
\usepackage{amsthm}
\usepackage{tikz}
\usepackage{bm}
\usetikzlibrary{arrows}
\usepackage{dsfont}
\usepackage{graphicx, graphics}
\usepackage{wrapfig}
\usepackage{thmtools,thm-restate}

\newcommand{\doublearrow}{{~\dashleftarrow \hspace{-5.5mm} \dashrightarrow~}}
\newcommand{\myrightarrow}{\xrightarrow{\hspace*{2.5mm}}}
\newcommand{\myleftarrow}{\xleftarrow{\hspace*{2.5mm}}}
\newcommand{\doubleplusrightarrow}{\:
	\smash{
	\underset{
		{\raisebox{0.75ex}{\smash{$\myrightarrow$}}}%
	}
	{
		{\raisebox{0.5ex}{\smash{$\doublearrow$}}}
	}
	}
	\:}

\newtheorem{assumption}{Assumption}
\newtheorem{definition}{Definition}
\newtheorem{corollary}{Corollary}
\newtheorem{remark}{Remark}

\newcommand{\Reals}{\mathbb{R}} % Reals
\newcommand{\Probability}{\mathbb{P}}
\newcommand{\Indicator}{\mathds{1}}
\newcommand{\Expectation}{\mathbb{E}}
\DeclareMathAlphabet{\mathbsf}{OT1}{cmss}{bx}{n}% bold sans serif
\DeclareMathAlphabet{\mathssf}{OT1}{cmss}{m}{sl}% slanted sans serif
\newcommand{\rvx}{{\mathssf{x}}}	% random variable x
\newcommand{\rvu}{{\mathssf{u}}}	% random variable u
\newcommand{\rvv}{{\mathssf{v}}}	% random variable v
\newcommand{\rvw}{{\mathssf{w}}}	% random variable w
\newcommand{\rvy}{{\mathssf{y}}}	% random variable y
\newcommand{\rvt}{{\mathssf{t}}}	% random variable t
\newcommand{\rve}{{\mathssf{e}}}	% random variable e
\newcommand{\rvs}{{\mathssf{s}}}	% random variable s
\newcommand{\rvbs}{{\mathbsf{s}}} % random vector s
\newcommand{\rvbx}{{\mathbsf{x}}} % random vector x
\newcommand{\rvbv}{{\mathbsf{v}}} % random vector v
\newcommand{\rvbw}{{\mathbsf{w}}} % random vector w
\newcommand{\rvbz}{{\mathbsf{z}}} % random vector z
\newcommand{\svbs}{{\boldsymbol{s}}} % deterministic vector s
\newcommand{\svbx}{{\boldsymbol{x}}} % deterministic vector x
\newcommand{\svbz}{{\boldsymbol{z}}} % deterministic vector z
\newcommand{\cV}{\mathcal{V}} % calligraphic V
\newcommand{\cG}{\mathcal{G}} % calligraphic G
\newcommand{\cN}{\mathcal{N}} % calligraphic N
\newcommand{\cS}{\mathcal{S}} % calligraphic S
\newcommand{\cP}{\mathcal{P}} % calligraphic P
\newcommand{\cX}{\mathcal{X}} % calligraphic X
\newcommand{\cM}{\mathcal{M}} % calligraphic M
\newcommand{\cW}{\mathcal{W}} % calligraphic W
\newcommand{\cU}{\mathcal{U}} % calligraphic U
\newcommand{\cZ}{\mathcal{Z}} % calligraphic Z
\newcommand{\pao}{\pi^{(o)}} % observed parents
\newcommand{\pau}{\pi^{(u)}} % unobserved parents
\newcommand{\pa}{\pi} % parents
\newcommand{\Uniform}{\text{Uniform}} % uniform
\newcommand{\Bernoulli}{\text{Bernoulli}} % uniform
\newcommand{\Sigmoid}{\text{Sigmoid}} %
\newcommand{\Softmax}{\text{Softmax}} %
\newcommand{\dsep}{{\perp \!\!\! \perp}_d~}
\newcommand{\indep}{\perp_p}
\newcommand{\notdsep}{{\not\!\perp\!\!\!\perp}_d}
\newcommand{\td}{\tilde{d}} % calligraphic V
\newcommand{\btheta}{\boldsymbol{\theta}} 

\title{Finding Valid Adjustments under Non-ignorability with Minimal DAG Knowledge 
}

\author{%
	Abhin Shah
	\\
	MIT\\
	\texttt{abhin@mit.edu} \\
	\and
	Karthikeyan Shanmugam \\
	IBM Research \\
	\texttt{karthikeyan.shanmugam2@ibm.com}
	\and
	Kartik Ahuja
	\\
	Mila \\
	\texttt{kartik.ahuja@mila.quebec}
}
\date{}
\begin{document}
\maketitle
\begin{abstract}
    Treatment effect estimation from observational data is a fundamental problem in causal inference. There are two very different schools of thought that have tackled this problem. On the one hand, the Pearlian framework commonly assumes structural knowledge (provided by an expert) in the form of directed acyclic graphs and provides graphical criteria such as the back-door criterion to identify the valid adjustment sets. On the other hand, the potential outcomes (PO) framework commonly assumes that all the observed features satisfy ignorability (i.e., no hidden confounding), which in general is untestable. In prior works that attempted to bridge these frameworks, there is an observational criteria to identify an \textit{anchor variable} and if a subset of covariates (not involving the anchor variable) passes a suitable conditional independence criteria, then that subset is a valid back-door. Our main result strengthens these prior results by showing that under a different expert-driven structural knowledge --- that one variable is a direct causal parent of the treatment variable --- remarkably, testing for subsets (not involving the known parent variable) that are valid back-doors is \textit{equivalent} to an invariance test. Importantly, we also cover the non-trivial case where the entire set of observed features is not ignorable (generalizing the PO framework) without requiring the knowledge of all the parents of the treatment variable. Our key technical idea involves generation of a synthetic sub-sampling (or environment) variable that is a function of the known parent variable. In addition to designing an invariance test, this sub-sampling variable allows us to leverage Invariant Risk Minimization, and thus, connects finding valid adjustments (in non-ignorable observational settings) to  representation learning. We demonstrate the effectiveness and tradeoffs of these approaches on a variety of synthetic datasets as well as real causal effect estimation benchmarks.
\end{abstract}
\section{Introduction}
\label{section:intro}
Estimating the impact of a treatment (or an action) is fundamental  to many scientific disciplines (e.g., economics \citep{imbens2015causal}, medicine \citep{shalit2017estimating, alaa2017bayesian}, policy making \citep{Lalonde1986, Smith2005}). In most of these fields, randomized clinical trials (RCT) is a common practice for estimating treatment effects.  However, conducting a RCT could be unethical or costly, and we may only have access to observational data. Estimating treatment effects with only  observational data is a challenging task and is of central interest to causal inference researchers. 

A fundamental question in treatment effect estimation is: \textit{Which subset of observed features should be adjusted for while estimating treatment effect from  observational data?}  Simpson's paradox \citep{pearl2014comment}, which is a phenonmenon that is observed in many real-life studies on treatment effect estimation, underscores the value of selecting appropriate features for treatment effect estimation. Over the years, two schools of thoughts have formed on how to tackle treatment effect estimation. The Pearlian framework \citep{Pearl2009} commonly assumes that an expert provides us with the causal generative model in the form of a directed acyclic graph (DAG) that relates unobserved exogenous variables to observed features, treatment variable, and outcome variables. With the knowledge of the DAG available, the framework provides different graphical criteria (e.g., back-door criterion \citep{Pearl1993}, front-door criterion \citep{pearl1995causal}) that answers whether a subset is valid for adjustment. The DAG framework allows for the existence of confounders -- unobserved variables that affect multiple observed variables.  The potential outcomes (PO) framework \citep{rubin1974estimating} makes an untestable assumption called \textit{ignorability} -- the assumption (in a rough sense) requires potential outcomes under different treatments be independent of the treatment conditioned on all (or a known subset of) observed features. In other words, ignorability implies that a subset of observed features is a valid adjustment and is known. The PO framework provides various techniques (e.g., inverse propensity weighing \citep{swaminathan2016off}, doubly robust estimation \citep{Funk2011}) for treatment effect estimation under ignorability. One can view the Pearlian DAG framework as providing graphical criteria implying ignorability of certain subsets. 

In summary, the Pearlian framework requires the knowledge of the DAG and the PO framework assumes ignorability with respect to the observed features. Motivated by the limitations of both of these frameworks we ask: \emph{can we significantly reduce the structural knowledge required about the DAG under non-ignorability of observed features and yet find valid adjustment sets}? 

\subsection{Our Contributions}
We assume the following minimal expert-driven local structural knowledge: a known observed feature is a direct causal parent of the treatment. Given this, we propose a simple \textit{invariance} test, and show that it is equivalent to testing if a subset not involving the known parent satisfies the back-door criterion (without requiring ignorability) when the features are pre-treatment. To design our invariance test, we use the known parent to create `fake environment variables'. We then test for invariance (across these environments) of the outcome conditioned on subsets of observed features (not containing the known parent) and the treatment. If a subset passes this invariance test, then it satisfies the back-door criterion (and therefore is a valid adjustment set) allowing for treatment effect estimation. Crucially, our result also goes in the other direction, i.e., if there exists a set (not containing the known parent) that satisfies the back-door criterion, then it will pass our invariance test. 

We propose two algorithms based on this equivalence result to identify valid adjustments. In the first algorithm, we use a subset based search procedure that exploits conditional independence (CI) testing to check our invariance criterion. As is standard with any subset based search approach, the application of our first algorithm is limited to small dimensional datasets. To overcome this, in our second algorithm, we leverage Invariant Risk Minimization (IRM) \citep{arjovsky2019invariant}, originally proposed to learn causal representations for out-of-distribution generalization, to act as a continuous optimization based scalable approximation for CI testing. We demonstrate the  effectiveness of our algorithms in treatment effect estimation  on both synthetic and benchmark datasets. In particular, we also show that IRM based algorithm  scales well with dimensions in contrast to the subset search based approach. The source code of our implementation is available at \url{https://github.com/Abhin02/invariance-via-subsampling}.

\subsection{Outline}
In Section \ref{section:related_works}, we look at some related literature. In Section \ref{section:problem_formulation}, we formulate our problem and state our assumptions. We provide our main results in Section \ref{section:main_results} and our algorithms in Section \ref{section:algorithm}. In Section \ref{section:experiments}, we demonstrate our experimental findings. In Section \ref{section:conclusion_limitations_future_work}, we conclude, discuss a few limitations and directions for future work. See the Appendix for its organization.
\section{Related Work}
\label{section:related_works}
Next, we provide an overview of related work that directly concerns finding valid adjustment in treatment effect estimation. See Appendix \ref{appendix:related_work} for an overview of prior work related to potential outcomes and usage of representation learning to debias treatment effect.

% \noindent{\bf Finding valid adjustment with global knowledge.} 
\subsection{Finding valid adjustment with global knowledge}
Finding valid adjustment sets for general interventional queries has been extensively studied in the Pearlian framework \citep{TianP2002}. Given the complete knowledge of the DAG, a sound and complete algorithm to find valid adjustments was proposed by \cite{shpitser2008complete}. When only the observational equivalence class is known, i.e., partial ancestral graph or PAG \citep{zhang2008causal},  \cite{perkovic2018complete} provided a sound and complete algorithm for finding valid adjustments. \cite{vanderweele2011new} showed that if a valid adjustment set exists amongst the observed features, then the union of all observed parents of outcome and all observed parents of treatment is also a valid adjustment set. However, they required global knowledge i.e., information about \textit{every} observed feature while our work requires knowledge about \textit{only one observed parent of treatment} i.e., local knowledge.

% \noindent{\bf Finding valid adjustment with local knowledge.}
\subsection{Finding valid adjustment with local knowledge.}
As opposed to the works described in the previous paragraph, another line of work (e.g., \cite{entner2013data, cheng2020towards, gultchin2020differentiable}) focused on finding valid adjustment sets by exploiting local knowledge of the DAG. In \cite{entner2013data}, a two-step approach was proposed. First, \textit{an anchor variable} is characterized by an observational criteria that is testable. Next, a conditional independence test is performed on the subsets not involving the anchor variable to find the valid adjustment set. In the reverse direction, if a valid adjustment set exists that does not contain the anchor variable, their test is shown to succeed only if the anchor variable has no observed or unobserved parents. As a result, even if it were possible to carry out consistent treatment effect estimation based on adjustment sets not involving the anchor, their procedure need not necessarily enable it. In contrast, in these settings, under the assumption that the anchor variable (direct causal parent of the treatment) is specified by the expert, our invariance test enables consistent treatment effect estimation. On the other hand, in \cite{cheng2020towards}, the anchor variable is characterized by topological properties of the PAG. We provide examples where our procedure can correctly declare that consistent treatment effect is not possible but they cannot.

Following \cite{entner2013data}, \cite{gultchin2020differentiable} proposed a fully-differentiable optimization framework to find a representation of the features that passes the conditional independence criteria analogous to \cite{entner2013data}. While their approach avoids the brute-force subset search required by \cite{entner2013data}, their approach is as limited in the reverse direction as \cite{entner2013data}. Further, their continuous optimization framework assumes the outcome is binary or the whole system (including the treatment) is linear Gaussian. Additionally, they use partial correlation as a proxy for conditional independence. This proxy is correct when the underlying distribution is Gaussian and in the worst-case constrains only the second moment. In other words, their framework doesn't provide formal guarantees even if one of the variables (e.g. treatment) isn't Gaussian. In contrast, our approach doesn't make these assumptions and is more general.

% \noindent{\bf Invariance principle.} 
\subsection{Invariance principle.} 
The invariance principle (also known as modularity condition) is fundamental to causal bayesian networks \citep{bareinboim2012local,scholkopf2019causality}. \cite{arjovsky2019invariant} proposed a continuous optimization framework called invariant risk minimization (IRM), to search for causal representations which satisfy invariance principle, that achieves out-of-distribution generalization. A recent line of work (e.g., \cite{shi2020invariant, shah2021treatment}) has focused on using IRM for treatment effect estimation. \cite{shi2020invariant} assumed (i) that there are no unmeasured confounders and (ii) access to interventional data is available (similar to IRM). We significantly differ from this as we allow unmeasured confounders and do not require interventional data -- we create artificial environments by sub-sampling observational data -- and leverage IRM to find valid adjustment sets that satisfy our criterion. On the other hand, while \cite{shah2021treatment} created environments artificially (similar to ours), their sub-sampling procedure lacks theoretical justification. Further, they focus primarily on the setting where there is little support overlap between the control and the treatment group, and lack formal guarantees on finding valid adjustment sets.
\section{Problem Formulation}
\label{section:problem_formulation}
\noindent{\bf Notations.} 
For a sequence of deterministic variables $s_1, \cdots , s_n$, we let $\svbs \coloneqq \{s_1, \cdots, s_n\}$. For a sequence of random variables $\rvs_1, \cdots , \rvs_n$, we let $\rvbs \coloneqq \{\rvs_1, \cdots, \rvs_n\}$. Let $\Indicator$ denote the indicator function. 

\subsection{Semi-Markovian Model, Effect Estimation, Valid Adjustment}
\label{subsection:semi-markovian}
Consider a causal effect estimation task with $\rvbx$ as the feature set, $\rvt$ as the observed treatment variable and $\rvy$ as the observed potential outcome. For the ease of exposition, we focus on binary $\rvt$. However, our results apply to non-binary $\rvt$ as well. Further, while we consider discrete $\rvbx$ and $\rvy$, our framework applies equally to continuous or mixed $\rvbx$ and $\rvy$. Let $\cG$ denote the underlying DAG over the set of vertices $\cW \coloneqq \{\rvbx,\rvt,\rvy\}$. For any variable $\rvw \in \cW$, let $\pa(\rvw)$ denote the set of parents of $\rvw$ i.e., $\pa(\rvw) = \{\rvw_1:\rvw_1 \myrightarrow \rvw\}$.

To estimate the causal effect of treatment $\rvt$ on outcome $\rvy$, a Markovian causal model requires the specification of the following three elements : $(a)$ $\cW$ -- the set of variables, $(b)$ $\cG$ -- the DAG  over the set of vertices $\cW$, and $(c)$ $\Probability(\rvw|\pa(\rvw))$ -- the conditional probability of $\rvw$ given its parents $\pa(\rvw)$ for every $\rvw \in \cW$. Given the DAG $\cG$, the causal effect of $\rvt$ on $\rvy$ can be estimated from observational data since $\Probability(\rvw|\pa(\rvw))$ is estimable from observational data whenever $\cW$ is observed. 

Our ability to estimate the causal effect of $\rvt$ on $\rvy$ from observational data is severely curtailed when some variables in a Markovian causal model are unobserved. Let $\rvbx^{(o)} \subseteq \rvbx$ be the subset of features that are observed and $\rvbx^{(u)} = \rvbx \setminus \rvbx^{(o)}$ be the subset of features that are unobserved. For any variable $\rvw \in \cW$, let $\pao(\rvw) \subseteq \pa(\rvw)$ denote the set of parents of $\rvw$ that are observed and let $\pau(\rvw) \coloneqq \pa(\rvw) \setminus \pao(\rvw)$ denote the set of parents of $\rvw$ that are unobserved. We focus on the semi-Markovian causal model \citep{TianP2002}, defined below, since any causal model with unobserved variables can be mapped to a semi-Markovian causal model while preserving the dependencies between the variables \citep{verma1990causal, acharya2018learning}.

\begin{definition} (Semi-Markovian Causal Model.)
A semi-Markovian causal model $\cM$ is a tuple $\big\langle \cV, \cU, \cG, \Probability(\rvv| \pao(\rvv), \pau(\rvv)), \Probability(\cU) \big\rangle$ where:
\begin{enumerate}[leftmargin=*,topsep=-0pt,itemsep=-3pt]
    \item  $\cV$ is the set of observed variables, i.e. $\cV = \{\rvbx^{(o)}, \rvt, \rvy \}$,
    \item  $\cU$ is the set of unobserved (or exogenous) features, i.e. $\cU \coloneqq \cW \setminus \cV = \rvbx^{(u)}$, 
    \item  $\cG$ is the DAG over the set of vertices $\cW$ such that each member in $\cU$ has no parents and at-most two children. 
    \item $\Probability(\rvv| \pao(\rvv), \pau(\rvv)) ~ \forall \rvv \in \cV$ is the set of unobserved conditional distributions of the observed variables, and
    \item  $\Probability(\cU)$ is the unobserved joint distribution over the unobserved features.
\end{enumerate}
\end{definition}

In a semi-Markovian model, unobserved variables with only one or no children are omitted entirely. See Figure \ref{fig:Gtoy_with_u} for a toy example of a semi-Markovian model with $\cV = \{\rvx_1, \rvx_2, \rvx_3, \rvt, \rvy\}$, $\cU = \{\rvu_1, \rvu_2, \rvu_3, \rvu_4\}$, and $\cG = \cG^{toy}$.
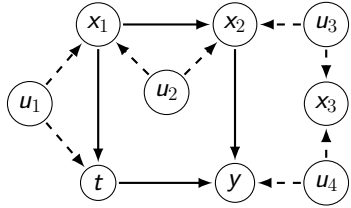
\begin{figure}[h]
 \centering
 \begin{tikzpicture}[scale=0.6,every node/.style={transform shape}, > = latex, shorten > = 1pt, shorten < = 1pt]
	\node[shape=circle,draw=black](x1) at (0,2) {\LARGE$\rvx_1$};
	\node[shape=circle,draw=black](u1) at (-1.5,0.25) {\LARGE$\rvu_1$};
	\node[shape=circle,draw=black](x2) at (3,2) {\LARGE$\rvx_2$};
	\node[shape=circle,draw=black](u2) at (1.5,0.5) {\LARGE$\rvu_2$};
	\node[shape=circle,draw=black](x3) at (5,0.25) {\LARGE$\rvx_3$};
	\node[shape=circle,draw=black](u3) at (5.0,2) {\LARGE$\rvu_3$};
	\node[shape=circle,draw=black](u4) at (5.0,-1.5) {\LARGE$\rvu_4$};
	\node[shape=circle,draw=black](y) at (3,-1.5) {\LARGE$\rvy$};
	\node[shape=circle,draw=black](t) at (0,-1.5) {\LARGE$\rvt$};
	\path[style=thick][->](x1) edge (x2);
	\path[style=thick][->](x1) edge (t);
	\path[style=thick][->](u2) [dashed] edge (x1);
	\path[style=thick][->](u2) [dashed] edge (x2);
	\path[style=thick][->](t) edge (y);		
	\path[style=thick][->](x2) edge (y);		
	\path[style=thick][->](u1) [dashed] edge (t);
	\path[style=thick][->](u1) [dashed] edge (x1);
	\path[style=thick][->](u3) [dashed] edge (x2);
	\path[style=thick][->](u3) [dashed] edge (x3);
	\path[style=thick][->](u4) [dashed] edge (x3);
	\path[style=thick][->](u4) [dashed] edge (y);
	\end{tikzpicture}
 \caption{The toy example $\cG^{toy}$.}
 \label{fig:Gtoy_with_u}
\end{figure}%

In observational data, we observe samples of $\cV$ from $\Probability(\cV)$ which is related to the semi-Markovian model $\cM$ by the following marginalization \citep{TianP2002}: 
$$ \Probability(\cV) = \Expectation_{\rvbx^{(u)}} \big[ \prod_{\rvv \in \cV} \Probability( \rvv| \pao(\rvv), \pau(\rvv)) \big].$$ Next, we define the notion of causal effect using the do-operator.
\begin{definition}\label{definition_causal_effect}
 (Causal Effect.)
 The causal effect of the treatment $\rvt$ on the outcome $\rvy$ is defined as
  \begin{align}
    \Probability(\rvy | do(\rvt = t)) =  \sum_{\rvt = t', \rvbx^{(o)} = \svbx^{(o)}} \Indicator_{t'=t} \Expectation_{\rvbx^{(u)}}  \nonumber \big[ \prod_{\rvv \in \cV \setminus \{\rvt\}} \Probability( \rvv | \pao(\rvv), \pau(\rvv) ) \big]
\end{align}
\end{definition}
The do-operator forces $\rvt$ to be $t$ in the causal model $\cM$, i.e the conditional factor $\Probability(\rvt = t'|\pao(\rvt),~ \pau(\rvt) )$ is replaced by the indicator $\Indicator_{\rvt=t'}$ and the resulting distribution is marginalized over all possible realizations of all observed random variables except $\rvy$. Next, we define average treatment effect and valid adjustment.
 
\begin{definition}\label{definition_average_treatment_effect}
(Average Treatment Effect.)
The average treatment effect (ATE) of a binary treatment  $\rvt$ on the outcome $\rvy$ is defined as ATE = $\Expectation[\rvy | do(\rvt = 1)] - \Expectation[\rvy | do(\rvt = 0)]$.
\end{definition}
 
\begin{definition}\label{definition_valid_adjustment}
(Valid Adjustment.)
A set of variables $\rvbz \subseteq \rvbx$ is said to be a \textit{valid adjustment} relative to the ordered pair of variables $(\rvt, \rvy)$ in the DAG $\cG$ if 
 $\Probability(\rvy | do(\rvt = t)) = \Expectation_{\rvbz} [\Probability(\rvy | \rvbz = \svbz, \rvt = t))].$
\end{definition}
If $\rvbz \subseteq \rvbx^{(o)}$ is a valid adjustment relative to $(\rvt, \rvy)$, then the ATE can be estimated from observational data by regressing the factual outcomes for the treated and the untreated sub-populations on $\rvbz$ i.e., ATE $ = \Expectation_{\rvbz} [ \Expectation_{\rvy} [ \rvy| \rvt=1, \rvbz] - \Expectation_{\rvy} [ \rvy| \rvt=0, \rvbz] ]$. 

For any variables $\rvw_1, \rvw_2 \in \cW$, and a set $\rvbw \subseteq \cW$, (a) let $\rvw_1 \indep \rvw_2 |\rvbz$ denote that $\rvw_1$ and $\rvw_2$ are conditionally independent given $\rvbz$ and (b) let $\rvw_1 \dsep \rvw_2 |\rvbw$ denote that $\rvw_1$ and $\rvw_2$ are d-separated by $\rvbw$ in $\cG$. For completeness, we provide the definition of d-separation in Appendix \ref{appendix:d_separation} as well as review potential outcomes (PO) framework  (\cite{imbens2015causal}), discuss ignorability and connect it with valid adjustment in Appendix \ref{appendix:po}.

\subsection{Back-door Criterion}
We now discuss the \textit{back-door} criterion  \citep{Pearl2016} -- a popular sufficient graphical criterion for finding valid adjustments i.e., any set satisfying the back-door criterion is a valid adjustment (\cite{Pearl1993}).
 
\begin{definition}\label{definition_back_door}
(Back-door criteria.)
A set of variables $\rvbz \subseteq \rvbx$ satisfies the back-door criterion relative to the ordered pair of variables $(\rvt, \rvy)$ in $\cG$ if no node in $\rvbz$ is a descendant of $\rvt$ and
$\rvbz$ blocks every path between $\rvt$ and $\rvy$ in $\cG$ that contains an arrow into $\rvt$.
\end{definition}

Often, $\cG$ is represented without explicitly showing elements of $\cU$ but, instead, using bi-directed edges (\cite{TianP2002}) to represent confounding effects of $\cU$. For example, Figure \ref{fig:Gtoy_with_u} uses bi-directed edges to represent unmeasured confounders (i.e., elements of $\cU$ that influence two variables in $\cV$) in the DAG $\cG^{toy}$.

\begin{definition}
(A Bi-directed Edge.)
A bi-directed edge between nodes $\rvv_1 \in \cV$ and $\rvv_2 \in \cV$ (i.e., $\rvv_1 \doublearrow \rvv_2$) represents the presence (in $\cG$) of a divergent path $\rvv_1 \dashleftarrow \rvu \dashrightarrow \rvv_2$ where $\rvu \in \cU$.
\end{definition}

In this work, we make the following structural assumption on the DAG $\cG$ under the semi-Markovian model $\cM$. This assumption is analogous to the common assumption that all observed features are pre-treatment variables. As an example, consider the DAG $\cG^{toy}$ in Figure \ref{fig:Gtoy_without_u} that satisfies this assumption. 

\begin{assumption}\label{assumption1}
Let the DAG $\cG$ be such that the treatment $\rvt$ has the outcome $\rvy$ as its only child. Further, the outcome $\rvy$ has no child.
\end{assumption}

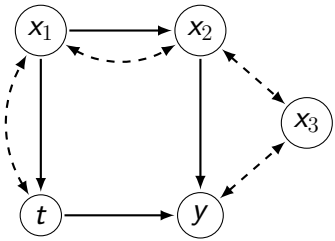
\begin{figure*}[h!]
\centering
	\begin{tikzpicture}[scale=0.7, every node/.style={transform shape}, > = latex, shorten > = 1pt, shorten < = 1pt]
	\node[shape=circle,draw=black](x1) at (0,2) {\LARGE$\rvx_1$};
	\node[shape=circle,draw=black](x2) at (3,2) {\LARGE$\rvx_2$};
	\node[shape=circle,draw=black](x3) at (5,0.25) {\LARGE$\rvx_3$};
	\node[shape=circle,draw=black](y) at (3,-1.5) {\LARGE$\rvy$};
	\node[shape=circle,draw=black](t) at (0,-1.5) {\LARGE$\rvt$};
	\path[style=thick][->](x1) edge (x2);
	\path[style=thick][->](x1) edge (t);
	\path[style=thick][<->, bend right](x1) [dashed] edge (x2);
	\path[style=thick][->](t) edge (y);		
	\path[style=thick][->](x2) edge (y);		
	\path[style=thick][<->, bend left](t) [dashed] edge (x1);
	\path[style=thick][<->](x2) [dashed] edge (x3);
	\path[style=thick][<->](y) [dashed] edge (x3);
	\end{tikzpicture}
  \caption{The toy example $\cG^{toy}$ with bi-directed edges.}
  \label{fig:Gtoy_without_u}
\end{figure*}
\section{Main Results}
\label{section:main_results}
In this section, we state our main results relating sub-sampling and invariance testing to the back-door criteria. First, we define the notions of sub-sampling and invariance. Next, we provide : (a) a sufficient d-separation condition (that can be realized by our invariance test under sub-sampling) for a class of back-door criteria (Theorem \ref{thm_sufficiency}) and (b) a necessary d-separation condition (that can be realized by our invariance test under sub-sampling) implied by a class of back-door criteria (Theorem \ref{thm_necessity}). Combining these, we show equivalence between an invariance based d-separation condition and a class of back-door criteria (Corollary \ref{corollary:backdoor_CI_equivalence}). Finally, we propose an algorithm to find all subsets of the observed features that satisfy the back-door criteria when all the parents of the treatment variable are known and observed (Appendix \ref{appendix:finding_back_doors}).\\

\noindent{\bf Sub-sampling.} We create a sub-sampling (or environment) variable $\rve$ from the observed distribution $\Probability(\cV)$. Formally, we use a specific observed variable $\rvx_{\rvt} \in \rvbx^{(o)}$ and a subset of the observed variables $\rvbv \subseteq \cV \setminus \{\rvx_t,\rvy\}$ to sub-sample  $\rve$ i.e., $\rve = f(\rvx_{\rvt},\rvbv,\eta)$ where $\eta$ is a noise variable independent of $\cW$, and $f$ is a function of $\rvx_{\rvt}, \rvbv$ and $\eta$. The choices of $\rvx_{\rvt}$ and $\rvbv$, which differ for the sufficient condition (Theorem \ref{thm_sufficiency}) and the necessary condition (Theorem \ref{thm_necessity}), are made clear in the respective theorem statements. We let the sub-sampling variable $\rve$ be discrete and think of the distinct values of $\rve$ as identities of distinct artificial environments created via sub-sampling. While the case where $\rve$ is continuous is similar in spirit, we postpone the nuances for a future work. Graphically, sub-sampling variable introduces a node $\rve$, an edge from $\rvx_{\rvt}$ to $\rve$ and edges from every $\rvv \in \rvbv$ to $\rve$ in the DAG $\cG$. For example, see Figure \ref{fig:G_toy_with_e} where $\rve$ is sub-sampled in toy example $\cG^{toy}$ with $\rvx_{\rvt} = \rvx_1$ and $\rvbv = \{\rvt\}$.\\

\noindent{\bf Invariance testing.} 
Our main results relate the back-door criterion to d-separation statements of the type $\rve \dsep \rvy | \rvbz$ for some $\rvbz \subseteq \cV \setminus \{\rvy\}$. While our goal is to infer sets satisfying the back-door criterion from observational data, such d-separation statements cannot be tested for from observational data. To tackle this, we propose the notion of invariance testing. An invariance test is a conditional independence test of the form $\rve \indep \rvy | \rvbz$ for some $\rvbz \subseteq \cV \setminus \{\rvy\}$ i.e., an invariance test tests if the sub-sampling variable is independent of the outcome conditioned on $\rvbz$ for some $\rvbz \subseteq \cV \setminus \{\rvy\}$.

For our results involving invariance testing, we require the following limited set of faithfulness assumptions to ensure invariance testing with $\rve$ is equivalent to d-separation statements involving $\rve$.
\begin{assumption}\label{assumption3}
 (Sub-sampling Faithfulness) If $ \rve \indep \rvy | \rvbz$, then $ \rve \dsep \rvy | \rvbz, ~\forall \rvbz \subseteq \cV \setminus \{\rvy\} $.
\end{assumption}
Thus, in effect, we create \textit{synthetic} environments and show that a class of back-door criterion either implies or is equivalent to a suitable invariance test. For our framework to work, we only require the knowledge of $\rvx_{\rvt}$ from an expert. This is in contrast to any detailed knowledge of the structure of the DAG $\cG$.\\

\noindent{\bf Sufficient condition.} Suppose an expert provides us with an observed feature that has a direct edge or a bi-directed edge to the treatment. Let $\rve$ be sub-sampled using this feature as $\rvx_{\rvt}$ and any $\rvbv \subseteq \cV \setminus \{\rvx_{\rvt},\rvy\}$.
The following result shows that any subset of the remaining observed features satisfying a d-separation involving $\rve$ (or an invariance test under Assumption \ref{assumption3}) also satisfies the back-door criterion. See Appendix \ref{appendix:proof_thm_sufficiency} for a proof. 
\begin{restatable}{theorem}{thmSufficiency}
\label{thm_sufficiency}
Let Assumption \ref{assumption1}  be satisfied. Consider any $\rvx_{\rvt} \in \rvbx^{(o)}$ that has a direct edge or a bi-directed edge to $\rvt$ 
i.e., either $\rvx_{\rvt} \myrightarrow \rvt$, $\rvx_{\rvt} \doublearrow \rvt$ or $\rvx_{\rvt} \doubleplusrightarrow \rvt$. 
Let $\rve$ be sub-sampled using $\rvx_{\rvt}$ and $\rvbv$ for any $\rvbv \subseteq \cV \setminus \{\rvx_{\rvt},\rvy\}$ i.e., $\rve = f(\rvx_{\rvt}, \rvbv,\eta)$.
Let $\rvbz \subseteq \rvbx^{(o)} \setminus \{\rvx_{\rvt}\}$. If $\rve$ is d-separated from $\rvy$ by $\rvbz$ and $\rvt$ in $\cG$ i.e., $\rve \dsep \rvy | \rvbz, \rvt$ in $\cG$, 
then $\rvbz$ satisfies the back-door criterion relative to $(\rvt, \rvy)$ in $\cG$.
\end{restatable}

\begin{remark}
A stronger result that subsumes Theorem \ref{thm_sufficiency} was proven in \cite{entner2013data}; we provide our theorem for clarity of exposition and completeness.
\end{remark}
\noindent{\bf Necessary condition.} Suppose an expert provides us with an observed feature that has a direct edge to the treatment. Let $\rve$ be sub-sampled using this variable as $\rvx_{\rvt}$ and any $\rvbv \subseteq \{\rvt\}$. 
The following result shows that any subset of the remaining observed features satisfying the back-door criterion satisfies a specific d-separation involving $\rve$ (as well as an invariance test).
See Appendix \ref{appendix:proof_thm_necessity} for a proof.

\begin{restatable}{theorem}{thmNecessity}
\label{thm_necessity}
Let Assumption \ref{assumption1}  be satisfied. Consider any $\rvx_{\rvt} \in \rvbx^{(o)}$ that has a direct edge to $\rvt$ i.e., $\rvx_{\rvt} \myrightarrow \rvt$ or $\rvx_{\rvt} \doubleplusrightarrow \rvt$.
Let $\rve$ be sub-sampled using $\rvx_{\rvt}$ and $\rvbv$ for any $\rvbv \subseteq \{\rvt\}$ i.e., $\rve = f(\rvx_{\rvt}, \rvbv, \eta)$.
Let $\rvbz \subseteq \rvbx^{(o)} \setminus \{\rvx_{\rvt}\}$. % and $\rvbx_{\cA} \neq \varnothing$. 
If $\rvbz$ satisfies the back-door criterion relative to $(\rvt, \rvy)$ in $\cG$, then $\rve$ is d-separated from $\rvy$ by $\rvbz$ and $\rvt$ in $\cG$ i.e., $\rve \dsep \rvy | \rvbz, \rvt$ in $\cG$.
\end{restatable}

\begin{remark}
Theorem \ref{thm_necessity} is useful to find out (some) sets that cannot be valid adjustments (see comparison with \cite{entner2013data} and \cite{gultchin2020differentiable} as well as comparison with \cite{cheng2020towards} below). Knowing whether a given set of features is valid for adjustment or not is crucial -- especially in healthcare and social sciences -- to avoid using decisions based on biased estimates from observational studies.
\end{remark}

\begin{remark}
We note that Theorem \ref{thm_necessity} requires $\rvx_{\rvt}$ to be a parent of $\rvt$ (i.e., a direct edge to $\rvt$) whereas Theorem \ref{thm_sufficiency} requires $\rvx_{\rvt}$ to be a parent of $\rvt$ or a spouse of $\rvt$ (i.e., a direct or a bi-directed edge to $\rvt$).
\end{remark}

\noindent{\bf Equivalence.} Suppose an expert provides us with a feature that has a direct edge to the treatment. Let $\rve$ be sub-sampled using this variable as $\rvx_{\rvt}$ and any $\rvbv \subseteq \{\rvt\}$. Combining Theorem \ref{thm_sufficiency} and Theorem \ref{thm_necessity}, we have the following Corollary showing equivalence of the back-door criterion and a specific d-separation involving $\rve$ (as well as an invariance test under Assumption \ref{assumption3}).
\begin{corollary}
\label{corollary:backdoor_CI_equivalence}
Let Assumption \ref{assumption1}  be satisfied. Consider any $\rvx_{\rvt} \in \rvbx^{(o)}$ that has a direct edge to $\rvt$ i.e., $\rvx_{\rvt} \myrightarrow \rvt$ or $\rvx_{\rvt} \doubleplusrightarrow \rvt$. Let $\rve$ be sub-sampled using $\rvx_{\rvt}$ and $\rvbv$ for any $\rvbv \subseteq \{\rvt\}$ i.e., $\rve = f(\rvx_{\rvt}, \rvbv, \eta)$. Let $\rvbz \subseteq \rvbx^{(o)} \setminus \{\rvx_{\rvt}\}$. Then, $\rvbz$ satisfies back-door criterion relative to the ordered pair of variables $(\rvt, \rvy)$ in $\cG$ if and only if $\rve$ is d-separated from $\rvy$ by $\rvbz$ and $\rvt$ in $\cG$ i.e., $\rve \dsep \rvy | \rvbz, \rvt$ in $\cG$.
\end{corollary}

\begin{remark}
\label{remark:m_bias}
While our framework captures a broad class of back-door criteria, it does not cover all the back-door criteria. For example, our method cannot capture that the M-bias problem \citep{liu2012implications, imbens2020potential} where no observed feature is a parent of the treatment. (see Appendix \ref{appendix:m_bias} for details). 
\end{remark}
 
\noindent{\bf An illustrative example.} First, we illustrate Corollary \ref{corollary:backdoor_CI_equivalence} with our toy example $\cG^{toy}$.
We let $\rvx_{\rvt} = \rvx_1$ and sub-sample $\rve$ using $\rvx_1$ and $\rvt$ (see Figure \ref{fig:G_toy_with_e}). For this example, $\rvbz \subseteq \{\rvx_2,\rvx_3\}$ i.e., $\rvbz \in \{\varnothing, \{\rvx_2\}, \{\rvx_3\}, \{\rvx_2, \rvx_3\}\}$. It is easy to verify that $\rvbz = \{\rvx_2\}$ satisfies the back-door criterion relative to $(\rvt, \rvy)$ in $\cG^{toy}$ but $\rvbz = \varnothing$, $\rvbz = \{\rvx_3\}$, and $\rvbz = \{\rvx_2,\rvx_3\}$ do not. 
Similarly, it is easy to verify that $\rve \dsep \rvy | \rvx_2, \rvt$ but $\rve \notdsep \rvy | \rvt$ , $\rve \notdsep \rvy | \rvx_2, \rvt$, and $\rve \notdsep \rvy | \rvx_2, \rvx_3 \rvt$ in $\cG^{toy}$. See Appendix \ref{appendix:illustrative_example} for an illustration tailored to Theorem \ref{thm_sufficiency}.

\begin{figure}[ht!]
\centering
\begin{subfigure}{.5\textwidth}
  \centering
	\begin{tikzpicture}[scale=0.7, every node/.style={transform shape}, > = latex, shorten > = 1pt, shorten < = 2pt]
	\node[shape=circle,draw=black](x1) at (0,2) {\LARGE${\rvx_1}$};
	\node[shape=circle,draw=black](x2) at (3,2) {\LARGE$\rvx_2$};
	\node[shape=circle,draw=black](x3) at (5,0.25) {\LARGE$\rvx_3$};
	\node[shape=circle,draw=black](e) at (-2,0.25) {\LARGE$\rve$};
	\node[shape=circle,draw=black](y) at (3,-1.5) {\LARGE$\rvy$};
	\node[shape=circle,draw=black](t) at (0,-1.5) {\LARGE$\rvt$};
	\path[style=thick][->](x1) edge (x2);
	\path[style=thick][->](x1) edge (t);
	\path[style=thick][<->, bend right](x1) [dashed] edge (x2);
	\path[style=thick][->](t) edge (y);		
	\path[style=thick][->](x2) edge (y);		
	\path[style=thick][<->, bend left](t) [dashed] edge (x1);
	\path[style=thick][<->](x2) [dashed] edge (x3);
	\path[style=thick][<->](y) [dashed] edge (x3);
	\path[style=thick, color = red][->](t) edge (e);
	\path[style=thick, color = red][->](x1) edge (e);
	\end{tikzpicture}
\caption{$\rve$ has been sub-sampled using $\rvx_1$ and $\rvt$.}
  \label{fig:G_toy_with_e}
\end{subfigure}%
\begin{subfigure}{.5\textwidth}
  \centering
	\begin{tikzpicture}[scale=0.7, every node/.style={transform shape}, > = latex, shorten > = 1pt, shorten < = 2pt]
	\node[shape=circle,draw=black](x1) at (0,2) {\LARGE$\rvx_1$};
	\node[shape=circle,draw=black](x2) at (3,2) {\LARGE$\rvx_2$};
	\node[shape=circle,draw=black](x3) at (5,0.25) {\LARGE$\rvx_3$};
	\node[shape=circle,draw=black](e) at (-2,0.25) {\LARGE$\rve$};
	\node[shape=circle,draw=black](y) at (3,-1.5) {\LARGE$\rvy$};
	\node[shape=circle,draw=black](t) at (0,-1.5) {\LARGE$\rvt$};
	\path[style=thick][->](x1) edge (x2);
	\path[style=thick][->](x1) edge (t);
	\path[style=thick][<->, bend right](x1) [dashed] edge (x2);
	\path[style=thick][->](t) edge (y);		
	\path[style=thick][->](x2) edge (y);		
	\path[style=thick][<->, bend left](t) [dashed] edge (x1);
	\path[style=thick][<->](x2) [dashed] edge (x3);
	\path[style=thick][<->](y) [dashed] edge (x3);
	\path[style=thick, color = red][->](t) edge (e);
	\path[style=thick, color = red][->](x3)  edge (e);
	\end{tikzpicture}
\caption{$\rve$ has been sub-sampled using $\rvx_3$ and $\rvt$.}
  \label{fig:G_toy_with_incorrect_e}
\end{subfigure}
\caption{The toy example $\cG^{toy}$ with different sub-sampling strategies.}
 \label{fig:illustrative_example}
\end{figure}
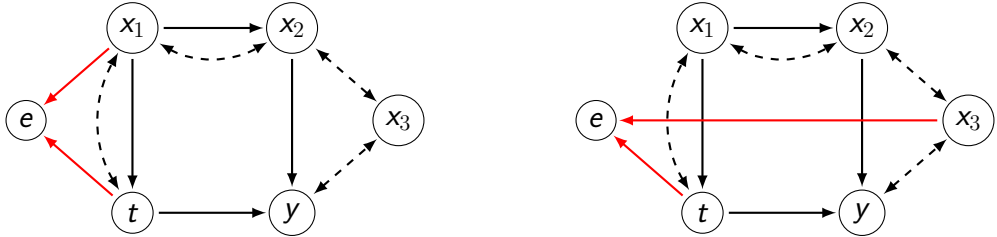

Next, we illustrate the significance of the criteria that qualifies $\rvx_{\rvt}$ to our results. In Figure \ref{fig:G_toy_with_e}, we let $\rvx_{\rvt} = \rvx_1$ ($\rvx_1$ has a direct or bi-directed edge to $\rvt$). Here, the d-separation $\rve \dsep \rvy | \rvx_2, \rvt$ holds implying that $\rvx_2$ satisfies back-door relative to $(\rvt,\rvy)$. In Figure \ref{fig:G_toy_with_incorrect_e}, we let $\rvx_{\rvt} = \rvx_3$ ($\rvx_3$ does not have a direct or bi-directed edge to $\rvt$). Here, the d-separation $\rve \dsep \rvy | \rvx_3, \rvt$ holds but $\rvx_3$ does not satisfy the back-door relative to $(\rvt,\rvy)$.\\

\noindent\textbf{Comparison with \cite{entner2013data} and \cite{gultchin2020differentiable}.} In \cite{entner2013data}, $\rvx_{a}$ is an anchor variable if it satisfies the observational criterion $\rvx_{a} \notdsep \rvy | \rvbz $ for some $\rvx_{a}$ and some $\rvbz$ not containing $\rvx_{a}$. Further, if the CI test implied by (the d-separation condition) $\rvx_{a} \dsep \rvy | \rvbz, \rvt $ is satisfied, then $\rvbz$ is shown to be a valid adjustment. While our sufficient condition in Theorem \ref{thm_sufficiency} is implied by this result, we provide a proof tailored to our condition and notations in Appendix \ref{appendix:proof_thm_sufficiency} for completeness.

However, the reverse direction in \cite{entner2013data} is as follows: if some $\rvbz$ (not containing $\rvx_{a}$) is a valid adjustment, then $\rvx_{a} \dsep \rvy \lvert \rvbz, \rvt $, only when $\rvx_{a}$ does not have any (observed or unobserved) parent in addition to satisfying the criteria for the anchor variable. Under our criterion, if $\rvx_\rvt$ is a direct parent of $\rvt$, the reverse direction can be shown in generality (our Theorem \ref{thm_necessity}).

As a concrete example, in $\cG^{toy}$, \cite{entner2013data} cannot conclude that $\emptyset, \{\rvx_3\}, \{\rvx_2,\rvx_3\}$ are not admissible i.e., not valid back-doors (because $\rvx_a=\rvx_1$ has an unobserved parent)  while our Theorem \ref{thm_necessity} can be used to conclude that. See the empirical comparison in Appendix \ref{app:entner}. Likewise, \cite{gultchin2020differentiable}, which build on \cite{entner2013data}, also cannot conclude that $\emptyset, \{\rvx_3\}, \{\rvx_2,\rvx_3\}$ are not valid adjustment sets in $\cG^{toy}$.\\

\noindent\textbf{Comparison with \cite{cheng2020towards}.} In \cite{cheng2020towards}, the anchor variable $\rvx_a$ is a \textit{COSO variable} i.e., either a parent or a spouse of the treatment but neither a parent or a spouse of the outcome in the true maximal ancestral graph (MAG). Our criteria for $\rvx_{\rvt}$ is different from this, and our result is neither implied by nor implies the result of \cite{cheng2020towards}.

Consider an example which is obtained by adding the edge $\rvx_1 \rightarrow \rvy $ to $\cG^{toy}$ in Figure \ref{fig:Gtoy_without_u}. The results of \cite{cheng2020towards} are not applicable since the anchor variable $\rvx_t$ is a parent of the outcome in the true DAG (and thereby in the MAG). However, $\rvx_t$ is a parent of the treatment (i.e., it satisfies our criteria), and our Theorem \ref{thm_necessity} is applicable. It can be used to conclude that $\emptyset, \{\rvx_2\}, \{\rvx_3\}$ and $\{\rvx_2,\rvx_3\}$ are not admissible sets. See the empirical comparison in Appendix \ref{app:cheng}.\\

\noindent{\bf Connections to Instrument Variable (IV).} While our anchor (i.e., $\rvx_{\rvt}$) may look similar to IV, this is not the case: (i) An IV needs to satisfy the exclusion restriction i.e., it needs to be d-separated from $\rvy$ in $\cG_{-\rvt}$ (i.e., the graph obtained by removing the edge from $\rvt$ to $\rvy$ in $\cG$). However, we do not require $\rvx_{\rvt}$ to be d-separated from $\rvy$ in $\mathcal{G}_{-\rvt}$. (ii) Unlike our work, IVs can \textit{only} provide bounds on ATE in non-parametric models; they provide perfect identifiability of ATE only in linear models \citep{balke1997bounds}.
\section{Algorithms}
\label{section:algorithm}
Our invariance criterion in Corollary \ref{corollary:backdoor_CI_equivalence} requires us to find  a $\rvbz$ such that  $\rve \indep \rvy | \rvbz, \rvt$.  In this section, given $n$ observational samples, we propose two algorithms that enable finding valid adjustment sets that pass our invariance criterion as well as use it to estimate ATE.
 
% \noindent\textbf{Testing and Subset Search:} 
\subsection{Testing and Subset Search}
First, we propose an algorithm (Algorithm \ref{alg:subset_search}) based on conditional independence (CI) testing and it works as follows. The algorithm takes the sub-sampling variable $\rve$ that is a function of $\rvx_{\rvt}$ ($\rve$ could also be a function of both $\rvx_{\rvt}$ and $\rvt$). The algorithm considers the set $\cX$ of all candidate adjustment sets that do not contain $\rvx_{\rvt}$. For every candidate adjustment set $\rvbz$ in $\cX$, our algorithm checks for CI between $\rve$ and $\rvy$ conditioned on $\rvbz$ and $\rvt$. If this CI holds, then $\rvbz$ satisfies the back-door criterion and is a valid adjustment set (see Corollary \ref{corollary:backdoor_CI_equivalence} and Assumption \ref{assumption3}). The ATE estimated by our algorithm is the average of the ATE estimated by regressing on such valid adjustment sets. On actual datasets, we use the following criterion as acceptance for CI: a p-value threshold $p_{value}$ is used to check if the p-value returned by the CI tester is greater than this threshold. We use the RCoT CI tester (see Appendix \ref{appendix:CI}).

Similar to \cite{entner2013data}, the computational complexity of Algorithm \ref{alg:subset_search} grows exponentially in the dimensionality of $\rvbx^{(o)}$. This makes it impractical for high dimensional settings.

\begin{algorithm}
\KwInput{ $n, n_r, \rvt, \rvy, \rve, \cX, p_{value}$}
\KwOutput{$\text{ATE}(\cX)$}
\KwInitialization{$\text{ATE}(\cX) = 0, c_1 = 0$}
\For(\tcp*[h]{Use a different train-test split in each run}) {$r = 1,\cdots,n_r$} 
{
    $c_2 = 0$; $\text{ATE}_{\text{d}} = 0$;\\
    \For {$\rvbz \in \cX$}
    {
    \If{$\text{CI}(\rve \indep \rvy | \rvbz, \rvt) > p_{value}$}
        {
        $c_2 = c_2 + 1$; $\text{ATE}_{\text{d}} = \text{ATE}_{\text{d}} + \frac{1}{n} \sum_{i=1}^{n} (\Expectation[\rvy | \rvbz = \svbz^{(i)}, \rvt = 1] - \Expectation[\rvy | \rvbz = \svbz^{(i)}, \rvt = 0])$;\\
        }
    }
    \If{$c_2 > 0$}
    {
    $c_1 = c_1 + 1$; $\text{ATE}(\cX) = \text{ATE}(\cX) + \text{ATE}_{\text{d}} / c_2$; \\
    }
}
$\text{ATE}(\cX) = \text{ATE}(\cX) / c_1;$
\caption{ATE estimation using subset search.}
\label{alg:subset_search}
\end{algorithm}

% \noindent\textbf{IRM based Representation Learning:} 
\subsection{IRM based Representation Learning}
To alleviate these concerns, we propose a second algorithm based on invariant risk minimization (IRM). This leverages our use of the subsampling variable and creation of synthetic environments. IRM was proposed to address out-of-distribution generalization for supervised learning tasks and is aimed at learning a predictor that relies only on the causal parents of the label $\rvy$ and ignore any other spurious variables. IRM takes data from different environments indexed as $\rve$ and learns a representation $\Phi$ that transforms the features $\rvbx$ such that $\rve \perp \rvy | \Phi(\rvbx)$. Given that our invariance criterion is of a similar form, and involves checking invariance of the outcome $\rvy$ conditioned on the feature set $\rvbz$ and the treatment $\rvt$ across environments $\rve$, IRM is a perfect fit to test this criterion.  

Our IRM based procedure leverages IRMv1 from \cite{arjovsky2019invariant} with linear representation $\Phi$. We take the data in treatment group $\rvt=1$ (or the control group $\rvt=0$) and divide it into different environments based on $\rve$ and  pass it as input to IRMv1. From the theory of IRM it follows that if the absolute value of some coefficient of $\Phi$ is low, then the corresponding component is unlikely to be a part of the subset that satisfies the invariance criterion. Following this observation, we define a vector  of absolute values of $\Phi$ and denote it as $|\Phi|$. We divide the values in $|\Phi|$ into two clusters using k-means clustering with $k=2$. We select the subset of the features that correspond to the cluster with a higher mean absolute value. We estimate the treatment effect by adjusting over this selected subset. Further details of the procedure can be found in Algorithm \ref{alg:irm_sear} (we describe the algorithm for treatment group and can run a similar procedure for control group). While the computational complexity of IRMv1 (and hence Algorithm \ref{alg:irm_sear}) is unclear yet, in practice, Algorithm \ref{alg:irm_sear} is much faster and scales better (see Figure \ref{fig:toy5_expts}) than Algorithm \ref{alg:subset_search}.

\begin{algorithm}
\KwInput{ $n, n_r, \rvt,\rvy, \rve, \rvbx^{(o)} \setminus \rvx_t$}
\KwOutput{$\text{ATE}$}
\KwInitialization{$\text{ATE} = 0, k=2$}
\For(\tcp*[h]{Use a different train-test split in each run}) {$r = 1,\cdots,n_r$} 
{
    
    { $\Phi \leftarrow \mathsf{IRMv1}(\rvy, \rvbx^{(o)} \setminus \rvx_t, \rve, \rvt=1)$ \\ 
      $ \rvbz_{\mathsf{irm}} \leftarrow \mathsf{kmeans}(|\Phi|, k)$ \tcp*[h]{$\rvbz_{\mathsf{irm}}$ is the subset of variables in the cluster with higher mean absolute value} \\
         $\text{ATE} = \text{ATE} + \frac{1}{n} \sum_{i=1}^{n} (\Expectation[\rvy |\rvbz_{\mathsf{irm}}= \svbz^{(i)}, \rvt = 1] - \Expectation[\rvy | \rvbz_{\mathsf{irm}}= \svbz^{(i)}, \rvt = 0])$
    }
}
$\text{ATE} = \text{ATE} / n_r;$
\caption{ATE estimation using IRM }
\label{alg:irm_sear}
\end{algorithm}
\section{Experiments}
\label{section:experiments}

\noindent{\bf ATE estimation and Performance metrics:}
To test how successful our method is with respect to finding valid adjustments, we consider estimating the ATE of $\rvt$ on $\rvy$. When the ground truth ATE is known, we report the absolute error in ATE prediction (averaged over $n_r$ runs). When the ground truth ATE is unknown, we report the estimated ATE (averaged over $n_r$ runs).

As described in Section \ref{subsection:semi-markovian}, ATE can be estimated from observational data by regressing $\rvy$ for the control and the treatment sub-populations on a valid adjustment set. We note that our work is complementary to works on ATE estimation as our focus is on finding valid adjustments. Once we select a valid adjustment, any of the available ATE estimation methods could be used. We use ridge regression with cross-validation as the regression model for baseline as well as our method.\\
 
\noindent{\bf Environment variable and parameters.}
For all of our experiments, we let $n_r = 100$ and $p_{value} = \{0.1,0.2,0.3,0.4,0.5\}$. For our experiments we create an environment variable as being a random function of $\rvx_{\rvt}$ and $\rvt$ (i.e., $\rve = f(\rvx_{\rvt}, \rvt)$). Exact details of their generation and alternate settings, such as the case of $\rve = f(\rvx_{\rvt})$ (i.e., $\rvbv = \emptyset$), are given in Appendix \ref{appendix:experiments}.\\

\noindent{\bf Algorithms.} We compare the following algorithms:
\begin{enumerate}
    \item \texttt{Baseline}: This uses regression on all of the observed features i.e., $\rvbx^{(o)}$ to estimate ATE. In other words, it assumes $\rvbx^{(o)}$ is ignorable. See Appendix \ref{appendix:baseline} for a pseudo-code of \texttt{Baseline}.
    \item \texttt{Exhaustive}: Given $\rvx_{\rvt}$, this applies Algorithm \ref{alg:subset_search} with $\cX$ being the set of all subsets of $\rvbx^{(o)} \setminus \rvx_{\rvt}$.
    \item \texttt{Sparse}: Given $\rvx_{\rvt}$, this applies Algorithm \ref{alg:subset_search} with $\cX$ being the set of all subsets of $\rvbx^{(o)} \setminus \rvx_{\rvt}$ of size at most $k$ (which is determined in the context).
    \item \texttt{IRM-t}: Given $\rvx_{\rvt}$, this applies Algorithm \ref{alg:irm_sear} to the samples from the treatment group.
    \item \texttt{IRM-c}: Given $\rvx_{\rvt}$, this applies Algorithm \ref{alg:irm_sear} to the samples from the control group.
\end{enumerate}

\subsection{Synthetic Experiment}
\label{subssection:synthetic_experiments}
\noindent{\bf Description.}  Consider the toy example $\cG^{toy}$ from Figure \ref{fig:Gtoy_with_u} with unobserved features $\rvu_1 \in \Reals, \rvu_2 \in \Reals^{\td}, \rvu_3 \in \Reals^{\td}, \rvu_4 \in \Reals^{\td}$ and observed features $\rvx_1 \in \Reals, \rvx_2 \in \Reals^{\td}, \rvx_3 \in \Reals^{\td}$ i.e., $\rvbx^{(u)} = \{ \rvu_1, \rvu_2, \rvu_3, \rvu_4\} \in \Reals^{3\td+1}$ and $\rvbx^{(o)} = \{\rvx_1, \rvx_2, \rvx_3\} \in \Reals^{2\td+1}$. Let $d = 2\td+1$ i.e., the dimension of the observed features. For dimension $d$, we generate a dataset (with $n=50000$) using linear structural equation models for $\rvu$'s, $\rvx$'s and $\rvy$ and a logistic linear model for $\rvt$ and $\rve$. See Appendix \ref{appendix:synthetic_experiments} for details.\\

\begin{figure*}[h]
\centering
\begin{subfigure}{.5\textwidth}
  \centering
  \includegraphics[width=0.9\textwidth]{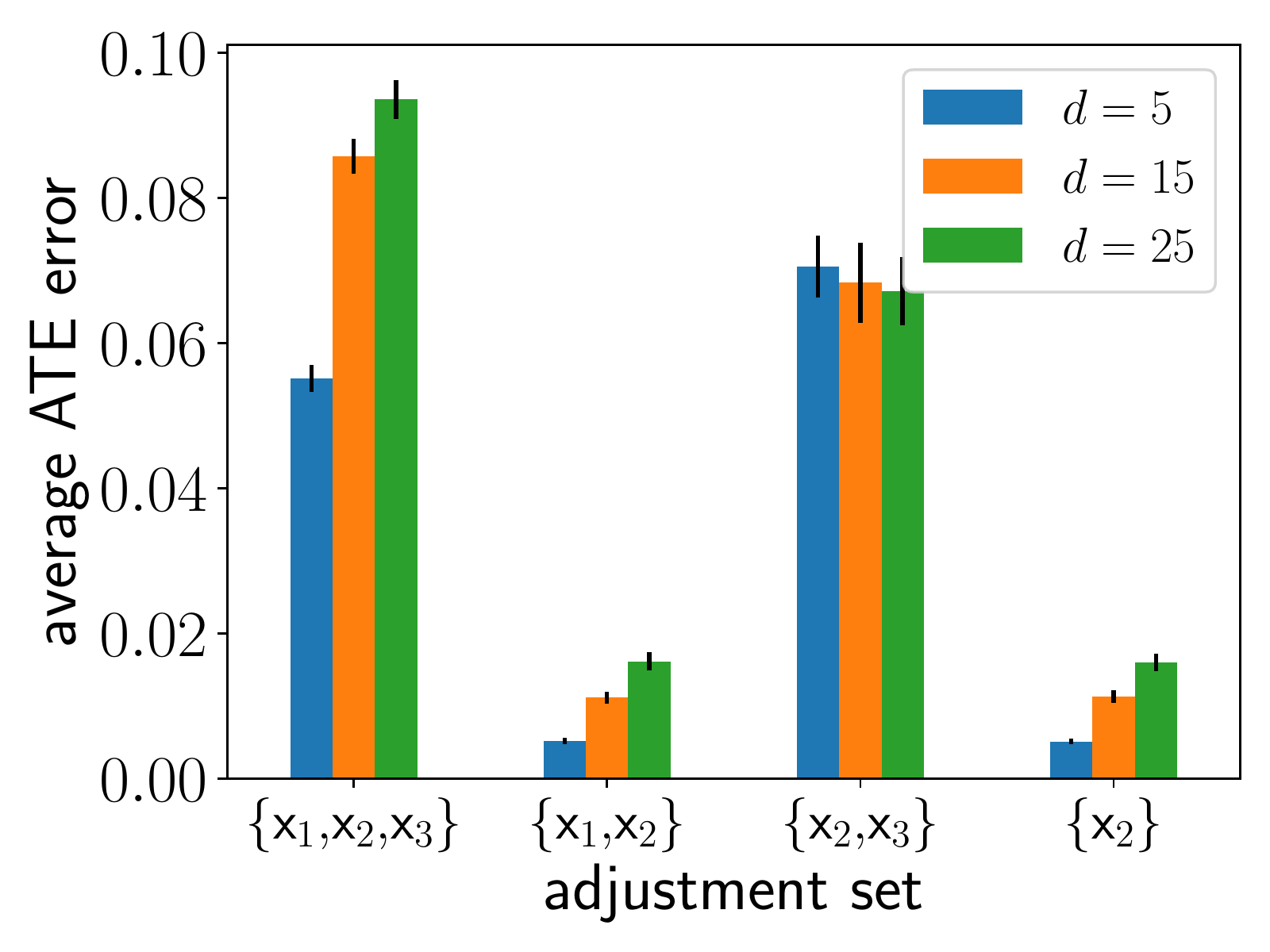}
  \caption{Sets not satisfying back-door ($\{\rvx_1, \rvx_2,\rvx_3\}$, \\ 
  $\{\rvx_2, \rvx_3\}$) result in high ATE error as opposed\\to sets satisfying
  back-door ($\{\rvx_1, \rvx_2\}, \{\rvx_2\}$).}
  \label{fig:toy3_expts}
\end{subfigure}%
\begin{subfigure}{.5\textwidth}
  %\centering
  \includegraphics[width=0.9\textwidth]{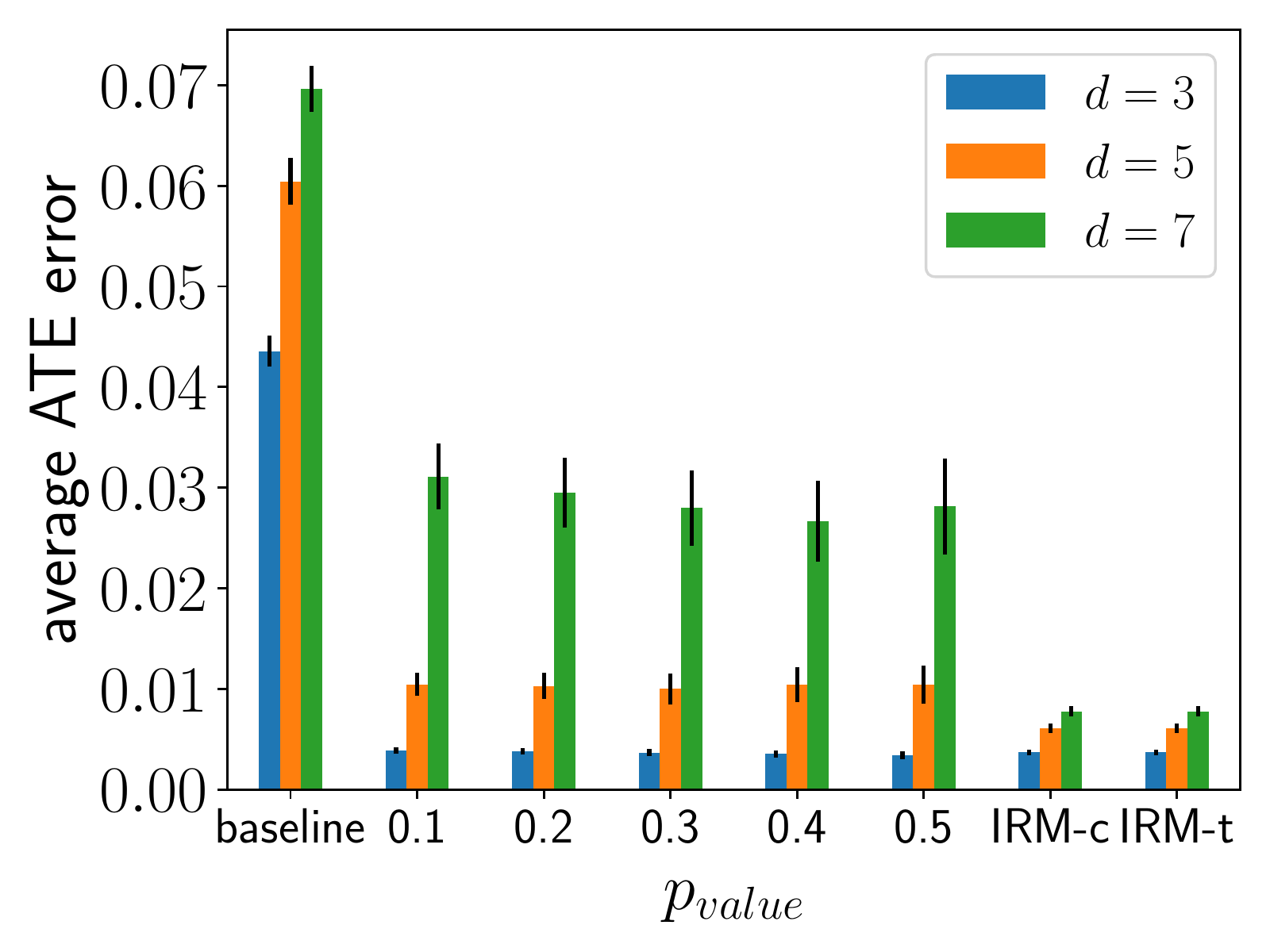}
  \caption{Performance of Algorithms \ref{alg:subset_search} and \ref{alg:irm_sear} on $\cG^{toy}$.}
  \label{fig:toy4_expts}
\end{subfigure}\\
\begin{subfigure}{.5\textwidth}
  %\centering
  \includegraphics[width=0.9\textwidth]{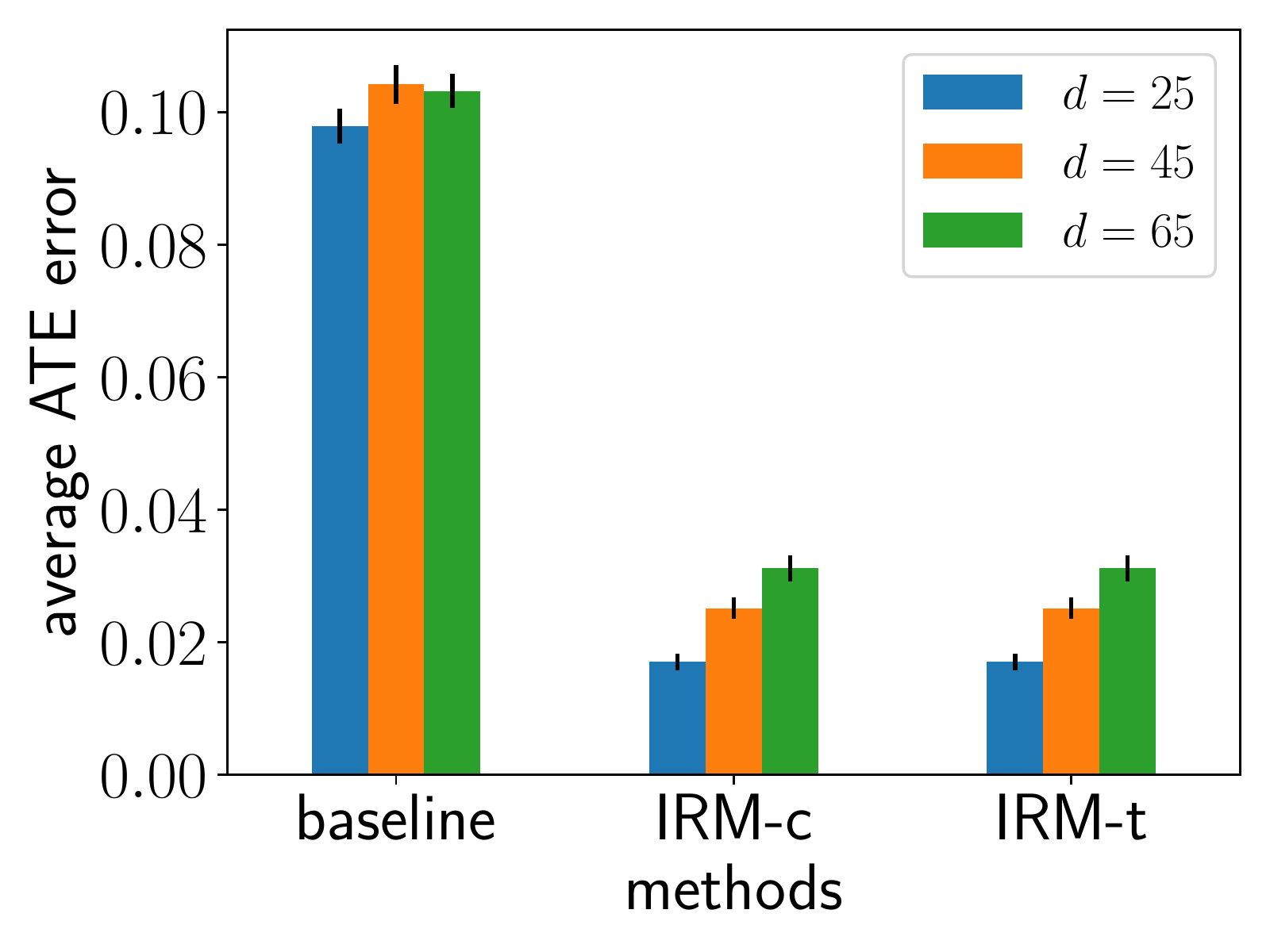}
  \caption{Performance of Algorithm \ref{alg:irm_sear} in high dimensions.}
  \label{fig:toy5_expts}
\end{subfigure}
\caption{Validating our theoretical results and our algorithms on the toy example $\cG^{toy}$.}
% (a) Sets not satisfying back-door ($\{\rvx_1, \rvx_2,\rvx_3\}$, 
%   $\{\rvx_2, \rvx_3\}$) result in high ATE error; sets satisfying 
%   back-door ($\{\rvx_1, \rvx_2\}, \{\rvx_2\}$) result in low ATE error. (b) Performance of Algorithm \ref{alg:subset_search} and \ref{alg:irm_sear} on $\cG^{toy}$. (c) Performance of Algorithm \ref{alg:irm_sear} in high dimensions.}
\label{fig:toy_expts}
\end{figure*}

\noindent{\bf Results.} 
First, we validate our theoretical results for $d = 5, 15, 25$ (see Figure \ref{fig:toy3_expts}): (a) the ATE error for adjusting on $\{\rvx_1,\rvx_2, \rvx_3\}$ is high since we are in a setting where $\rvbx^{(o)}$ is  \textit{not ignorable}, (b) the ATE error for adjusting on $\{\rvx_1,\rvx_2\}$ is low since it satisfies the back-door criterion, (c) the ATE error for adjusting on $\{\rvx_2,\rvx_3\}$ is high since $\rve \notdsep \rvy | \rvx_2, \rvx_3, \rvt$, (d) the ATE error for adjusting on $\{\rvx_2\}$ is low since $\rve \dsep \rvy | \rvx_2, \rvt$. Next, we validate our algorithms via Figure \ref{fig:toy4_expts}. With $\rvx_{\rvt} = \rvx_1$, our algorithms \texttt{Exhaustive}, \texttt{IRM-t}, and \texttt{IRM-c} significantly outperform \texttt{Baseline} for $d = 3, 5, 7$ even for multiple $p_{value}$ thresholds for \texttt{Exhaustive}. We note that IRM based algorithms significantly outperform the testing based algorithm even in moderately high dimensions ($d=7$) and performs very well even for $d=65$ as seen through in Figure \ref{fig:toy5_expts}.

\subsection{Semi-synthetic Dataset : Infant Health and Development Program (IHDP)}
\label{subsection:ihdp}
\noindent{\bf Description.} IHDP \citep{hill2011bayesian} is generated based on a RCT targeting low-birth-weight, premature infants. The 25-dimensional feature set (comprising of 17 different features) is pre-treatment i.e., it satisfies Assumption \ref{assumption1}. The features measure various aspects about the children and their mothers e.g., child's birth-weight, the number of weeks pre-term that the child was born. See Appendix \ref{appendix:ihdp} for details. In the treated group, the infants were provided with both intensive high-quality childcare and specialist home visits. A biased subset of the treated group is typically removed to create imbalance leaving 139 samples with $\rvt = 1$ and 608 samples with $\rvt = 0$. The outcome, typically simulated using setting ``A'' of the NPCI package \citep{Dorie2016}, is infants' cognitive test score.\\

\noindent{\bf Analysis.} 
The outcome depends on all observed features. Therefore, the set of all observed features satisfies back-door (see Appendix \ref{appendix:ihdp}). To test our method, we drop 7 features and denote the resulting 16-dimensional feature set (comprising of 10 features) by $\rvbx^{(o)}$ to create a challenging non-ignorable case. We use \textit{child's birth-weight} as $\rvx_{\rvt}$. Therefore, we keep this feature in $\rvbx^{(o)}$. See Appendix \ref{appendix:ihdp} for the choice of other features in $\rvbx^{(o)}$.\\

\noindent{\bf Results.} 
We compare \texttt{Baseline}, \texttt{Exhaustive}, \texttt{Sparse} with $k = 5$, \texttt{IRM-c} and \texttt{IRM-t}. All our algorithms except \texttt{IRM-c} significantly outperform \texttt{Baseline} (see Figure \ref{fig:ihdp}). The intuition behind $k = 5$ is the belief that valid adjustments of size 5 exist (see Appendix \ref{appendix:ihdp})\footnote{We note that \texttt{Sparse} still has to perform $\sum_{i = 0}^{5} \binom{9}{i} = 382$ tests to estimate ATE. Therefore, \texttt{Sparse} performs not very differently from \texttt{Exhaustive}.}.

\begin{figure}[ht!]
%\centering
\begin{subfigure}{.5\textwidth}
  \centering
  \includegraphics[width=0.9\textwidth]{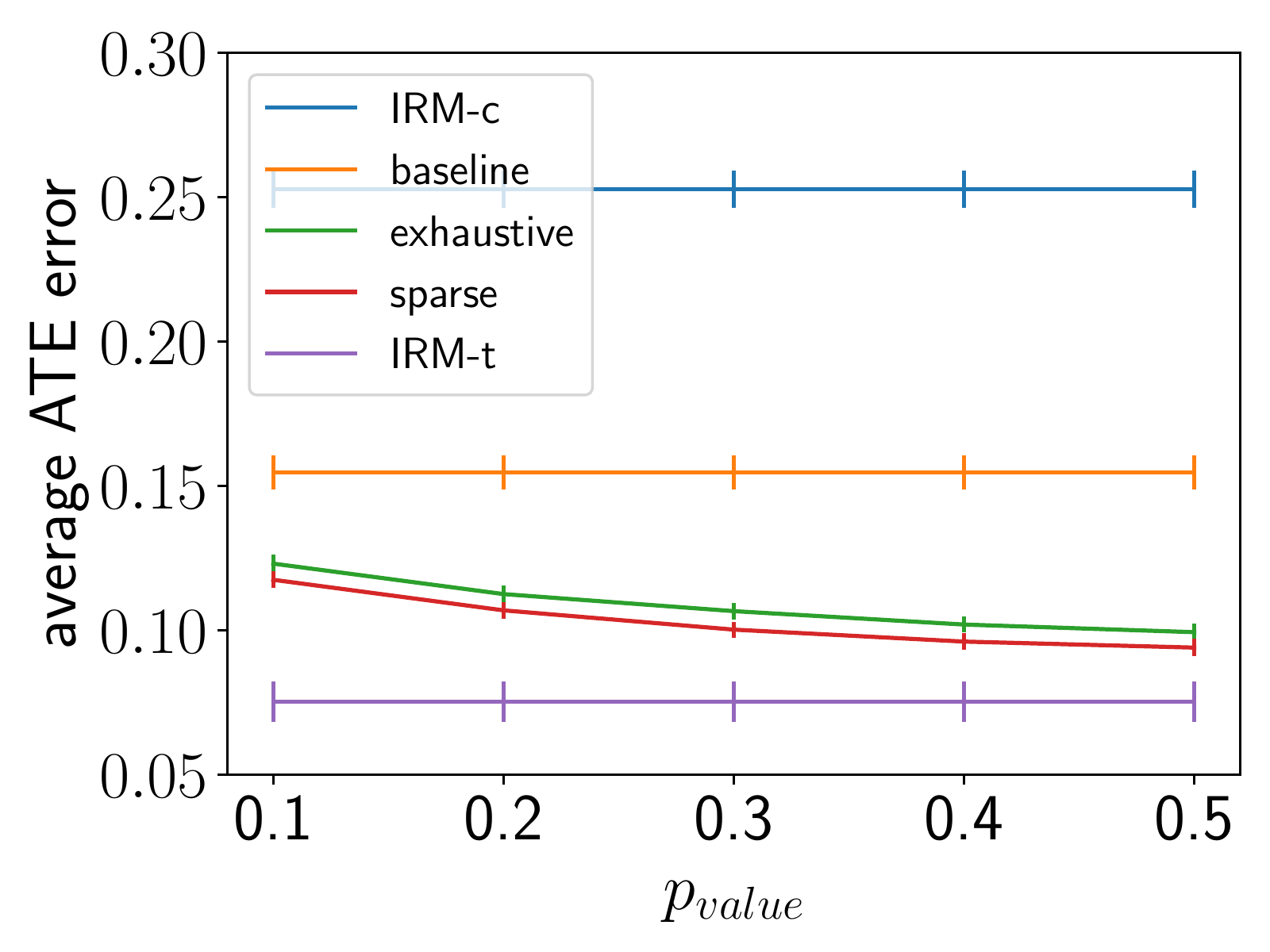}
  \caption{Performance of our algorithms and \\baselines on IHDP dataset.}
  \label{fig:ihdp}
\end{subfigure}%
\begin{subfigure}{.5\textwidth}
\centering
  \includegraphics[width=0.9\textwidth]{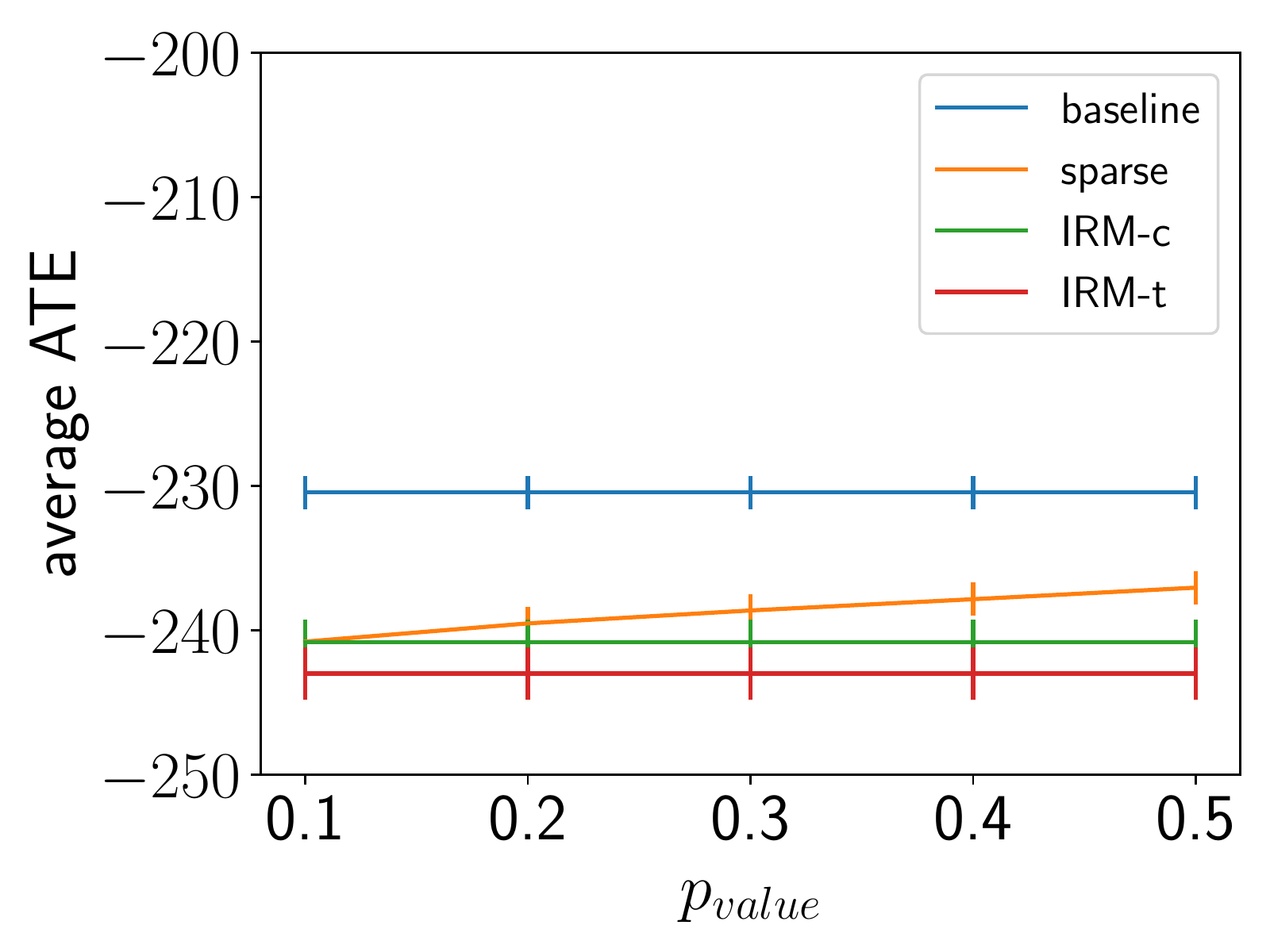}
%  \caption{Cattaneo2}
%\end{subfigure}
\caption{Performance of our algorithms and \\baselines on Cattaneo2 dataset.}
\label{fig:cattaneo2}
%\label{fig:test}
\end{subfigure}%
\caption{Performance of our algorithms compared to the baseline on benchmark datasets}
\end{figure}
\subsection{Real Dataset : Cattaneo2}
\label{subsection:cattaneo}

\noindent{\bf Description.} 
Cattaneo2 \citep{cattaneo2010efficient}
studies the effect of maternal smoking on babies' birth weight. 
The 20 observed features
measure various attributes about the children, their mothers and their fathers. 
See Appendix \ref{appendix:cattaneo} for details.
The dataset 
% (with $n=4642$) 
considers the maternal smoking habit during pregnancy as the treatment i.e., $\rvt = 1$ if smoking (864 samples) and $\rvt = 0$ if not smoking (3778 samples). \\

\noindent{\bf Analysis.}
Out of the features we have access to (see Appendix \ref{appendix:cattaneo}), we pick mother's age to be $\rvx_{\rvt}$. \\

\noindent{\bf Results.}
The ground truth ATE is unknown (because for every sample either $\rvy_0$ or $\rvy_1$ is observed). However, the authors in \cite{AlmondCL2005} expect a strong negative effect of maternal smoking on the weights of babies -- about 200 to 250 grams lighter for a baby with a mother smoking during pregnancy. We compare all the algorithms except \texttt{Exhaustive} with $\rvx_{\rvt} =$ mother’s age. For the \texttt{sparse} algorithm, we set $k = 5$ to ensure a reasonable run-time. As seen in Figure \ref{fig:cattaneo2}, the ATE estimated using all our algorithms fall in the desired interval (i.e., (-250,-200)) and suggest a larger negative effect compared to the \texttt{Baseline}.
\section{Conclusion and Discussion}
\label{section:conclusion_limitations_future_work}
We showed that it is possible to find valid adjustment sets under non-ignorability with the knowledge of a single causal parent of the treatment. We achieved this by providing an invariance test that exactly identifies all the subsets of observed features (not involving this parent) that satisfy the back-door criterion.\\

\noindent{\bf Knowledge of a causal parent of the treatment.} Our invariance test depends on the causal parent of the treatment i.e., $\rvx_{\rvt}$ via the environment variable i.e., $\rve$. Therefore, our approach works even when the expert knowledge of $\rvx_{\rvt}$ is not available or samples of $\rvx_{\rvt}$ are not observed so long as we have samples of $\rve$ directly. Investigating the application of this insight is an interesting question for future research.\\

\noindent{\bf Assumption \ref{assumption1} and \ref{assumption3}.} Assumption \ref{assumption1} and faithfulness (a stronger version of Assumption \ref{assumption3}) are commonly used in data-driven covariate selection works \citep{entner2013data, gultchin2020differentiable, cheng2020towards}.   While settings beyond Assumption \ref{assumption1} are interesting for future research, finding valid adjustments under Assumption \ref{assumption1} is non-trivial and important in both PO and Pearlian framework (see the first paragraph in \cite{vanderweele2011new}). Further, we note that Assumption \ref{assumption1} holds for some benchmark causal effect estimation datasets (e.g., IHDP, Twins). Lastly, while it is common to assume faithfulness with respect to conditional independencies involving the entire DAG, we assume faithfulness only with respect to conditional independencies involving the sub-sampling variable.\\

\noindent{\bf Alternate minimal DAG knowledge.} As discussed in Remark \ref{remark:m_bias}, our method doesn't cover all back-door criteria (e.g., the M-bias problem). Therefore, exploring alternate minimal DAG knowledge sufficient to test for a broader/different family of valid adjustments could be fruitful.
\section*{Acknolwedgements}

We thank the anonymous reviewers of NeurIPS 2021 for bringing to our notice the works of \cite{entner2013data} and \cite{cheng2020towards} as well as for several suggestions. We also thank the anonymous referees of AISTATS 2022 for their comments and feedback. Kartik Ahuja acknowledges the support provided by IVADO postdoctoral fellowship funding program.

\bibliographystyle{abbrvnat}
\bibliography{content/references}

\begin{thebibliography}{51}
\providecommand{\natexlab}[1]{#1}
\providecommand{\url}[1]{\texttt{#1}}
\expandafter\ifx\csname urlstyle\endcsname\relax
  \providecommand{\doi}[1]{doi: #1}\else
  \providecommand{\doi}{doi: \begingroup \urlstyle{rm}\Url}\fi

\bibitem[Abadie et~al.(2004)Abadie, Drukker, Herr, and Imbens]{Abadie2004}
A.~Abadie, D.~Drukker, J.~L. Herr, and G.~W. Imbens.
\newblock Implementing matching estimators for average treatment effects in
  stata.
\newblock \emph{The stata journal}, 4\penalty0 (3):\penalty0 290--311, 2004.

\bibitem[Abadie et~al.(2010)Abadie, Diamond, and
  Hainmueller]{abadie2010synthetic}
A.~Abadie, A.~Diamond, and J.~Hainmueller.
\newblock Synthetic control methods for comparative case studies: Estimating
  the effect of california’s tobacco control program.
\newblock \emph{Journal of the American statistical Association}, 105\penalty0
  (490):\penalty0 493--505, 2010.

\bibitem[Acharya et~al.(2018)Acharya, Bhattacharyya, Daskalakis, and
  Kandasamy]{acharya2018learning}
J.~Acharya, A.~Bhattacharyya, C.~Daskalakis, and S.~Kandasamy.
\newblock Learning and testing causal models with interventions.
\newblock \emph{Advances in Neural Information Processing Systems}, 31, 2018.

\bibitem[Alaa and van~der Schaar(2017)]{alaa2017bayesian}
A.~M. Alaa and M.~van~der Schaar.
\newblock Bayesian inference of individualized treatment effects using
  multi-task gaussian processes.
\newblock \emph{arXiv preprint arXiv:1704.02801}, 2017.

\bibitem[Almond et~al.(2005)Almond, Chay, and Lee]{AlmondCL2005}
D.~Almond, K.~Y. Chay, and D.~S. Lee.
\newblock The costs of low birth weight.
\newblock \emph{The Quarterly Journal of Economics}, 120\penalty0 (3):\penalty0
  1031--1083, 2005.

\bibitem[Arjovsky et~al.(2019)Arjovsky, Bottou, Gulrajani, and
  Lopez-Paz]{arjovsky2019invariant}
M.~Arjovsky, L.~Bottou, I.~Gulrajani, and D.~Lopez-Paz.
\newblock Invariant risk minimization.
\newblock \emph{arXiv preprint arXiv:1907.02893}, 2019.

\bibitem[Balke and Pearl(1997)]{balke1997bounds}
A.~Balke and J.~Pearl.
\newblock Bounds on treatment effects from studies with imperfect compliance.
\newblock \emph{Journal of the American Statistical Association}, 92\penalty0
  (439):\penalty0 1171--1176, 1997.

\bibitem[Bareinboim et~al.(2012)Bareinboim, Brito, and
  Pearl]{bareinboim2012local}
E.~Bareinboim, C.~Brito, and J.~Pearl.
\newblock Local characterizations of causal bayesian networks.
\newblock In \emph{Graph Structures for Knowledge Representation and
  Reasoning}, pages 1--17. Springer, 2012.

\bibitem[Cattaneo(2010)]{cattaneo2010efficient}
M.~D. Cattaneo.
\newblock Efficient semiparametric estimation of multi-valued treatment effects
  under ignorability.
\newblock \emph{Journal of Econometrics}, 155\penalty0 (2):\penalty0 138--154,
  2010.

\bibitem[Cheng et~al.(2020)Cheng, Li, Liu, Yu, Lee, and Liu]{cheng2020towards}
D.~Cheng, J.~Li, L.~Liu, K.~Yu, T.~D. Lee, and J.~Liu.
\newblock Towards unique and unbiased causal effect estimation from data with
  hidden variables.
\newblock \emph{arXiv preprint arXiv:2002.10091}, 2020.

\bibitem[Dorie(2016)]{Dorie2016}
V.~Dorie.
\newblock Npci: Non-parametrics for causal inference.
\newblock 2016.
\newblock URL \url{https://github.com/vdorie/npci}.

\bibitem[Entner et~al.(2013)Entner, Hoyer, and Spirtes]{entner2013data}
D.~Entner, P.~Hoyer, and P.~Spirtes.
\newblock Data-driven covariate selection for nonparametric estimation of
  causal effects.
\newblock In \emph{Artificial Intelligence and Statistics}, pages 256--264.
  PMLR, 2013.

\bibitem[Funk et~al.(2011)Funk, Westreich, Wiesen, St{\"u}rmer, Brookhart, and
  Davidian]{Funk2011}
M.~J. Funk, D.~Westreich, C.~Wiesen, T.~St{\"u}rmer, M.~A. Brookhart, and
  M.~Davidian.
\newblock Doubly robust estimation of causal effects.
\newblock \emph{American journal of epidemiology}, 173\penalty0 (7):\penalty0
  761--767, 2011.

\bibitem[Gultchin et~al.(2020)Gultchin, Kusner, Kanade, and
  Silva]{gultchin2020differentiable}
L.~Gultchin, M.~Kusner, V.~Kanade, and R.~Silva.
\newblock Differentiable causal backdoor discovery.
\newblock In \emph{International Conference on Artificial Intelligence and
  Statistics}, pages 3970--3979. PMLR, 2020.

\bibitem[Hill(2011)]{hill2011bayesian}
J.~L. Hill.
\newblock Bayesian nonparametric modeling for causal inference.
\newblock \emph{Journal of Computational and Graphical Statistics}, 20\penalty0
  (1):\penalty0 217--240, 2011.

\bibitem[Imbens(2020)]{imbens2020potential}
G.~W. Imbens.
\newblock Potential outcome and directed acyclic graph approaches to causality:
  Relevance for empirical practice in economics.
\newblock \emph{Journal of Economic Literature}, 58\penalty0 (4):\penalty0
  1129--79, 2020.

\bibitem[Imbens and Rubin(2010)]{imbens2010rubin}
G.~W. Imbens and D.~B. Rubin.
\newblock Rubin causal model.
\newblock In \emph{Microeconometrics}, pages 229--241. Springer, 2010.

\bibitem[Imbens and Rubin(2015)]{imbens2015causal}
G.~W. Imbens and D.~B. Rubin.
\newblock \emph{Causal inference in statistics, social, and biomedical
  sciences}.
\newblock Cambridge University Press, 2015.

\bibitem[Johansson et~al.(2016)Johansson, Shalit, and
  Sontag]{johansson2016learning}
F.~Johansson, U.~Shalit, and D.~Sontag.
\newblock Learning representations for counterfactual inference.
\newblock In \emph{International conference on machine learning}, pages
  3020--3029. PMLR, 2016.

\bibitem[Kallus(2020)]{kallus2020deepmatch}
N.~Kallus.
\newblock Deepmatch: Balancing deep covariate representations for causal
  inference using adversarial training.
\newblock In \emph{International Conference on Machine Learning}, pages
  5067--5077. PMLR, 2020.

\bibitem[K{\"u}nzel et~al.(2019)K{\"u}nzel, Sekhon, Bickel, and
  Yu]{kunzel2019metalearners}
S.~R. K{\"u}nzel, J.~S. Sekhon, P.~J. Bickel, and B.~Yu.
\newblock Metalearners for estimating heterogeneous treatment effects using
  machine learning.
\newblock \emph{Proceedings of the national academy of sciences}, 116\penalty0
  (10):\penalty0 4156--4165, 2019.

\bibitem[LaLonde(1986)]{Lalonde1986}
R.~J. LaLonde.
\newblock Evaluating the econometric evaluations of training programs with
  experimental data.
\newblock \emph{The American economic review}, pages 604--620, 1986.

\bibitem[Liu et~al.(2012)Liu, Brookhart, Schneeweiss, Mi, and
  Setoguchi]{liu2012implications}
W.~Liu, M.~A. Brookhart, S.~Schneeweiss, X.~Mi, and S.~Setoguchi.
\newblock Implications of m bias in epidemiologic studies: a simulation study.
\newblock \emph{American journal of epidemiology}, 176\penalty0 (10):\penalty0
  938--948, 2012.

\bibitem[Pearl(1993)]{Pearl1993}
J.~Pearl.
\newblock [bayesian analysis in expert systems]: Comment: graphical models,
  causality and intervention.
\newblock \emph{Statistical Science}, 8\penalty0 (3):\penalty0 266--269, 1993.

\bibitem[Pearl(1995)]{pearl1995causal}
J.~Pearl.
\newblock Causal diagrams for empirical research.
\newblock \emph{Biometrika}, 82\penalty0 (4):\penalty0 669--688, 1995.

\bibitem[Pearl(2009)]{Pearl2009}
J.~Pearl.
\newblock \emph{Causality}.
\newblock Cambridge university press, 2009.

\bibitem[Pearl(2014)]{pearl2014comment}
J.~Pearl.
\newblock Comment: understanding simpson’s paradox.
\newblock \emph{The American Statistician}, 68\penalty0 (1):\penalty0 8--13,
  2014.

\bibitem[Pearl et~al.(2016)Pearl, Glymour, and Jewell]{Pearl2016}
J.~Pearl, M.~Glymour, and N.~P. Jewell.
\newblock \emph{Causal inference in statistics: A primer}.
\newblock John Wiley \& Sons, 2016.

\bibitem[Perkovic et~al.(2018)Perkovic, Textor, Kalisch, and
  Maathuis]{perkovic2018complete}
E.~Perkovic, J.~Textor, M.~Kalisch, and M.~H. Maathuis.
\newblock Complete graphical characterization and construction of adjustment
  sets in markov equivalence classes of ancestral graphs.
\newblock 2018.

\bibitem[Rosenbaum(1989)]{Rosenbaum1989}
P.~R. Rosenbaum.
\newblock Optimal matching for observational studies.
\newblock \emph{Journal of the American Statistical Association}, 84\penalty0
  (408):\penalty0 1024--1032, 1989.

\bibitem[Rosenbaum and Rubin(1983)]{rosenbaum1983central}
P.~R. Rosenbaum and D.~B. Rubin.
\newblock The central role of the propensity score in observational studies for
  causal effects.
\newblock \emph{Biometrika}, 70\penalty0 (1):\penalty0 41--55, 1983.

\bibitem[Rosenbaum and Rubin(1985)]{rosenbaum1985constructing}
P.~R. Rosenbaum and D.~B. Rubin.
\newblock Constructing a control group using multivariate matched sampling
  methods that incorporate the propensity score.
\newblock \emph{The American Statistician}, 39\penalty0 (1):\penalty0 33--38,
  1985.

\bibitem[Rubin(1973)]{Rubin1973}
D.~B. Rubin.
\newblock Matching to remove bias in observational studies.
\newblock \emph{Biometrics}, pages 159--183, 1973.

\bibitem[Rubin(1974)]{rubin1974estimating}
D.~B. Rubin.
\newblock Estimating causal effects of treatments in randomized and
  nonrandomized studies.
\newblock \emph{Journal of educational Psychology}, 66\penalty0 (5):\penalty0
  688, 1974.

\bibitem[Sch{\"o}lkopf(2019)]{scholkopf2019causality}
B.~Sch{\"o}lkopf.
\newblock Causality for machine learning.
\newblock \emph{arXiv preprint arXiv:1911.10500}, 2019.

\bibitem[Shah et~al.(2021)Shah, Ahuja, Shanmugam, Wei, Varshney, and
  Dhurandhar]{shah2021treatment}
A.~Shah, K.~Ahuja, K.~Shanmugam, D.~Wei, K.~R. Varshney, and A.~Dhurandhar.
\newblock Treatment effect estimation using invariant risk minimization.
\newblock In \emph{ICASSP 2021-2021 IEEE International Conference on Acoustics,
  Speech and Signal Processing (ICASSP)}, pages 5005--5009. IEEE, 2021.

\bibitem[Shalit et~al.(2017)Shalit, Johansson, and
  Sontag]{shalit2017estimating}
U.~Shalit, F.~D. Johansson, and D.~Sontag.
\newblock Estimating individual treatment effect: generalization bounds and
  algorithms.
\newblock In \emph{International Conference on Machine Learning}, pages
  3076--3085. PMLR, 2017.

\bibitem[Shi et~al.(2019)Shi, Blei, and Veitch]{shi2019adapting}
C.~Shi, D.~M. Blei, and V.~Veitch.
\newblock Adapting neural networks for the estimation of treatment effects.
\newblock \emph{arXiv preprint arXiv:1906.02120}, 2019.

\bibitem[Shi et~al.(2020)Shi, Veitch, and Blei]{shi2020invariant}
C.~Shi, V.~Veitch, and D.~Blei.
\newblock Invariant representation learning for treatment effect estimation.
\newblock \emph{arXiv preprint arXiv:2011.12379}, 2020.

\bibitem[Shimoni et~al.(2019)Shimoni, Karavani, Ravid, Bak, Ng, Alford, Meade,
  and Goldschmidt]{causalevaluations}
Y.~Shimoni, E.~Karavani, S.~Ravid, P.~Bak, T.~H. Ng, S.~H. Alford, D.~Meade,
  and Y.~Goldschmidt.
\newblock An evaluation toolkit to guide model selection and cohort definition
  in causal inference.
\newblock \emph{arXiv preprint arXiv:1906.00442}, 2019.

\bibitem[Shpitser and Pearl(2008)]{shpitser2008complete}
I.~Shpitser and J.~Pearl.
\newblock Complete identification methods for the causal hierarchy.
\newblock \emph{Journal of Machine Learning Research}, 9:\penalty0 1941--1979,
  2008.

\bibitem[Smith and Todd(2005)]{Smith2005}
J.~A. Smith and P.~E. Todd.
\newblock Does matching overcome lalonde's critique of nonexperimental
  estimators?
\newblock \emph{Journal of econometrics}, 125\penalty0 (1-2):\penalty0
  305--353, 2005.

\bibitem[Strobl et~al.(2019)Strobl, Zhang, and
  Visweswaran]{strobl2019approximate}
E.~V. Strobl, K.~Zhang, and S.~Visweswaran.
\newblock Approximate kernel-based conditional independence tests for fast
  non-parametric causal discovery.
\newblock \emph{Journal of Causal Inference}, 7\penalty0 (1), 2019.

\bibitem[Swaminathan et~al.(2016)Swaminathan, Krishnamurthy, Agarwal,
  Dud{\'\i}k, Langford, Jose, and Zitouni]{swaminathan2016off}
A.~Swaminathan, A.~Krishnamurthy, A.~Agarwal, M.~Dud{\'\i}k, J.~Langford,
  D.~Jose, and I.~Zitouni.
\newblock Off-policy evaluation for slate recommendation.
\newblock \emph{arXiv preprint arXiv:1605.04812}, 2016.

\bibitem[Tian and Pearl(2002)]{TianP2002}
J.~Tian and J.~Pearl.
\newblock A general identification condition for causal effects.
\newblock In \emph{Aaai/iaai}, pages 567--573, 2002.

\bibitem[Uhler et~al.(2013)Uhler, Raskutti, B{\"u}hlmann, and
  Yu]{uhler2013geometry}
C.~Uhler, G.~Raskutti, P.~B{\"u}hlmann, and B.~Yu.
\newblock Geometry of the faithfulness assumption in causal inference.
\newblock \emph{The Annals of Statistics}, pages 436--463, 2013.

\bibitem[VanderWeele and Shpitser(2011)]{vanderweele2011new}
T.~J. VanderWeele and I.~Shpitser.
\newblock A new criterion for confounder selection.
\newblock \emph{Biometrics}, 67\penalty0 (4):\penalty0 1406--1413, 2011.

\bibitem[Verma and Pearl(1990)]{verma1990causal}
T.~Verma and J.~Pearl.
\newblock Causal networks: Semantics and expressiveness.
\newblock In \emph{Machine intelligence and pattern recognition}, volume~9,
  pages 69--76. Elsevier, 1990.

\bibitem[Wager and Athey(2018)]{wager2018estimation}
S.~Wager and S.~Athey.
\newblock Estimation and inference of heterogeneous treatment effects using
  random forests.
\newblock \emph{Journal of the American Statistical Association}, 113\penalty0
  (523):\penalty0 1228--1242, 2018.

\bibitem[Yoon et~al.(2018)Yoon, Jordon, and Van Der~Schaar]{yoon2018ganite}
J.~Yoon, J.~Jordon, and M.~Van Der~Schaar.
\newblock Ganite: Estimation of individualized treatment effects using
  generative adversarial nets.
\newblock In \emph{International Conference on Learning Representations}, 2018.

\bibitem[Zhang(2008)]{zhang2008causal}
J.~Zhang.
\newblock Causal reasoning with ancestral graphs.
\newblock \emph{Journal of Machine Learning Research}, 9:\penalty0 1437--1474,
  2008.

\end{thebibliography}
\clearpage
\appendix
\section*{Appendix}
\textbf{Organization.}
In Appendix \ref{appendix:societal_impacts} we briefly discuss any potential societal impacts of our work. 
In Appendix \ref{appendix:related_work}, we discuss  prior work related to potential outcomes and usage of representation learning to debias treatment effect.
In Appendix \ref{appendix:po}, we review potential outcomes (PO) framework, discuss ignorability and connect it with valid adjustment.
In Appendix \ref{appendix:d_separation}, we provide the definition of d-separation as well as a few related definitions.
In Appendix \ref{appendix:additional_notations}, we provide a few additional notations.
In Appendix \ref{appendix:proof_thm_sufficiency}, we provide a proof of Theorem \ref{thm_sufficiency} and also provide an illustrative example for Theorem \ref{thm_sufficiency}.
In Appendix \ref{appendix:proof_thm_necessity}, we provide a proof of Theorem \ref{thm_necessity}.
In Appendix \ref{appendix:m_bias}, we provide a discussion on the M-bias problem.
In Appendix \ref{appendix:finding_back_doors}, we provide an Algorithm (Algorithm \ref{alg:backdoors}) that, when all the parents of the treatment are observed and known, finds all subsets of the observed features 
satisfying the back-door criterion relative to $(\rvt, \rvy)$ in $\cG$ as promised in Section \ref{section:main_results}. We also provide an example illustrating Algorithm \ref{alg:backdoors} and the associated result via Corollary \ref{corollary:backdoor}.
In Appendix \ref{appendix:baseline}, we provide an implementation of the \texttt{Baseline} ATE estimation routine considered in this work. In Appendix \ref{appendix:experiments}, we discuss the usage of real-world CI testers in Algorithm \ref{alg:subset_search}, provide more discussions on experiments from Section \ref{section:experiments}, specify all the training details, as well as provide more details regarding the comparison of our method with \cite{entner2013data}, \cite{gultchin2020differentiable}, and \cite{cheng2020towards}.
% Further, we also specify all the training details, and provide a discussion on the assets used in this work in Appendix \ref{appendix:experiments}.

\section{Societal Impact}
\label{appendix:societal_impacts}
In health-care scenarios, since it is sometimes difficult/unethical to do randomized control trials (RCTs), sometimes the consensus treatment protocol is decided based on observational studies. Our algorithm could pick out a correct valid adjustment set when some existing methods assume ignorability due to lack of expert knowledge about the causal model.

On the flip side, due to lower testing power at finite samples or mis-identification of a feature as a direct parent of the treatment (a local causal knowledge required in our work), our algorithm could pick an incorrect valid adjustment set. This, in turn, could potentially result in miscalculation of the treatment effect. The consensus treatment protocols based on such observational conclusions could prove detrimental.
% (for e.g., sometimes effect could flip sign compared to observational studies upon conducting an RCT under confounding). 
However, we emphasize that this is a risk associated with most (if not all) observational studies and effect estimation algorithms.

\section{Additional related work}\label{appendix:related_work}
\noindent{\bf Potential Outcomes framework.} Potential outcomes (PO) framework formalizes the notion of ignorability as a condition on the observed features that is sufficient (amongst others) for valid adjustment in treatment effect estimation \citep{imbens2010rubin}. Various methods like propensity scoring (\cite{rosenbaum1983central}), \textit{matching} (\cite{rosenbaum1985constructing}) of the treatment group and the control group based on features that satisfy ignorability, and 
synthetic control methods  (\cite{abadie2010synthetic}) have been used to debias effect estimation. In another line of work (\cite{wager2018estimation,kunzel2019metalearners,alaa2017bayesian}), treatment effect was estimated by regressing the outcome on the treated and the untreated sub-populations. 
While this list of works on the PO framework is by no means exhaustive, in a nutshell, these methods can be seen as techniques to estimate the treatment effect when a valid adjustment set is given.\\

\noindent{\bf Representation learning based techniques.} 
Following the main idea behind matching (\cite{Rubin1973, Abadie2004, Rosenbaum1989}), recent methods inspired by deep learning and domain adaptation, used a neural network to transform the features 
and then carry out matching in the representation space (\cite{shi2019adapting,shalit2017estimating,johansson2016learning,yoon2018ganite,kallus2020deepmatch}). These methods aimed to correct the lack of overlap between the treated and the control groups while assuming that the representation learned is ignorable (i.e., a valid adjustment).

% Following the approach of \cite{entner2013data} to find valid adjustments using \textit{an anchor variable}, \cite{gultchin2020differentiable} proposed a fully-differentiable optimization framework to find a representation of the features that passes the conditional independence criteria analogous to \cite{entner2013data}. While their approach avoids the brute-force search required by \cite{entner2013data}, their approach is as limited in the reverse direction as \cite{entner2013data} (see the related discussion in Section \ref{section:related_works}). More specifically, their approach cannot be used to conclude that $\emptyset, \{\rvx_3\}, \{\rvx_2,\rvx_3\}$ are not valid adjustment sets in $\cG^{toy}$ similar to \cite{entner2013data} (see the related discussion in Section \ref{section:main_results} and Appendix \ref{appendix:experiments}).

\section{Review of potential outcomes and ignorability}\label{appendix:po}
We briefly review the potential outcomes (PO) framework in the context of treatment effect estimation \citep{imbens2015causal}.
%While the semi-Markovian model typically considers the Causal Bayesian Network (CBN), the randomness in CBNs (associated with the conditional distributions) is removed in the PO framework by introducing additional exogenous variables affecting the involved variable.
In the PO framework, there are exogenous variables called \textit{units}. With a slight abuse of notation, we denote them by $\rvbx^{(u)}$ as well.
% In summary,  In other words, 
When $\rvbx^{(u)}$ is fixed to say $\svbx^{(u)}$, the observed variables (including $\rvy$) are deterministically fixed i.e., only the randomness in the units induces randomness in the observed variables.
% The PO framework typically studies the case where the features
The PO framework typically studies the setup where the observed features $\rvbx^{(o)}$ are pre-treatment (similar to semi-Markovian model under Assumption \ref{assumption1}).
Every observational sample $(\svbx^{(o)}, t, y)$ has an associated unit $\svbx^{(u)}$.
% (realization of the unobserved exogenous variables). Then $y$ is the factual outcome under treatment $\rvt=t$.
% and the observed outcome $y$ is known as the factual outcome under $\rvt=t$.
For $t' \in \{0,1\}$, the \textit{potential outcome} $\rvy_{t'}$ is the resulting outcome for the unit $\svbx^{(u)}$ when the treatment $\rvt$ is set (by an intervention) to $t'$. 
% The \textit{Average Treatment Effect} (ATE) is defined as : $\Expectation_{\rvbx^{(u)}}[\rvy_{1}(\rvbx^{(u)}) - \rvy_{0}(\rvbx^{(u)}) ]$

% ATE can be estimated by regressing the factual outcomes for the treated and untreated populations on a subset of features $\rvbx_{\cA} \subseteq \rvbx^{(o)}$, i.e. $\mathrm{ATE}= \Expectation_{\rvbx_{\cA}} \left [\Expectation_{\rvy} [ \rvy| \rvt=1, \rvbx_{\cA}] - \Expectation_{\rvy} [ \rvy| \rvt=0, \rvbx_{\cA}] \right] $, if it satisfies what is called the \textit{ignorability} condition. 
% When this is possible, one also calls $\rvbx_{\cA}$ as a \textit{valid adjustment set}.

\begin{definition}
 (Ignorability.) Any $\rvbz \subseteq \rvbx^{(o)}$ satisfies the ignorability condition if $\rvy_{0}, \rvy_{1} \indep \rvt \lvert \rvbz $. 
\end{definition}

In the above definition, the potential outcomes $\rvy_{0}$ and $\rvy_{1}$, the observed treatment $\rvt$, and the features $\rvbz$ are all deterministic functions of the units $\rvbx^{(u)}$. Therefore, the conditional independence criterion makes sense over the common probability measurable in the space of the units $\rvbx^{(u)}$. As mentioned in Section \ref{section:intro}, ignorability cannot be tested for from observational data since  for every observational sample either $\rvy_{0}$ or $\rvy_{1}$ is observed (and not both).
% The \textit{Average Treatment Effect} (ATE) is defined 

In the PO framework, the ATE is defined as $\Expectation_{\rvbx^{(u)}}[\rvy_{1} - \rvy_{0}]$. When $\rvbz \subseteq \rvbx^{(o)}$ is ignorable, it is also a valid adjustment relative to $(\rvt, \rvy)$ in $\cG$ and therefore the ATE can be estimated by regressing on $\rvbz$.
% If $\rvbx_{\cA} \subseteq \rvbx^{(o)}$ satisfies the ignorability condition i.e., $\rvbx_{\cA}$ is ignorable, the ATE can be estimated by regressing the factual outcomes for the treated and untreated sub-populations on $\rvbx_{\cA}$ i.e., ATE $ = \Expectation_{\rvbx_{\cA}} [ \Expectation_{\rvy} [ \rvy| \rvt=1, \rvbx_{\cA}] - \Expectation_{\rvy} [ \rvy| \rvt=0, \rvbx_{\cA}] ]$. 
% When $\rvbx_{\cA}$ is ignorable, it is also a \textit{valid adjustment set}.

The Pearlian framework provides a generative model for this setup i.e., a semi-Markovian model (specifying a DAG that encodes causal assumptions relating exogenous and observed variables) as well as specifies graphical criterions that imply existence of valid adjustments relative to $(\rvt, \rvy)$ in $\cG$. 
% which are valid adjustments. 
% Next, we discuss the \textit{backdoor} criterion -- a popular sufficient graphical criterion.

\section{D-separation}
\label{appendix:d_separation}
In this section, we define d-separation with respect to a semi-Markovian DAG $\cG$. The d-separation or directed-separation is a commonly used graph separation criterion that characterizes conditional independencies in DAGs. First, we will define the notion of a path.

For any positive integer $k$, let $[k] \coloneqq \{1,\cdots, k\}$.
% \begin{definition}(Path)
% A path $\cP(\rvx_1,\rvx_k)$ is an ordered sequence 
% of distinct nodes $\rvx_1 \ldots \rvx_k$ such that for any $i \in [k]$, $\rvx_i \in \rvbx^{(o)}$ and for any $i \in [k-1]$, either $\rvx_i \myrightarrow \rvx_{i+1}, ~\rvx_i \myleftarrow \rvx_{i+1}, ~\rvx_i \doublearrow  \rvx_{i+1}, ~\rvx_i \doubleplusrightarrow  \rvx_{i+1}$ or $~\rvx_i \doubleplusleftarrow  \rvx_{i+1}$ exists in $\cG$.
% \end{definition}
\begin{definition}(Path)
A path $\cP(\rvv_1,\rvv_k)$ is an ordered sequence 
of distinct nodes $\rvv_1 \ldots \rvv_k$ and the edges between these nodes such that for any $i \in [k]$, $\rvv_i \in \cV$ and for any $i \in [k-1]$, either $\rvv_i \myrightarrow \rvv_{i+1}, ~\rvv_i \myleftarrow \rvv_{i+1}$ or $~\rvv_i \doublearrow  \rvv_{i+1}$.
\end{definition}
For example, in Figure \ref{fig:Gtoy_without_u}, $\cP(\rvx_1,\rvx_3) = \{\rvx_1 \myrightarrow \rvx_2\doublearrow \rvx_3\}$ and $\cP(\rvt,\rvy) = \{\rvt \doublearrow \rvx_1 \myrightarrow \rvx_2 \doublearrow \rvx_3 \doublearrow \rvy\}$ are two distinct paths.
Next, we will define the notion of a collider.
\begin{definition}
(Collider) In a path $\cP(\rvv_1,\rvv_k)$, for any $i \in \{2,\cdots, k-1\}$, a collider at $\rvv_i$ mean that the arrows (or edges) meet head-to-head (collide) at $\rvv_i$ i.e. either $ \rvv_{i-1} \myrightarrow \rvv_i \myleftarrow \rvv_{i+1}$, $ \rvv_{i-1} \doublearrow \rvv_i \myleftarrow \rvv_{i+1}$, $ \rvv_{i-1} \myrightarrow \rvv_i \doublearrow \rvv_{i+1}$ or $ \rvv_{i-1} \doublearrow \rvv_i \doublearrow \rvv_{i+1}$.
\end{definition}
For example, in $\cP(\rvx_1,\rvx_3)$ defined above, there is a collider at $\rvx_2$.
Next, we define the notion of a descendant path.
\begin{definition}(Descendant path)
A path $\cP(\rvv_1,\rvv_k)$ is said to be an descendant path from $\rvv_1$ to $\rvv_k$ if $~\forall i \in [k-1]$, $\rvv_i \myrightarrow \rvv_{i+1}$. 
 \end{definition}
 For example, in Figure \ref{fig:Gtoy_without_u}, $\cP(\rvx_1,\rvy) = \{\rvx_1 \myrightarrow \rvx_2 \myrightarrow \rvy\}$ is an descendant path from $\rvx_1$ to $\rvy$.
Next, we define the notion of a descendant.
\begin{definition}(Descendant)
A variable $\rvv_k$ is a descendant of a variable $\rvv_1$ if there exists an descendant path $\cP(\rvv_1,\rvv_k)$ from $\rvv_1$ to $\rvv_k$.
 \end{definition}
 For example, in Figure \ref{fig:Gtoy_without_u}, $\rvy$ is a descendant of $\rvx_1$.

\begin{definition}
(Blocking path) For any variables $\rvv_1, \rvv_2 \in \cW$, a set $\rvbv \subseteq \cW$, and a path $\cP(\rvv_1,\rvv_2)$, $\rvbv$ blocks the path $\cP(\rvv_1,\rvv_2)$ if there exists a variable $\rvv$ in the path $\cP(\rvv_1,\rvv_2)$ that satisfies either of the following two conditions:
\begin{enumerate}
    \item[(1)]  $\rvv \in \rvbv$ and $\rvv$ is not a collider.
    \item[(2)] neither variable $\rvv$ nor any of it’s descendant is in $\rvbv$; and $\rvv$ is a collider.
\end{enumerate}
\end{definition}
 For example, in Figure \ref{fig:Gtoy_without_u}, $\{\rvx_2\}$ blocks the path $\cP(\rvx_1,\rvy) = \{\rvx_1 \myrightarrow \rvx_2 \myrightarrow \rvy\}$ because $\rvx_2 \in \{\rvx_2\}$ and $\rvx_2$ is not a collider. Further, $\{\rvx_2\}$ also blocks the path $\cP(\rvx_1,\rvy) = \{\rvx_1 \myrightarrow \rvx_2 \doublearrow \rvx_3 \doublearrow \rvy\}$ because $\rvx_3 \notin \{\rvx_2\}$ and $\rvx_3$ is a collider..

\begin{definition}
(D-separation) For any variables $\rvv_1, \rvv_2 \in \cW$, and a set $\rvbv \subseteq \cW$, $\rvv_1$ and $\rvv_2$ are d-separated by $\rvbv$ in $\cG$ if $\rvbv$ blocks every path between $\rvv_1$ and $\rvv_2$ in $\cG$.
% (D-separation) For any variables $\rvw_1, \rvw_2 \in \cW$, and a set $\rvbw \subseteq \cW$, $\rvw_1$ and $\rvw_2$ are d-separated by $\rvbw$ in $\cG$ if for every path between $\rvw_1$ and $\rvw_2$ in $\cG$ there exists a node $\rvw$ on that path that satisfies either of the following:
% \begin{enumerate}
%     \item[(1)]  $\rvw \in \rvbw$ and $\rvw$ is not a collider.
%     \item[(2)] neither node $\rvw$ nor any of it’s descendant is in $\rvbw$; and $\rvw$ is a collider.
% \end{enumerate}
% let $\rvw_1 \dsep \rvw_2 |\rvbw$ denote that 
\end{definition}
% We also say that $\rvbw$ blocks a specific path between $\rvw_1$ and $\rvw_2$ if there exists a node $\rvw$ on that path which satisfies either of those two conditions in the definition.
For example, in Figure \ref{fig:Gtoy_without_u}, $\rvx_1$ and $\rvy$ are d-separated by $\{\rvx_2\}$.

\section{Additional notations}
\label{appendix:additional_notations}
In this section, we will look at a few additional notations that will be used in the proofs of Theorem \ref{thm_sufficiency}, Theorem \ref{thm_necessity}, and Corollary \ref{corollary:backdoor}. 
\subsection{\texorpdfstring{$\cG_{-\rvt}$}{}}

Often it is favorable to think of the back-door criterion in terms of the graph obtained by removing the edge from $\rvt$ to $\rvy$ in $\cG$. Let $\cG_{-\rvt}$ denote this graph. The following (well-known) remark connects the back-door criterion to $\cG_{-\rvt}$.
% Let $\cG_{-\rvt}$ denote the graph obtained by removing the edge from $\rvt$ to $\rvy$ in $\cG$.
\begin{restatable}{remark}{remarkOne}
\label{remark1}
Under Assumption \ref{assumption1}, a set of variables $\rvbz \subseteq \rvbx$ satisfies the back-door criterion relative to the ordered pair of variables $(\rvt, \rvy)$ in $\cG$ if and only if $\rvt$ and $\rvy$ are d-separated by $\rvbz$ in $\cG_{-\rvt}$. 
% i.e., $\rvt \dsep \rvy | \rvbz $ in $\cG_{-\rvt}$.
\end{restatable}
% Although well known, we provide a proof of Remark \ref{remark1} in Appendix \ref{appendix:remark1}. 
% For simplicity, we keep our theorem statements in Section \ref{section:main_results} free of $\cG_{-\rvt}$. However, $\cG_{-\rvt}$ is crucial for the proofs of the theorems.

% In this section, we prove Remark \ref{remark1}. We re-state the remark below and then provide the proof.
% \remarkOne*
\begin{proof}
Under Assumption \ref{assumption1}, $\rvy$ is the only descendant of $\rvt$ i.e., no node in $\rvbx$ is a descendant of $\rvt$. Therefore, from Definition \ref{definition_back_door}, $\rvbz$ satisfying the back-door criterion relative to $(\rvt, \rvy)$ in $\cG$ is equivalent to $\rvbz$ blocking every path between $\rvt$ and $\rvy$ in $\cG$ that contains an arrow into $\rvt$. Further, under Assumption \ref{assumption1}, there are no paths between $\rvt$ and $\rvy$ in $\cG$ that contain an arrow out of $\rvt$ apart from the direct path $\rvt \myrightarrow \rvy$. However, this direct path $\rvt \myrightarrow \rvy$ does not exist in $\cG_{-\rvt}$. Therefore, $\rvbz$ blocking every path between $\rvt$ and $\rvy$ in $\cG$ that contains an arrow into $\rvt$ is equivalent to $\rvbz$ blocking every path between $\rvt$ and $\rvy$ in $\cG_{-\rvt}$. Thus, $\rvbz$ satisfying the back-door criterion relative to $(\rvt, \rvy)$ in $\cG$ is equivalent to $\rvbz$ blocking every path between $\rvt$ and $\rvy$ in $\cG_{-\rvt}$ i.e., $\rvt \dsep \rvy | \rvbz $ in $\cG_{-\rvt}$.
\end{proof}

\subsection{Subset of a path}
Now, we will define the notion of a subset of a path.
% \begin{definition}(Path)
% A path $\cP(\rvx_1,\rvx_k)$ is an ordered sequence 
% of distinct nodes $\rvx_1 \ldots \rvx_k$ such that for any $i \in [k]$, $\rvx_i \in \rvbx^{(o)}$ and for any $i \in [k-1]$, either $\rvx_i \myrightarrow \rvx_{i+1}, ~\rvx_i \myleftarrow \rvx_{i+1}, ~\rvx_i \doublearrow  \rvx_{i+1}, ~\rvx_i \doubleplusrightarrow  \rvx_{i+1}$ or $~\rvx_i \doubleplusleftarrow  \rvx_{i+1}$ exists in $\cG$.
% \end{definition}
% \begin{definition}(Path)
% A path $\cP(\rvx_1,\rvx_k)$ is an ordered sequence 
% of distinct nodes $\rvx_1 \ldots \rvx_k$ and the edges between these nodes such that for any $i \in [k]$, $\rvx_i \in \rvbx^{(o)}$ and for any $i \in [k-1]$, either $\rvx_i \myrightarrow \rvx_{i+1}, ~\rvx_i \myleftarrow \rvx_{i+1}$ or $~\rvx_i \doublearrow  \rvx_{i+1}$.
% \end{definition}
% For example, in Figure \ref{fig:Gtoy_without_u}, $\cP(\rvx_1,\rvx_3) = \{\rvx_1 \myrightarrow \rvx_2\doublearrow \rvx_3\}$ and $\cP(\rvt,\rvy) = \{\rvt \doublearrow \rvx_1 \myrightarrow \rvx_2 \doublearrow \rvx_3 \doublearrow \rvy\}$ are two distinct paths.

\begin{definition}(Subset of a path)
A path $\cP'(\rvy_1,\rvy_j)$ is said to be a subset of the path $\cP(\rvx_1,\rvx_k)$ (denoted by $\cP'(\rvy_1,\rvy_j) \subset \cP(\rvx_1,\rvx_k)$) if $j < k$, $\exists ~i \in [k+1-j]$ such that $\rvx_i = \rvy_1$, $\rvx_{i+1} = \rvy_2, \cdots, \rvx_{i+j-1} = \rvy_j$ and the edge between $\rvx_{i+l-1}$ and $\rvx_{i+l}$ is same as the edge between $\rvy_{l}$ and $\rvy_{l+1}$ $\forall l \in [j-1]$.
\end{definition}
For example, in Figure \ref{fig:Gtoy_without_u}, $\cP(\rvx_1,\rvx_3) = \{\rvx_1 \myrightarrow \rvx_2\doublearrow \rvx_3\}$ is a subset of the path 
$\cP(\rvt,\rvy) = \{\rvt \doublearrow \rvx_1 \myrightarrow \rvx_2 \doublearrow \rvx_3 \doublearrow \rvy\}$ i.e., $\cP(\rvx_1,\rvx_3) \subset \cP(\rvt,\rvy)$. For a path $\cP(\rvx_1, \rvx_k)$, it is often convenient to represent the subset obtained by removing the nodes at each extreme and the corresponding edges by $\cP(\rvx_1, \rvx_k) \setminus \{\rvx_1, \rvx_k\}$. For example, $\cP(\rvx_1,\rvx_3) = \cP(\rvt,\rvy) \setminus \{\rvt, \rvy\}$.
% in Figure \ref{fig:Gtoy_without_u}.

\section{Proof of Theorem \ref{thm_sufficiency} and an illustrative example}
\label{appendix:proof_thm_sufficiency}
In this section, we will prove Theorem \ref{thm_sufficiency} and also provide an illustrative example for Theorem \ref{thm_sufficiency}. 
Recall the notions of \textit{path}, \textit{collider}, \textit{descendant path}, \textit{blocking path} and \textit{d-separation} from Appendix \ref{appendix:d_separation}. Also, recall the notions of \textit{subset of a path} and $\cG_{-\rvt}$ as well as Remark \ref{remark1} from Appendix \ref{appendix:additional_notations}.

\subsection{Proof of Theorem \ref{thm_sufficiency}}
We re-state the Theorem below and then provide the proof.\footnote{We say that $\rvbz$ satisfies the backdoor criterion if it blocks all the backdoor paths between $\rvt$ and $\rvy$ in $\cG$ i.e., paths between $\rvt$ and $\rvy$ in $\cG$ that contains an arrow into $\rvt$. Please see Definition \ref{definition_back_door} in the main paper.}
\thmSufficiency*
\begin{proof}
We will prove this by contradiction. Suppose $\rvbz$ does not satisfy the back-door criterion relative to $(\rvt, \rvy)$ in $\cG$. From Remark \ref{remark1}, under Assumption \ref{assumption1}, this is equivalent to 
% $\rvt \notdsep \rvy | \rvbz$ in $\cG_{-\rvt}$ i.e., 
$\rvt$ and $\rvy$ not being d-separated by $\rvbz$ in $\cG_{-\rvt}$. This is further equivalent to saying that there exists at least one unblocked path (not containing the edge $\rvt \myrightarrow \rvy$) from $\rvt$ to $\rvy$ in $\cG$ when $\rvbz$ is conditioned on. Let $\cP (\rvt, \rvy)$ denote the shortest of these unblocked paths. We have the following two scenarios depending on whether or not $\cP (\rvt, \rvy)$ contains $\rvx_{\rvt}$. First, we will show that in both of these cases there exists an unblocked path\footnote{Note: There is no possibility of an unblocked path from $\rvx_{\rvt} $ to $\rvy$ in $\cG$ containing the edge $\rvt \myrightarrow \rvy$ when $\rvbz, \rvt$ are conditioned on. This is because $\rvt$ is conditioned on and any such path to $\rvy$ cannot form a collider at $\rvt$.} $\cP'(\rvx_{\rvt},\rvy)$ from $\rvx_{\rvt} $ to $\rvy$ in $\cG$ when $\rvbz, \rvt$ are conditioned on.

% We will show that in both of these cases there exists an unblocked back-door path \footnote{There is no possibility of a front-door path here because $\rvt$ is conditioned on and front-door paths to $\rvy$ cannot form a v-structure at $\rvt$.} from $\rve $ to $\rvy$ in $\cG$ when $\rvbz, \rvt$ are conditioned. 

Note : All bi-directed edges in $\cG$ are unblocked because (a) none of the unobserved feature is conditioned on and (b) there is no collider at any of the unobserved feature.

\begin{enumerate}[leftmargin=5mm]
    \item[(i)] $\rvx_{\rvt} \in \cP (\rvt, \rvy)$: This implies that there is an unblocked path  $\cP''(\rvx_{\rvt},\rvy) \subset \cP (\rvt, \rvy)$ from $\rvx_{\rvt} $ to $ \rvy$ in $\cG$ when $\rvbz$ is conditioned on.
    % In this case, the path $\cP$ has to be of the form either $\rvt \leftrightarrow \rvx_{\rvt} \myrightarrow \cdots \rvy$ or $\rvt \myleftarrow \rvx_{\rvt} \cdots \rvy$ because $\rvx_{\rvt}$ is not conditioned on and $\cP$ is the shortest path. 
    Suppose we now condition on $\rvt$ in addition to $\rvbz$. The conditioning on $\rvt$ can affect the path $\cP''(\rvx_{\rvt},\rvy)$ only\footnote{$\rvt \notin \cP''(\rvx_{\rvt},\rvy)$ because $\cP (\rvt, \rvy)$ is the shortest unblocked path (not containing the edge $\rvt \myrightarrow \rvy$) from $\rvt$ to $\rvy$ in $\cG$ when $\rvbz$ is conditioned on.} if a) there is an unblocked descendant path from some $\rvx_s \in \cP''(\rvx_{\rvt},\rvy) \setminus \{\rvx_{\rvt}, \rvy\}$ to $\rvt$ and b) $\rvx_s$ is a collider in the path $\cP''(\rvx_{\rvt},\rvy) \setminus \{\rvx_{\rvt}, \rvy\}$ . However, conditioning on such a $\rvt$ cannot block the path $\cP''(\rvx_{\rvt},\rvy)$. Thus, there exists an unblocked path $\cP'(\rvx_{\rvt},\rvy) = \cP''(\rvx_{\rvt},\rvy)$ in $\cG$ when $\rvbz, \rvt$ are conditioned on.
    \item[(ii)] $\rvx_{\rvt} \notin \cP (\rvt, \rvy)$: Under Assumption \ref{assumption1}, $\cG$ cannot contain the edge $\rvt \myleftarrow \rvy$ (because a DAG cannot have a cycle). Furthermore, under Assumption \ref{assumption1}, $\rvt$ has no child other than $\rvy$. Therefore, in this case, the path $\cP (\rvt, \rvy)$ takes one of the following two forms : (a) $\rvt \myleftarrow \rvx_s \cdots \rvy$ or (b) $\rvt \doublearrow \rvx_s \cdots \rvy$ for some 
    % $\rvx_s \in \pao(\rvt) \cup \doublepi(\rvt)$ such that 
    $\rvx_s \neq \rvx_{\rvt}$. In either case, there is a collider at $\rvt$ 
    (i.e., either $\rvx_{\rvt} \myrightarrow \rvt \myleftarrow \rvx_s$, $\rvx_{\rvt} \myrightarrow \rvt \doublearrow \rvx_s$, $\rvx_{\rvt} \doublearrow \rvt \myleftarrow \rvx_s$ or $\rvx_{\rvt} \doublearrow \rvt \doublearrow \rvx_s$) 
    in the path $\cP'''(\rvx_{\rvt}, \rvx_s)$ from $\rvx_{\rvt}$ to $\rvx_s$. 
    %(i.e., either $\rvx_{\rvt} \myrightarrow \rvt \myleftarrow \rvx_s$, $\rvx_{\rvt} \myrightarrow \rvt \leftrightarrow \rvx_s$, $\rvx_{\rvt} \leftrightarrow \rvt \myleftarrow \rvx_s$ or $\rvx_{\rvt} \leftrightarrow \rvt \leftrightarrow \rvx_s$) 
     Suppose we now condition on $\rvt$ in addition to $\rvbz$. The conditioning on $\rvt$ unblocks the path $\cP'''(\rvx_{\rvt}, \rvx_s)$ because there is a collider at $\rvt$.
     Also, similar to the previous case, the conditioning on $\rvt$ cannot block the path $\cP(\rvt, \rvy)$ from $\rvt$ to $\rvy$ (passing through $\rvx_s$). Therefore, we see that there is an unblocked path $\cP'(\rvx_{\rvt},\rvy)$ from $\rvx_{\rvt}$ to $\rvy$ (passing through $\rvt$ and $\rvx_s$) in $\cG$ when $\rvbz, \rvt$ are conditioned on
    (i.e., either $\rvx_{\rvt} \myrightarrow \rvt \myleftarrow \rvx_s \cdots \rvy$, $\rvx_{\rvt} \myrightarrow \rvt \doublearrow \rvx_s \cdots \rvy$, $\rvx_{\rvt} \doublearrow \rvt \myleftarrow \rvx_s \cdots \rvy$ or $\rvx_{\rvt} \doublearrow \rvt \doublearrow \rvx_s \cdots \rvy$).
    % \item[(ii)] $\rvx_{\rvt} \notin \cP (\rvt, \rvy)$: Under Assumption \ref{assumption1}, $\cG$ cannot contain the edge $\rvt \myleftarrow \rvy$ (because a DAG cannot have a cycle). Furthermore, under Assumption \ref{assumption1}, $\rvt$ has no child other than $\rvy$. Therefore, in this case, the path $\cP (\rvt, \rvy)$ is of the form $\rvt \myleftarrow \rvx_s \cdots \rvy$ for some $\rvx_s \in \pa(\rvt) \subseteq \rvbx$ such that $\rvx_s \neq \rvx_{\rvt}$. There is a collider at $\rvt$ in the path $\cP''(\rvx_{\rvt}, \rvx_s)$ from $\rvx_{\rvt}$ to $\rvx_s$ (i.e., either $\rvx_{\rvt} \myrightarrow \rvt \myleftarrow \rvx_s$, $\rvx_{\rvt} \doublearrow \rvt \myleftarrow \rvx_s$).
    %  Suppose we now condition on $\rvt$ in addition to $\rvbz$. The conditioning on $\rvt$ unblocks the path $\cP''(\rvx_{\rvt}, \rvx_s)$ because there is a collider at $\rvt$ \footnote[3]{The path $\rvx_{\rvt} \doublearrow \rvt$ is unblocked because no variable in $\cU$ is conditioned on.}.
    %  Also, similar to the previous case, the conditioning on $\rvt$ cannot block the path $\cP(\rvt, \rvy)$ from $\rvt$ to $\rvy$ (passing through $\rvx_s$). Therefore, we see that there is an unblocked path $\cP'(\rvx_{\rvt},\rvy)$ from $\rvx_{\rvt}$ to $\rvy$ (passing through $\rvt$ and $\rvx_s$) in $\cG$ when $\rvbz, \rvt$ are conditioned on
    % (i.e., either $\rvx_{\rvt} \myrightarrow \rvt \myleftarrow \rvx_s \cdots \rvy$ or $\rvx_{\rvt} \doublearrow \rvt \myleftarrow \rvx_s \cdots \rvy$).
    
\end{enumerate}
 Now, in each of the above cases, there is an edge from $\rvx_{\rvt}$ to $\rve$ because $\rve$ is sub-sampled using $\rvx_{\rvt}$.  Therefore, there exists an unblocked path $\cP''''(\rve,\rvy) \supset \cP'(\rvx_{\rvt},\rvy)$ of the form $\rve \myleftarrow \rvx_{\rvt} \cdots \rvy$ in $\cG$ when $\rvbz, \rvt$ are conditioned on because $\rvx_{\rvt} \notin \rvbz$ i.e., $\rvx_{\rvt}$ is not conditioned on. This is true regardless of whether $\rvx_{\rvt}$ is an ancestor of $\rvbz$ or not since the edge $\rve \myleftarrow \rvx_{\rvt}$ cannot create a collider at $\rvx_{\rvt}$. The existence of the path $\cP''''(\rve,\rvy)$ contradicts the fact that $\rve$ is d-separated from $\rvy$ by $\rvbz$ and $\rvt$ in $\cG$.
%  $\rve \dsep \rvy | \rvbz, \rvt$ in $\cG$ 
 This completes the proof.
\end{proof}

\subsection{An illustrative example for Theorem \ref{thm_sufficiency}}
\label{appendix:illustrative_example}
Now, we will look into an example illustrating Theorem \ref{thm_sufficiency}.
Consider the DAG $\cG^{bi}$ in Figure \ref{fig:G_41}. 
We let $\rvx_{\rvt} = \rvx_1$ (because $\rvx_1 \doublearrow \rvt$) and sub-sample $\rve$ using $\rvx_1$ and $\rvt$ (see Figure \ref{fig:G_41}). For this example, $\rvbz \subseteq \{\rvx_2,\rvx_3\}$ i.e., $\rvbz \in \{\varnothing, \{\rvx_2\}, \{\rvx_3\}, \{\rvx_2, \rvx_3\}\}$. It is easy to verify that $\rve \dsep \rvy | \rvx_2, \rvt$ but $\rve \notdsep \rvy | \rvt$, $\rve \notdsep \rvy | \rvx_3, \rvt$, and $\rve \notdsep \rvy | \rvx_2, \rvx_3 \rvt$ in $\cG^{bi}$. Given these, Theorem \ref{thm_sufficiency} implies that $\rvbz = \{\rvx_2\}$ should satisfy the back-door criterion relative to $(\rvt, \rvy)$ in $\cG^{bi}$. This is indeed the case and can be verified easily. Thus, we see that our framework has the potential to identify
valid adjustment sets ($\{\rvx_2\}$ for $\cG^{bi}$) in the scenario where no causal parent of the treatment variable is known but a bi-directed neighbor of the treatment is known.

Note : Theorem \ref{thm_sufficiency} does not comment on whether $\varnothing, \{\rvx_3\}$ and $\{\rvx_2, \rvx_3\}$ satisfy or do not satisfy the back-door criterion relative to $(\rvt, \rvy)$ in $\cG^{bi}$.

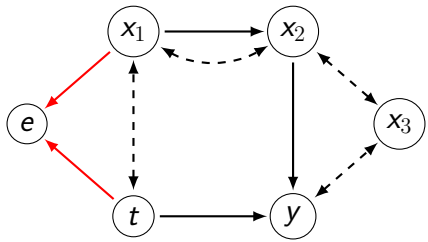
\begin{figure}[ht!]
\centering
% \begin{subfigure}{.75\textwidth}
  \centering
	\begin{tikzpicture}[scale=0.7, every node/.style={transform shape}, > = latex, shorten > = 1pt, shorten < = 2pt]
	\node[shape=circle,draw=black](x1) at (0,2) {\LARGE${\rvx_1}$};
	\node[shape=circle,draw=black](x2) at (3,2) {\LARGE$\rvx_2$};
	\node[shape=circle,draw=black](x3) at (5,0.25) {\LARGE$\rvx_3$};
	\node[shape=circle,draw=black](e) at (-2,0.25) {\LARGE$\rve$};
	\node[shape=circle,draw=black](y) at (3,-1.5) {\LARGE$\rvy$};
	\node[shape=circle,draw=black](t) at (0,-1.5) {\LARGE$\rvt$};
	\path[style=thick][->](x1) edge (x2);
% 	\path[style=thick][->](x1) edge (t);
	\path[style=thick][<->, bend right](x1) [dashed] edge (x2);
	\path[style=thick][->](t) edge (y);		
	\path[style=thick][->](x2) edge (y);	
	\path[style=thick][<->](t) [dashed] edge (x1);
	\path[style=thick][<->](x2) [dashed] edge (x3);
	\path[style=thick][<->](y) [dashed] edge (x3);
	\path[style=thick, color = red][->](t) edge (e);
	\path[style=thick, color = red][->](x1) edge (e);
	\end{tikzpicture}
\caption{The DAG $\cG^{bi}$ where $\rve$ has been sub-sampled using $\rvx_1$ and $\rvt$.}
  \label{fig:G_41}
% \end{subfigure}%
% \caption{A DAG to illustrate Theorem \ref{thm_sufficiency}.}
\end{figure}

\section{Proof of Theorem \ref{thm_necessity}}
\label{appendix:proof_thm_necessity}
In this section, we will prove Theorem \ref{thm_necessity}. 
Recall the notions of \textit{path}, \textit{collider}, \textit{descendant path}, \textit{descendant}, \textit{blocking path} and \textit{d-separation} from Appendix \ref{appendix:d_separation}. Also, recall the notions of \textit{subset of a path} and $\cG_{-\rvt}$ as well as Remark \ref{remark1} from Appendix \ref{appendix:additional_notations}. 

We re-state the Theorem below and then provide the proof.\footnote{We say that $\rvbz$ satisfies the backdoor criterion if it blocks all the backdoor paths between $\rvt$ and $\rvy$ in $\cG$ i.e., paths between $\rvt$ and $\rvy$ in $\cG$ that contains an arrow into $\rvt$. Please see Definition \ref{definition_back_door} in the main paper.}
\thmNecessity*
\begin{proof}
We will prove this by contradiction. Suppose $\rve \notdsep \rvy | \rvbz, \rvt$ in $\cG$ i.e., $\rve$ and $\rvy$ are not \emph{d-separated} by $\rvbz, \rvt$ in $\cG$. In other words, there exists at least one unblocked path from $\rve$ to $\rvy$ in $\cG$ when $\rvbz, \rvt$ are conditioned on. Let $\cP (\rve, \rvy)$ denote the shortest of these unblocked paths. 

Depending on the choice of $\rvbv$, we have the following two cases. In each of this cases, we will show that the path $\cP (\rve, \rvy)$ is of the form $\rve \myleftarrow \rvx_{\rvt} \cdots \rvy$.
\begin{itemize}[leftmargin=*]
    \item $\rvbv = \{\rvt\}$ : $\rve$ is sub-sampled using $\rvt$ and $\rvx_{\rvt}$. Therefore, the path $\cP (\rve, \rvy)$ can take one of the following two forms : (a) $\rve \myleftarrow \rvt \cdots \rvy$ or (b) $\rve \myleftarrow \rvx_{\rvt} \cdots \rvy$. However, $\rvt$ is conditioned on and the path $\rve \myleftarrow \rvt \cdots \rvy$ cannot form a collider at $\rvt$ (because of the edge $\rve \myleftarrow \rvt$). Therefore, the path $\cP (\rve, \rvy)$ cannot be of the form $\rve \myleftarrow \rvt \cdots \rvy$ and has to be of the form $\rve \myleftarrow \rvx_{\rvt} \cdots \rvy$.
    \item $\rvbv = \varnothing$ : $\rve$ is sub-sampled using only $\rvx_{\rvt}$. Therefore, the path $\cP (\rve, \rvy)$ has to be of the form $\rve \myleftarrow \rvx_{\rvt} \cdots \rvy$.
\end{itemize}
Now, observe that there is no collider at $\rvx_{\rvt}$ in the path $\rve \myleftarrow \rvx_{\rvt} \cdots \rvy$ and $\rvx_{\rvt}$ is not conditioned on (because $\rvx_{\rvt} \notin \rvbz$). Therefore, there exists at least one unblocked path from $\rvx_{\rvt}$ to $\rvy$ in $\cG$ when $\rvbz, \rvt$ are conditioned on. Let $\cP' (\rvx_{\rvt}, \rvy) \subset \cP (\rve, \rvy)$ denote the shortest of these unblocked paths from $\rvx_{\rvt}$ to $\rvy$ in $\cG$ when $\rvbz, \rvt$ are conditioned on. The path $\cP' (\rvx_{\rvt}, \rvy)$ cannot contain the edge $\rvt \myrightarrow \rvy$ since $\rvt$ is conditioned on and the path cannot form a collider at $\rvt$ (because of the edge $\rvt \myrightarrow \rvy$).

    We have the following two scenarios depending on whether or not $\cP' (\rvx_{\rvt}, \rvy)$ contains $\rvt$. First, we will show that in both of these cases there exists an unblocked path $\cP''(\rvt,\rvy)$ from $\rvt $ to $\rvy$ (that does not contain the edge $\rvt \myrightarrow \rvy$) in $\cG$ when $\rvbz$ is conditioned on.
    
    Note : All bi-directed edges in $\cG$ are unblocked because (a) none of the unobserved feature is conditioned on and (b) there is no collider at any of the unobserved feature.
    
    \begin{enumerate}[leftmargin=5mm]
        \item[(1)] $\rvt \notin \cP'(\rvx_{\rvt}, \rvy)$: Suppose we now uncondition on $\rvt$ (but still condition on $\rvbz$). We have the following two scenarios depending on whether or not unconditioning on $\rvt$ blocks the path $\cP'(\rvx_{\rvt}, \rvy)$ (while $\rvbz$ is still conditioned on).
        \begin{enumerate}[leftmargin=3mm]
        \item[(i)] Unconditioning on $\rvt$ does not block the path $\cP'(\rvx_{\rvt}, \rvy)$: Consider the path $\cP''(\rvt,\rvy) \supset \cP'(\rvx_{\rvt},\rvy)$ from $\rvt$ to $\rvy$ of the form $\rvt \myleftarrow \rvx_{\rvt} \cdots \rvy$. This path is unblocked in $\cG$ when $\rvbz$ is conditioned on because (a) by assumption the path $\cP'(\rvx_{\rvt}, \rvy)$ is unblocked in $\cG$ when $\rvbz$ is conditioned on and (b) there is no collider at $\rvx_{\rvt}$ in this path (in addition to $\rvx_{\rvt}$ not being conditioned on since $\rvx_{\rvt} \notin \rvbz$). $\cP''(\rvt,\rvy)$ does not contain the edge $\rvt \myrightarrow \rvy$ because $\cP'(\rvx_{\rvt},\rvy)$ does not contain the edge $\rvt \myrightarrow \rvy$.
        \item[(ii)] Unconditioning on $\rvt$ blocks the path $\cP'(\rvx_{\rvt}, \rvy)$ (Refer Figure \ref{fig_case_thm_necessity} for an illustration of this case): We will first create a set $\rvbx_{\cS}$ consisting of all the nodes at which the path $\cP'(\rvx_{\rvt},\rvy)$ is blocked when $\rvt$ is unconditioned on (while $\rvbz$ is still conditioned on).  Define the set $\rvbx_{\cS} \subseteq \rvbx^{(o)}$ such that for any $\rvx_s \in \rvbx_{\cS}$ the following are true: (a) $\rvx_s \in \cP'(\rvx_{\rvt},\rvy) \setminus \{\rvx_{\rvt}, \rvy\}$, (b) the path  $\cP'(\rvx_{\rvt}, \rvy)$ contains a collider at  $\rvx_s$, (c) there is a descendant path $\cP^d(\rvx_s, \rvt)$ from $\rvx_s$ to $\rvt$, (d) the descendant path $\cP^d(\rvx_s, \rvt)$ is unblocked when $\rvbz$ is conditioned on, (e)  $\rvx_s \notin \rvbz$, and (f) there is no unblocked descendant path from $\rvx_s$ to any $\rvx_a \in \rvbz$.
            
        Since the path $\cP'(\rvx_{\rvt},\rvy)$ is blocked when $\rvt$ is unconditioned on (while $\rvbz$ is still conditioned on), we must have that $\rvbx_{\cS} \neq \varnothing$. Let $\rvx_c \in \rvbx_{\cS}$ be that node which is closest to $\rvy$ in the path $\cP'(\rvx_{\rvt},\rvy)$. By the definition of $\rvbx_{\cS}$ and the choice of $\rvx_c$, unconditioning on $\rvt$ cannot block the path $\cP'''(\rvx_c,\rvy) \subset \cP'(\rvx_{\rvt}, \rvy)$ when $\rvbz$ is still conditioned on. Also, by the definition of $\rvbx_{\cS}$, the descendant path $\cP^d(\rvx_c, \rvt)$ from $\rvx_c$ to $\rvt$ is unblocked when $\rvbz$ is conditioned on.
            
        Now consider the path $\cP''(\rvt,\rvy) $ of the form $\rvt \myleftarrow \cdots \myleftarrow \rvx_c \myleftarrow \cdots \rvy$ i.e., $\cP''(\rvt,\rvy) \supset \cP'''(\rvx_c,\rvy)$ and $\cP''(\rvt,\rvy) \supset \cP^d(\rvx_c, \rvt)$. The path $\cP''(\rvt,\rvy) $ is unblocked when $\rvbz$ is conditioned on since (a) $\cP^d(\rvx_c, \rvt)$ is unblocked when $\rvbz$ is conditioned on, (b) $\cP'''(\rvx_c,\rvy)$ is unblocked when $\rvbz$ is conditioned on, and (c) there is no collider at $\rvx_c$ and $\rvx_c$ is not conditioned on since $\rvx_c \notin \rvbz$. Furthermore, $\cP''(\rvt,\rvy)$ does not contain the edge $\rvt \myrightarrow \rvy$ because $\cP'''(\rvx_c,\rvy) \subset \cP'(\rvx_{\rvt},\rvy)$ does not contain the edge $\rvt \myrightarrow \rvy$ and $\cP^d(\rvx_c, \rvt)$ does not contain the edge $\rvt \myrightarrow \rvy$. 
            % ($\cP^d(\rvx_c, \rvt)$ is an ancestor path to $\rvt$).
        \end{enumerate}
        \item[(2)] $\rvt \in \cP'(\rvx_{\rvt}, \rvy)$: 
            In this case, there is an unblocked path $\cP''''(\rvt, \rvy) \subset \cP'(\rvx_{\rvt},\rvy)$ from $\rvt$ to $\rvy$ when $\rvbz, \rvt$ are  conditioned on. There are two sub-cases depending on whether or not unconditioning on $\rvt$ can block the path $\cP''''(\rvt, \rvy)$ (while $\rvbz$ is still conditioned on).
            \begin{enumerate}
                \item[(A)] Unconditioning on $\rvt$ does not block the path $\cP''''(\rvt, \rvy)$ : In this case, by assumption, the path $\cP''(\rvt,\rvy) = \cP''''(\rvt, \rvy)$ in $\cG$ is unblocked when $\rvbz$ is conditioned on. Furthermore, since $\cP''(\rvt,\rvy) \subset \cP'(\rvx_{\rvt},\rvy)$, $\cP''(\rvt,\rvy)$ does not contain the edge $\rvt \myrightarrow \rvy$.
                \item[(B)] Unconditioning on $\rvt$ blocks the path $\cP''''(\rvt, \rvy)$: 
                % In this case, under Assumption \ref{assumption1}, the path $\cP ''''(\rvt, \rvy)$ is of the form $\rvt \myleftarrow \rvx_{t'} \cdots \rvy$ for some $\rvx_{t'} \in \pa(\rvt) \subseteq \rvbx$ such that $\rvx_{t'} \neq \rvx_{\rvt}$.
                Let $\rvx_{t'}$ be the node adjacent to $\rvt$ in the path $\cP''''(\rvt, \rvy)$. 
                Consider the path $\cP'''''(\rvx_{t'}, \rvy) \subset \cP''''(\rvt, \rvy)$. Clearly, $\rvt \notin \cP'''''(\rvx_{t'}, \rvy)$ since the path $\cP'(\rvx_{\rvt}, \rvy)$ was assumed to be the shortest unblocked path from $\rvx_{\rvt}$ to $\rvy$. Therefore, the only way unconditioning on $\rvt$ could block the path $\cP''''(\rvt, \rvy)$ is if it blocked the path $\cP'''''(\rvx_{t'}, \rvy)$. Now, this sub-case is similar to the case (1)(ii) with $\rvx_{\rvt} = \rvx_{t'}$ and $\cP'(\rvx_{\rvt}, \rvy) = \cP'''''(\rvx_{t'}, \rvy)$\footnote{The choice of edge ($\myrightarrow$ or ${~\dashleftarrow \hspace{-3.0mm} \dashrightarrow}$) between $\rvx_{\rvt}$ and $\rvt$ does not matter in (1)(ii). 
                % Further, all bi-directed edges in $\cG$ are unblocked because (a) none of the unobserved feature is conditioned on and (b) there is no collider at any of the unobserved feature.
                }.
                As in (1)(ii), it can be shown that there exists an unblocked path $\cP''(\rvt,\rvy)$ in $\cG$ (that does not contain the edge $\rvt \myrightarrow \rvy$) when $\rvbz$ is conditioned on.
            \end{enumerate}
    \end{enumerate}
    \begin{figure}[ht!]
  \centering
	\begin{tikzpicture}[scale=0.7, every node/.style={transform shape}, > = latex, shorten > = 1pt, shorten < = 2pt]
	\node[shape=circle,draw=black](xt) at (0,2) {\LARGE$\rvx_{\rvt}$};
	\node[shape=rectangle,draw=none](xtdummy1) at (2,2) {};
	\node[shape=rectangle,draw=none](xtdummy2) at (1.75,2) {};
	\node[shape=circle,draw=black](xs) at (3,2) {\LARGE$\rvx_c$};
	\node[shape=rectangle,draw=none](ydummy1) at (3,1) {};
	\node[shape=rectangle,draw=none](ydummy2) at (3,0.75) {};
	\node[shape=rectangle,draw=none](a1) at (2,0.83){};
	\node[shape=rectangle,draw=none](a2) at (1,-0.33){};
	\node[shape=rectangle,draw=none](a3) at (2.1,0.94){};
	\node[shape=rectangle,draw=none](a4) at (0.9,-0.44){};
	\node[shape=circle,draw=black](e) at (-2,0.25) {\LARGE$\rve$};
	\node[shape=circle,draw=black](y) at (3,-1.5) {\LARGE$\rvy$};
	\node[shape=circle,draw=black](t) at (0,-1.5) {\LARGE$\rvt$};
	\path[style=thick][->](xtdummy2) edge (xs);
	\path[style=thick][-](xt) [dotted] edge (xtdummy1);
	\path[style=thick][->](ydummy2) edge (xs);
	\path[style=thick][-](y) [dotted] edge (ydummy1);
% 	\path[style=thick][<->, bend right](x1) [dashed] edge (x2);
	\path[style=thick][->](t) edge (y);		
% 	\path[style=thick][<-](xs) [dotted] edge (y);
	\path[style=thick][->](xt) edge (t);
% 	\path[style=thick][<->](x2) [dashed] edge (x3);
% 	\path[style=thick][<->](y) [dashed] edge (x3);
	\path[style=thick, color = red][->](t)  edge (e);
	\path[style=thick, color = red][->](xt)  edge (e);
	\path[style=thick][->](xs)  edge (a1);
	\path[style=thick][->](a2)  edge (t);
	\path[style=thick][-](a3) [dotted]  edge (a4);
	\end{tikzpicture}
\caption{Illustrating the case (1)(ii) in the proof of Theorem \ref{thm_necessity}}
\label{fig_case_thm_necessity}
\end{figure}
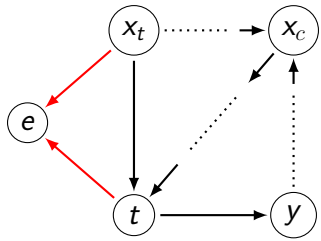

     Now, in each of the above cases, there exists an unblocked path $\cP''(\rvt,\rvy)$ in $\cG$ when $\rvbz$ is conditioned on and this path  does not contain the edge $\rvt \myrightarrow \rvy$. Therefore, there exists an unblocked path $\cP''(\rvt,\rvy)$ in $\cG_{-\rvt}$ when $\rvbz$ is conditioned on (since $\cP''(\rvt,\rvy)$ does not contain the edge $\rvt \myrightarrow \rvy$) implying $\rvt \notdsep \rvy | \rvbz$ in $\cG_{-\rvt}$. From Remark \ref{remark1}, under Assumption \ref{assumption1}, this is equivalent to  
     $\rvbz$ not satisfying the back-door criterion relative to $(\rvt, \rvy)$ in $\cG$ leading to a contradiction. This completes the proof.
\end{proof}

\section{The M-bias model}
\label{appendix:m_bias}
In this section, we discuss the M-bias problem. It is a causal model under which although some observed features (that are pre-treatment) are provided, one must not adjust for any of it. This model has been widely discussed \citep{imbens2020potential, liu2012implications} in the literature to underscore the need for algorithms that find valid adjustment sets.

We illustrate the M-bias problem using the semi-Markov model (with the corresponding DAG $\cG^{M}$) in Figure \ref{fig:M_bias}. 

\begin{figure}[ht!]
  \centering
	\begin{tikzpicture}[scale=0.7, every node/.style={transform shape}, > = latex, shorten > = 1pt, shorten < = 2pt]
	\node[shape=circle,draw=black](x1) at (1.5,2) {\LARGE$\rvx_1$};
	\node[shape=circle,draw=black](y) at (3,-1.5) {\LARGE$\rvy$};
	\node[shape=circle,draw=black](t) at (0,-1.5) {\LARGE$\rvt$};
	\path[style=thick][->](t) edge (y);		
	\path[style=thick][<->, bend right](x1) [dashed] edge (t);
	\path[style=thick][<->, bend left](x1) [dashed] edge (y);
	\end{tikzpicture}
\caption{The DAG $\cG^{M}$ illustrating the M-bias problem.}
\label{fig:M_bias}
\end{figure}
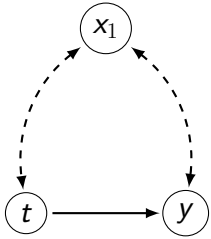

The DAG $\cG^{M}$ consists of the following edges: $\rvt \myrightarrow \rvy, \rvx_1 \doublearrow \rvt, \rvx_1 \doublearrow \rvy$. It is easy to verify that $\{\rvx_1\}$ does not satisfy the back-door criterion with respect to $(\rvt,\rvy)$ in $\cG^{M}$. Further, it is also easy to verify that the empty set i.e., $\varnothing$ satisfies the back-door criterion with respect to $(\rvt,\rvy)$ in $\cG^{M}$. In what follows, we will see how our framework cannot be used to arrive at this conclusion.

There are no observed parents of $\rvt$ in $\cG^{M}$. Therefore, Theorem \ref{thm_necessity} (i.e., the necessary condition) does not apply here. 
For Theorem \ref{thm_sufficiency} to be applicable, there is only one choice of $\rvx_{\rvt}$ i.e., one must use $\rvx_{\rvt} = \rvx_1$.
Now, for any $\rvbv \subseteq \{\rvt\}$ such that $\rve$ is sub-sampled according to $\rve = f(\rvx_1, \rvbv, \eta)$, $\rve$ is not d-separated from $\rvy$ given only $\rvt$. Therefore, one cannot conclude whether or not $\rvbz= \varnothing$ satisfies the back-door criterion with respect to $(\rvt,\rvy)$ in $\cG^{M}$ from Theorem \ref{thm_sufficiency} (i.e., the sufficiency condition).
In summary, we see that our sufficient condition cannot identify the set satisfying the back-door criterion (i.e., the null set) and necessity condition does not apply in the case of the M-bias problem.

Therefore, there are models where sets satisfying the back-door criterion exist (for e.g., the empty set in the M-bias problem) and our results may not be able to identify them.

\section{Finding all back-doors}
\label{appendix:finding_back_doors}

%\noindent{\bf Finding all back-doors.}
Building on Corollary \ref{corollary:backdoor_CI_equivalence}, we provide 
an  Algorithm 
(Algorithm \ref{alg:backdoors})
% in the supplement)
that, when all the parents of the treatment are observed and known, 
% (and no parent has a bi-directed edge to the treatment), 
finds the set of all the subsets of the observed features 
satisfying the 
% that finds all the sets that (are subsets of the observed features and) satisfy the 
back-door criterion relative to $(\rvt, \rvy)$ in $\cG$ which we denote by $\cZ$.
% (under Assumptions \ref{assumption1}, \ref{assumption2}, \ref{assumption3}). 
% A sufficient condition for Algorithm \ref{alg:backdoors} to succeed is that all the parents of the treatment are observed and known. 
We initialize Algorithm \ref{alg:backdoors} with the set $\cZ_1$ obtained by adding $\pi(t)$ to every element of the power set of $\rvbx^{(o)} \setminus \pi(t)$. The set $\cZ_1$ can be constructed easily with the knowledge of $\rvbx^{(o)}$ and $\pi(t)$ provided to Algorithm \ref{alg:backdoors}. Then, we repeatedly apply Corollary \ref{corollary:backdoor_CI_equivalence} to each parent in turn to identify all back-doors. We state this result formally in Corollary \ref{corollary:backdoor} below.
% with a proof in Appendix \ref{appendix:finding_back_doors}. See Appendix \ref{appendix:finding_back_doors} for more discussion on Algorithm \ref{alg:backdoors}.\\

\begin{algorithm}
% \SetCustomAlgoRuledWidth{0.4\textwidth} 
% \DontPrintSemicolon
\KwInput{ $\pi(t), \rve, \rvt, \rvy, \rvbx^{(o)}$}
\KwOutput{$\cZ$}
\KwInitialization{$\cZ = \cZ_1$}
{
    \For{$\rvx_{\rvt} \in \pi(t)$}
    {
    \For{$\rvbz \subseteq \rvbx^{(o)} \setminus \{\rvx_{\rvt}\}$}
    {
    \If{$\rve \indep \rvy | \rvbz, \rvt$}
    {
        $\cZ = \cZ \cup \rvbz$
    }
    }
    }
}
\caption{Finding all back-doors}
\label{alg:backdoors}
\end{algorithm}

%In this section, we provide a discussion on Algorithm \ref{alg:backdoors}, 

%\subsection{Discussion on Algorithm %\ref{alg:backdoors}}
%\label{appendix:algorithm_back-door}

% %\vspace{-3mm}
% \begin{algorithm}[H]
% % \SetCustomAlgoRuledWidth{0.4\textwidth} 
% % \DontPrintSemicolon
% \KwInput{ $\pi(t), \rve, \rvt, \rvy, \rvbx^{(o)}$}
% \KwOutput{$\cZ$}
% \KwInitialization{$\cZ = \cZ_1$}
% {
%     \For{$\rvx_{\rvt} \in \pi(t)$}
%     {
%     \For{$\rvbz \in \rvbx^{(o)} \setminus \rvx_{\rvt}$}
%     {
%     \If{$\rve \indep \rvy | \rvbz, \rvt$}
%     {
%         $\cZ = \cZ \cup \rvbz$
%     }
%     }
%     }
% }
% \caption{Finding all back-doors}
% \label{alg:back-doors}
% \end{algorithm}

\textbf{Remark:} Algorithm \ref{alg:backdoors} is based on two key ideas : (1) Any subset of the observed features that contains all the parents of the treatment satisfies the back-door criterion relative to $(\rvt, \rvy)$ in $\cG$. Formally, consider the set $\cZ_1$ obtained by adding $\pi(t)$ to every element of the power set of $\rvbx^{(o)} \setminus \pi(t)$. Then, any $\rvbz \in \cZ_1$ satisfies the back-door criterion relative to $(\rvt, \rvy)$ in $\cG$. We use the set $\cZ_1$ in the initialization step of Algorithm \ref{alg:backdoors} as it can be constructed easily with the knowledge of $\rvbx^{(o)}$ and $\pi(t)$. (2) For any $\rvbz \notin \cZ_1$ that satisfies the back-door criterion relative to $(\rvt, \rvy)$ in $\cG$, there exists $\rvx_{\rvt} \in \pi(\rvt)$ such that $\rvbz \subseteq \rvbx^{(o)} \setminus \rvx_{\rvt}$. In this scenario, Algorithm \ref{alg:backdoors} captures $\rvbz$ because $\rve \indep \rvy | \rvbz, \rvt$ from Corollary \ref{corollary:backdoor_CI_equivalence} (under Assumption \ref{assumption3}). 

We now provide an example illustrating Algorithm \ref{alg:backdoors}, followed by Corollary \ref{corollary:backdoor} and its proof.
\subsection{Example}
We illustrate Algorithm \ref{alg:backdoors} with an example. Consider the DAG $\cG_{bd}$ in Figure \ref{fig:illustrating_algorithm}. 
It is easy to verify that, for $\cG_{bd}$, $\cZ = \{\{\rvx_3\}, \{\rvx_1,\rvx_3\}, \{\rvx_2,\rvx_3\}, \{\rvx_1,\rvx_2\}, \{\rvx_1,\rvx_2, \rvx_3\}, \{\rvx_1, \rvx_2,\rvx_4\}, \{\rvx_1, \rvx_2,\rvx_3, \rvx_4\} \}$.
Now, Algorithm \ref{alg:backdoors} takes $\pa(\rvt) = \{\rvx_1, \rvx_2\}$ and $\rvbx^{(o)} = \{\rvx_1, \rvx_2, \rvx_3, \rvx_4\}$ as inputs. Therefore, $\cZ_1 = \{\rvx_1,\rvx_2\}, \{\rvx_1,\rvx_2, \rvx_3\}, \{\rvx_1, \rvx_2,\rvx_4\}, \{\rvx_1, \rvx_2,\rvx_3, \rvx_4\}$ can be constructed by adding $\pa(\rvt)$ to every element of the power set of $\rvbx^{(o)} \setminus \pi(t)$ i.e., to the power set of $\{\rvx_3, \rvx_4\}$).  
% It is easy to verify that any $\rvbz \in \cZ_1$ satisfies the back-door criterion relative to $(\rvt, \rvy)$ in $\cG_{bd}$. 
Algorithm \ref{alg:backdoors} is initialized with $\cZ_1$ and the only remaining sets to be identified are $\{\rvx_3\}, \{\rvx_1,\rvx_3\}$, and $\{\rvx_2,\rvx_3\}$. When $\rvx_{\rvt} = \rvx_1$, Algorithm \ref{alg:backdoors} will identify $\{\rvx_3\}$ and $\{\rvx_2,\rvx_3\}$ as sets that satisfy the back-door criterion relative to $(\rvt, \rvy)$ in $\cG_{bd}$. Similarly, when $\rvx_{\rvt} = \rvx_2$, Algorithm \ref{alg:backdoors} will identify $\{\rvx_3\}$ and $\{\rvx_1,\rvx_3\}$ as sets that satisfy the back-door criterion relative to $(\rvt, \rvy)$ in $\cG_{bd}$.
 
 % See Appendix \ref{appendix:proof_corollary_back-door} for the proof of Corollary \ref{corollary:back-door}. The proof outline is as follows. First, we show that any $\rvbz \in \cZ_1$ satisfies the back-door criterion relative to $(\rvt, \rvy)$ in $\cG$. Next, we show that for any $\rvbz \in \cZ \setminus \cZ_1$, there exists $\rvx_{\rvt} \in \pi(\rvt)$ such that $\rvbz \subseteq \rvbx^{(o)} \setminus \rvx_{\rvt}$. Finally, we use Corollary \ref{corollary:back-door_CI_equivalence} to show that Algorithm \ref{alg:back-doors} captures $\rvbz$.
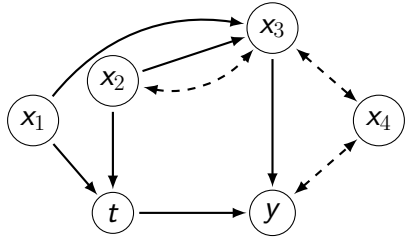
\begin{figure}[h]
  \centering
	\begin{tikzpicture}[scale=0.7, every node/.style={transform shape}, > = latex, shorten > = 1pt, shorten < = 2pt]
	\node[shape=circle,draw=black](x2) at (0,1) {\LARGE$\rvx_2$};
	\node[shape=circle,draw=black](x3) at (3,2) {\LARGE$\rvx_3$};
	\node[shape=circle,draw=black](x4) at (5,0.25) {\LARGE$\rvx_4$};
	\node[shape=circle,draw=black](x1) at (-1.5,0.25) {\LARGE$\rvx_1$};
	\node[shape=circle,draw=black](y) at (3,-1.5) {\LARGE$\rvy$};
	\node[shape=circle,draw=black](t) at (0,-1.5) {\LARGE$\rvt$};
	\path[style=thick][->](x2) edge (x3);
	\path[style=thick][<->, bend right](x2) [dashed] edge (x3);
	\path[style=thick][->](t) edge (y);		
	\path[style=thick][->](x3) edge (y);		
	\path[style=thick][->](x2) edge (t);
	\path[style=thick][<->](x3) [dashed] edge (x4);
	\path[style=thick][<->](y) [dashed] edge (x4);
	\path[style=thick][->](x1) edge (t);
	\path[style=thick][->, bend left](x1)  edge (x3);
	\end{tikzpicture}
%   \caption{$\cG_{bd}$}
%   \label{fig:illustrating_algorithm}
\caption{The DAG $\cG_{bd}$ for illustrating Algorithm \ref{alg:backdoors}}
\label{fig:illustrating_algorithm}
\end{figure}

\subsection{Corollary \ref{corollary:backdoor}}
\label{appendix:proof_corollary_back-door}
Recall the notions of \textit{path}, \textit{collider}, \textit{descendant path}, \textit{descendant}, \textit{blocking path} and \textit{d-separation} from Appendix \ref{appendix:d_separation}. Also, recall the notions of \textit{subset of a path} and $\cG_{-\rvt}$ as well as Remark \ref{remark1} from Appendix \ref{appendix:additional_notations}.
% We re-state the Corollary below and then provide the proof.

\begin{restatable}{corollary}{corrBackdoor}
% \begin{corollary}
\label{corollary:backdoor}
Let Assumptions \ref{assumption1} and \ref{assumption3} be satisfied. Let $\cZ$ be the set of all sets $\rvbz \subseteq \rvbx^{(o)}$ that satisfy the \textit{back-door} criterion relative to the ordered pair of variables $(\rvt, \rvy)$ in $\cG$. 
If all the parents of $\rvt$ are observed and known i.e., $\pi(t) = \pi^{(o)}(t)$ is known, 
% and no $\rvx_{\rvt} \in \pi(t)$ has a bi-directed edge to $\rvt$, 
then Algorithm \ref{alg:backdoors} returns the set $\cZ$.
% Let $\doublepi(t) = \emptyset$. If $\pi^{(u)}(t) =  \emptyset$ (i.e., $\pi(t) = \pi^{(o)}(t)$) and the set $\pi(t)$ is known, Algorithm \ref{alg:backdoors} returns the set $\cZ$.
\end{restatable}
% \corrBackdoor*
\begin{proof}
% We will divide the set $\cZ$ into two disjoint sets as follows. Let $\cZ_1 \subseteq \cZ$ be defined as $\cZ_1 \coloneqq \{\rvbz \in \cZ : \pi^{(o)}(t) \subseteq \rvbz \}$ and let $\cZ_2 = \cZ \setminus \cZ_1$. 
From Remark \ref{remark1}, under assumption \ref{assumption1}, $\rvbz$ satisfying the back-door criterion relative to the ordered pair of variables $(\rvt, \rvy)$ in $\cG$ is equivalent to $\rvt$ and $\rvy$ being d-separated by $\rvbz$ in $\cG_{-\rvt}$ i.e., $\rvt \dsep \rvy | \rvbz$ in $\cG_{-\rvt}$. From \cite{Pearl2016},  $\pi(t)$ always satisfies the back-door criterion relative to the ordered pair of variables $(\rvt, \rvy)$ in $\cG$ i.e., $\rvt \dsep \rvy | \pi(t)$ in $\cG_{-\rvt}$.  Consider any $\rvbz \subseteq \rvbx^{(o)}$ such that $\pi(t) \subseteq \rvbz$. First, we will show that $\rvt \dsep \rvy | \rvbz$ in $\cG_{-\rvt}$ i.e., $\rvbz$ satisfies the back-door criterion relative to the ordered pair of variables $(\rvt, \rvy)$ in $\cG$.

Suppose $\rvt \notdsep \rvy | \rvbz$ in $\cG_{-\rvt}$ i.e., $\rvt$ and $\rvy$ are not d-separated by $\rvbz$ in $\cG_{-\rvt}$. This is equivalent to saying that there exists at least one unblocked path (not containing the edge $\rvt \myrightarrow \rvy$) from $\rvt$ to $\rvy$ in $\cG_{-\rvt}$ when $\rvbz$ is conditioned on. Without the loss of generality, let $\cP(\rvt, \rvy)$ denote any one of these unblocked paths. The path $\cP (\rvt, \rvy)$ has to be of the form $\rvt \myleftarrow \rvx_{\rvt} \cdots \rvy$ where $\rvx_{\rvt} \in \pi(t)$ because (a) under Assumption \ref{assumption1}, $\cG$ cannot contain the edge $\rvt \myleftarrow \rvy$ (because a DAG cannot have a cycle) and (b) under Assumption \ref{assumption1}, $\rvt$ has no child other than $\rvy$.
% , and (c) $\doublepi(t) = \emptyset$. 
However, $\rvx_{\rvt} \in \pi(t) \subseteq \rvbz$ i.e., $\rvx_{\rvt}$ is conditioned on. Now since there is no collider at  $\rvx_{\rvt}$ in the path $\cP(\rvt, \rvy)$, it cannot be unblocked and this leads to a contradiction.  Therefore, $\rvbz$ satisfies the back-door criterion relative to the ordered pair of variables $(\rvt, \rvy)$ in $\cG$.

Now, consider the set $\cZ_1$ obtained by adding $\pi(\rvt)$ to every element of the power set of $\rvbx^{(o)} \setminus \pi(\rvt)$ i.e., $\cZ_1 \coloneqq \{\rvbz \subseteq \rvbx^{(o)} : \pi(t) \subseteq \rvbz\}$. From the argument above, we have $\cZ_1 \subseteq \cZ$. From the knowledge of $\pi(t)$ and $\rvbx^{(o)}$, one can easily construct the set $\cZ_1$ and thus initialize $\cZ$ in Algorithm \ref{alg:backdoors} with $\cZ_1$. 

Now, consider the set $\cZ_2 \coloneqq \cZ \setminus \cZ_1$. Consider any set $\rvbz \in \cZ_2$ satisfying the back-door criterion relative to the ordered pair of variables $(\rvt, \rvy)$ in $\cG$. By the definition of $\cZ_1$ (and $\cZ_2$), there exists
at least one parent of $\rvt$ not present in the set $\rvbz$.
% $\rvx_{\rvt} \in \pi(\rvt)$ such that $\rvx_{\rvt} \notin \rvbz$. 
In other words, there exists $\rvx_{\rvt} \in \pi(\rvt)$ such that $\rvbz \subseteq \rvbx^{(o)} \setminus \rvx_{\rvt}$. From Corollary \ref{corollary:backdoor_CI_equivalence}, under Assumption \ref{assumption3}, this is equivalent to $\rve \indep \rvy | \rvbz, \rvt$. Therefore, Algorithm \ref{alg:backdoors} will capture the set $\rvbz$. Since the choice of $\rvbz$ was random, Algorithm \ref{alg:backdoors} will capture every $\rvbz \in \cZ_2$ and return $\cZ_1 \cup \cZ_2$. This completes the proof.
% Let $
% \psi = \rvbx^{(o)} \setminus \pi^{(o)}(t)$. Consider the power set of $\psi$ i.e., $2^{\psi}$. We will show that the set $\cZ_1 = \pi^{(o)}(t) + 2^{\psi}$ is a subset of $\cZ$ i.e., $\cZ_1 \subseteq \cZ$. 

% Consider any $\rvbz \in \cZ_1$. By definition, $\pi^{(o)}(t) \in \rvbz$. From \cite{Pearl2016},  $\pi^{(o)}(t)$ always satisfies the back-door criterion relative to the ordered pair of variables $(\rvt, \rvy)$ in $\cG$.  

% In other words, $\pi^{(o)}(t)$ blocks every path between $\rvt$ and $\rvy$ that contains an arrow into $\rvt$ (because $\rvy$ is the only descendant of $\rvt$ from Assumptions \ref{assumption1} and \ref{assumption2}). 

% Therefore, $\pi^{(o)}(t) \in \cZ_1$. 
\end{proof}

\section{The baseline}
\label{appendix:baseline}
In this section, we provide an implementation of the \texttt{Baseline} considered in Section \ref{section:experiments}. This routine estimates the ATE from the observational data by regressing $\rvy$ for the treated and the untreated sub-populations on a given set $\rvbz$. The \texttt{Baseline} we consider in this work is an instance of this routine.  
More specifically, for the \texttt{Baseline},
% For the baseline, we compare against 
% involving no adjustment set finding procedure, 
we set $\rvbz$ to be the set of all the observed features i.e., $\rvbz = \rvbx^{(o)}$. See Section \ref{section:experiments} for details.
\begin{algorithm}
\DontPrintSemicolon
\KwInput{ $n, n_r, \rvt, \rvy, \rvbz$}
\KwOutput{$\text{ATE}(\rvbz)$}
% , $p(\rvbz, \cG)$, $p(\rvbz, \cG_{-t})$
\KwInitialization{$\text{ATE}(\rvbz) = 0$}
% , $p(\rvbz, \cG) = 0$, $p(\rvbz, \cG_{-t}) = 0$}
\For(\tcp*[h]{Use a different train-test split in each run}) {$r = 1,\cdots,n_r$} 
{
    $\text{ATE}(\rvbz) = \text{ATE}(\rvbz) + \frac{1}{n} \sum_{i=1}^{n} (\Expectation[\rvy | \rvbz = \svbz^{(i)}, \rvt = 1] - \Expectation[\rvy | \rvbz = \svbz^{(i)}, \rvt = 0])$;
    % $p(\rvbz, \cG) = p(\rvbz, \cG) + \text{CI}(\rve \indep \rvy | \rvbz, \rvt)$\\
    % $p(\rvbz, \cG_{-t}) = p(\rvbz, \cG_{-t}) + \text{CI}(\rvy \indep \rvt | \rvbz)$
}
$\text{ATE}(\rvbz) = \text{ATE}(\rvbz) / n_r$;
% $p(\rvbz, \cG) = p(\rvbz, \cG) / n_r$\\
% $p(\rvbz, \cG_{-t}) = p(\rvbz, \cG_{-t}) / n_r$
\caption{ATE estimation using $\rvbz$ as an adjustment set }\label{alg:ate_generic}
\end{algorithm}

\section{Additional experiments}
\label{appendix:experiments}
In this section, we briefly discuss the usage of real-world CI testers in Algorithm \ref{alg:subset_search}. We also provide in-depth discussions on the synthetic experiment from Section \ref{subssection:synthetic_experiments}, the experiments on IHDP from Section \ref{subsection:ihdp}, and the experiments on Cattaneo from Section \ref{subsection:cattaneo}. Additionally, we specify all the training details, as well as provide more details regarding the comparison of our method with \cite{entner2013data}, \cite{gultchin2020differentiable}, and \cite{cheng2020towards}.

\subsection{Usage of CI testers in Algorithm \ref{alg:subset_search}}
\label{appendix:CI}
In this work we use the RCot real-world CI tester \citep{strobl2019approximate}.

The real-world CI testers produce a p-value close to zero if the CI does not hold and produce a p-value uniformly distributed between 0 and 1 if the CI holds. Since we use a non-zero p-value threshold, depending on the quality of the CI tester, the false positive rate for valid adjustment sets may be non-zero.

Suppose, for a CI tester and for an increasing sample size $n$, we find a sequence of Type-I error rate ($\alpha_n$) and Type-II error rate ($\beta_n$) going to zero i.e., $\alpha_n, \beta_n \myrightarrow 0$. 
% for a CI tester given by $\alpha_n, \beta_n (\alpha_n) \myrightarrow 0$, 
Then, if there is a valid adjustment set, it is easy to see that our algorithm will have zero bias in the estimated effect when the significance threshold $\alpha_n$ is used as the p-value threshold in our algorithm.

\subsection{Synthetic experiment}
\label{appendix:synthetic_experiments}
In this sub-section, we provide more details on the synthetic experiment in Section \ref{subssection:synthetic_experiments}. 

Let $\Uniform(a,b)$ denote the uniform distribution over the interval $[a,b]$ for $a,b \in \Reals$ such that $a<b$. Let $\cN(\mu,\sigma^2)$ denote the Gaussian distribution with mean $\mu$ and variance $\sigma^2$. Let $\Bernoulli(p)$ denote the Bernoulli distribution which takes the value 1 with probability $p$. Let $\Sigmoid(\cdot)$ denote the sigmoid function i.e., for any $a \in \Reals$, $\Sigmoid(a) = 1/1+e^{-a}$. Let $\Softmax(\cdot)$ denote the softmax function.\\

\noindent{\bf Dataset Description.}
% First, we revisit the dataset generation. 
We generate different variables as below:
\begin{itemize}[topsep=2pt,itemsep=-0pt]
    \item $\rvu_i \sim \Uniform(1,2)$
    \item $\rvx_1 \sim \theta_{11}\rvu_1 + \theta_{12}\rvu_2 + \cN(0,0.01)$ where $\theta_{11}, \theta_{12} \in \Uniform(1,2)$
    \item $\rvx_2 \sim \theta_{21}\rvx_1 + \theta_{22}\rvu_2 + \theta_{23}\rvu_3 + \cN(0,0.01)$ where $\theta_{21}, \theta_{22}, \theta_{23} \in \Uniform(1,2)$
    \item $\rvx_3 \sim \theta_{31}\rvu_3 + \theta_{32}\rvu_4 +  \cN(0,0.01)$ where $\theta_{31}, \theta_{12} \in \Uniform(1,2)$
    \item $\rvt \sim \Bernoulli(\Sigmoid(\theta_{51} \rvx_1 + \theta_{52} \rvu_1))$ where $\theta_{51},\theta_{52} \in \Uniform(1,2)$
    \item $\rvy \sim \theta_{41}\rvx_2 + \theta_{42}\rvu_4 + \theta_{43}\rvt + \cN(0,0.01)$ where $\theta_{41}, \theta_{42}, \theta_{43} \in \Uniform(1,2)$
\end{itemize}
% We intervene on $\cG^{toy}$ and remove the edge from $\rvt$ to $\rvy$ to create the DAG $\cG^{toy}_{-\rvt}$.
% \begin{itemize}
%     \item $\rvy_{\cG_{-\rvt}} \sim \theta_{41}\rvx_2 + \theta_{42}\rvu_4 + \cN(0,0.01)$ where $\theta_{41}, \theta_{42}, \theta_{43} \in \Uniform(1,2)$
% \end{itemize}
We generate the weight vectors from $\Uniform(1,2)$ to ensure that the faithfulness assumption with respect to the sub-sampling variable is satisfied (i.e., Assumption \ref{assumption3}). This is because for smaller weights, it is possible that conditionally dependent relations are declared as conditionally independent. See \cite{uhler2013geometry} for details.\\

For all our experiments, we use $3$ environments i.e., $\rve \in \{0,1,2\}$ and generate the sub-sampling variable as below with $\hat{\Expectation}$ denoting the empirical expectation. While other choices of sub-sampling function $f$ could be explored, the natural choice (for discrete $\rve$) of softmax with random weights suffices.
\begin{itemize}[topsep=2pt,itemsep=-0pt]
    \item $\rve \sim \Softmax(\btheta_{61} (\rvx_1 - \hat{\Expectation}[\rvx_1]) + \btheta_{62} (\rvt - \hat{\Expectation}[\rvt]) )$ with $\btheta_{61} \coloneqq (\theta_{61}^{(1)}, \theta_{61}^{(2)}, \theta_{61}^{(3)}) \in \Reals^{3}$ and $\btheta_{62} \coloneqq (\theta_{62}^{(1)}, \theta_{62}^{(2)}, \theta_{62}^{(3)}) \in \Reals^{3}$ such that $\theta_{61}^{(1)}, \theta_{62}^{(1)} \in \Uniform(1,2)$, $\theta_{61}^{(2)} = \theta_{62}^{(2)} = 0$, and $\theta_{61}^{(3)}, \theta_{62}^{(3)} \in \Uniform(-2,-1)$
\end{itemize}
In other words, we keep separation between the weight vectors associated with different environments to make sure that the environments look different from each other as expected by IRM.\\

% \noindent{\bf Results.}
% Recall Figure \ref{fig:toy4_expts} where we validate our algorithm. As is evident, the gains of our algorithm over the baseline are much more in the low dimensions compared to the high dimensions. We believe the conditional independence tester leaks more false positive in high dimensions leading to lower gains. See Appendix \ref{appendix:CI} for more details.

\noindent{\bf Success Probability.}
For a given $p_{value}$ threshold, we let the success probability of the set $\{\rvx_2\}$ be the fraction of times (in $n_r$ runs) the p-value of CI$(\rve \indep \rvy |\rvx_2, \rvt )$ is more than $p_{value}$.
In Figure \ref{fig:syn1_sp} below, we show how the success probability of the set $\{\rvx_2\}$ varies with different $p_{value}$ thresholds i.e., $\{0.1,0.2,0.3,0.4,0.5\}$ for the dataset used in Section \ref{subssection:synthetic_experiments}. As we can see in Figure \ref{fig:syn1_sp}, the success probability of the set $\{\rvx_2\}$, for the same $p_{value}$ threshold, is much lower in high dimensions compared to low dimensions. We believe this happens (a) because of the non-ideal CI tester and (b) because the number of samples are finite. In contrast, our algorithms \texttt{IRM-t} and \texttt{IRM-c} always pick the set $\{\rvx_2\}$ to adjust on i.e., $\rvbz_{\mathsf{irm}} = \{\rvx_2\}$ for both \texttt{IRM-t} and \texttt{IRM-c} for $d = 3,5,7$.\\
% This behavior also justifies the higher gains in low dimensions compared to the high dimensions of our algorithm in Figure \ref{fig:toy4_expts}.

\begin{figure}[ht!]
\centering
\begin{subfigure}{.5\textwidth}
  \centering
  \includegraphics[width=0.9\textwidth]{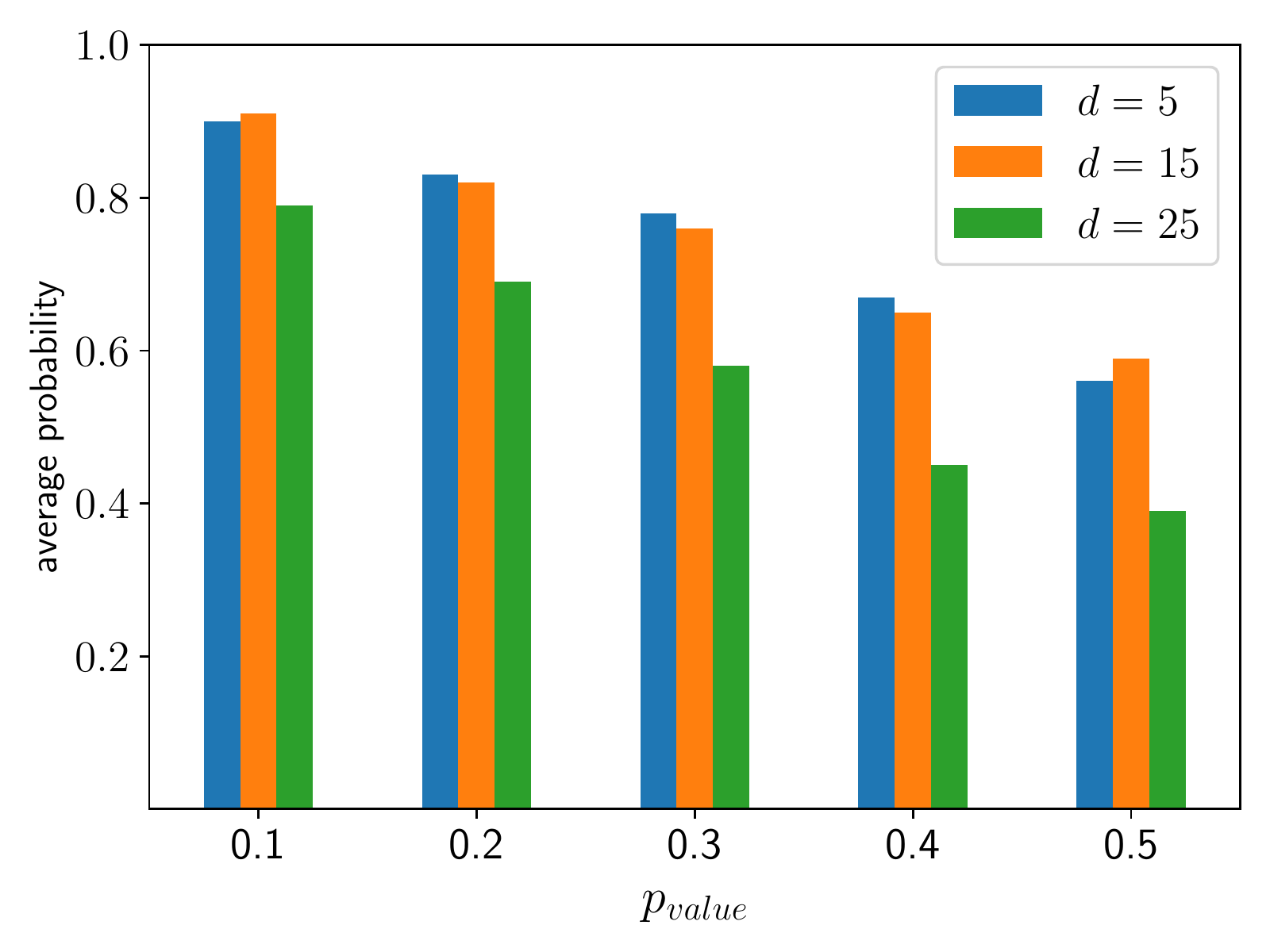}
  \caption{Success probability of the set $\{\rvx_2\}$ in the \\toy
  example $\cG^{toy}$ for different $p_{value}$ thresholds.}
  \label{fig:syn1_sp}
\end{subfigure}%
\begin{subfigure}{.5\textwidth}
  \centering
  \includegraphics[width=0.9\textwidth]{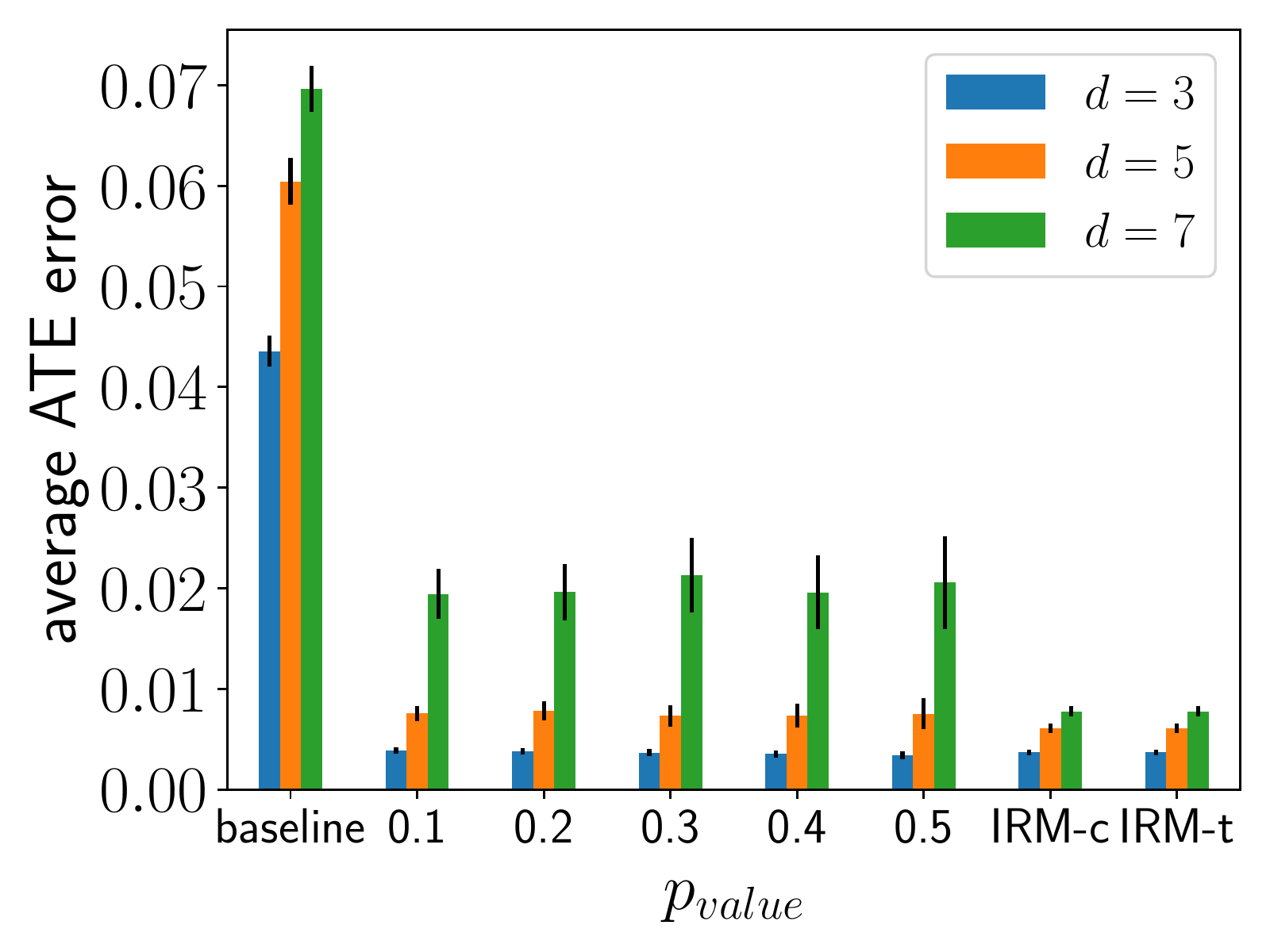}
  \caption{Performance of Algorithm \ref{alg:subset_search} on $\cG^{toy}$ when \\
  the candidate adjustment sets are $\td$-dimensional.}
  \label{fig:toy4_sparse}
\end{subfigure}
\caption{Additional analysis on the toy example $\cG^{toy}$.}
\label{fig:toy_more_plots}
\end{figure}

\noindent{\bf Sparse subset search.} In Section \ref{subssection:synthetic_experiments}, we validated our algorithm by letting $\cX$ be the set of all subsets of $\rvbx^{(o)} \setminus \{\rvx_{\rvt}\}$. However, for this synthetic experiment, we do know that only $\rvx_2 \in \cX$ satisfies the back-door criterion relative to $(\rvt,\rvy)$. Further, we know that $\rvx_2$ is $\td$-dimensional. Therefore, with this additional knowledge, we could instead let $\cX$ be the set of all $\td$-dimensional subsets of $\rvbx^{(o)} \setminus \{\rvx_{\rvt}\}$. 
% instead of letting $\cX$ be the set of all the subsets of $\rvbx^{(o)} \setminus \{\rvx_{\rvt}\}$. 
In other words, we consider the \texttt{Sparse} algorithm from Section \ref{section:experiments} with $k = \td$\footnote{More precisely, the \texttt{Sparse} algorithm considers subsets of size at-most $k$. Here, we consider subsets of size exactly equal to $k$.}. We show the performance of this algorithm for this choice of $\cX$, in comparison to the \texttt{Baseline} (i.e. using all observed features) as well as \texttt{IRM-t} and \texttt{IRM-c}, in Figure \ref{fig:toy4_sparse} for $d = 3,5,7$. With this restriction on the candidate adjustment sets, our algorithm performs better than it does in Figure \ref{fig:toy4_expts} where there are no restrictions on the candidate adjustment sets.\\

\noindent{\bf Performance with dimensions.}
The gains of our testing and subset search based algorithm over the \texttt{Baseline} are much more in the low dimensions compared to the high dimensions as seen in Figures \ref{fig:toy4_expts} and \ref{fig:toy4_sparse}. We believe there are two primary reasons behind this : (a) The CI tester leaks more false positive in high dimensions compared to low dimensions (see Appendix \ref{appendix:CI}) and (b) The CI tester fails to consistently output a high p-value for the set $\{\rvx_2\}$ in high dimensions (see Figure \ref{fig:syn1_sp}). The gains of our IRM based algorithm remain consistent even in high dimensions as expected.

\subsection{Generating the environment/sub-sampling variable}
\label{subsec:app_env}

In all our experiments in Section \ref{section:experiments}, we let the sub-sampling variable depend on both $\rvx_{\rvt}$ and $\rvt$. Now, we will look into the case where the sub-sampling variable is generated as a function of only $\rvx_{\rvt} = \rvx_1$ i.e., $\rve = f(\rvx_1)$. More specifically, we generate the sub-sampling variable as below:
\begin{itemize}[topsep=2pt,itemsep=-0pt]
    \item $\rve \sim \Softmax(\btheta_{61} (\rvx_1 - \hat{\Expectation}[\rvx_1]))$ with $\btheta_{61} \coloneqq (\theta_{61}^{(1)}, \theta_{61}^{(2)}, \theta_{61}^{(3)}) \in \Reals^{3}$ such that $\theta_{61}^{(1)} \in \Uniform(1,2)$, $\theta_{61}^{(2)} = 0$, and $\theta_{61}^{(3)} \in \Uniform(-2,-1)$
\end{itemize}
For this setting, we show the plots analogous to those in Figure \ref{fig:toy3_expts}, Figure \ref{fig:toy4_expts}, Figure \ref{fig:syn1_sp} and Figure \ref{fig:toy4_sparse} in Figure \ref{fig:app_toy_expts}. As we can see in Figure \ref{fig:app_toy3_expts}, Figure \ref{fig:app_toy4_expts}, Figure \ref{fig:app_syn1_sp}, and Figure \ref{fig:app_toy4_sparse}, the performance of our algorithm with $\rve = f(\rvx_{\rvt})$ is similar to (at a high level) its performance with $\rve = f(\rvx_{\rvt}, \rvt)$. This should not be surprising since Corollary \ref{corollary:backdoor_CI_equivalence} holds for any $\rvbv \subseteq \{\rvt\}$ i.e., for both $\rvbv = \varnothing$ and $\rvbv = \{\rvt\}$. In other words, while theoretical tradeoff between the choice of $\rvbv$ i.e., $\varnothing$ or $\{\rvt\}$ is unclear, there is no major empirical difference. Note: We do not show the performance of IRM based algorithms for $\rve = f(\rvx_1)$ since it is exactly the same as the performance for $\rve = f(\rvx_1, \rvt)$.

\begin{figure}[ht!]
\centering
\begin{subfigure}{.5\textwidth}
  \centering
  \includegraphics[width=0.9\textwidth]{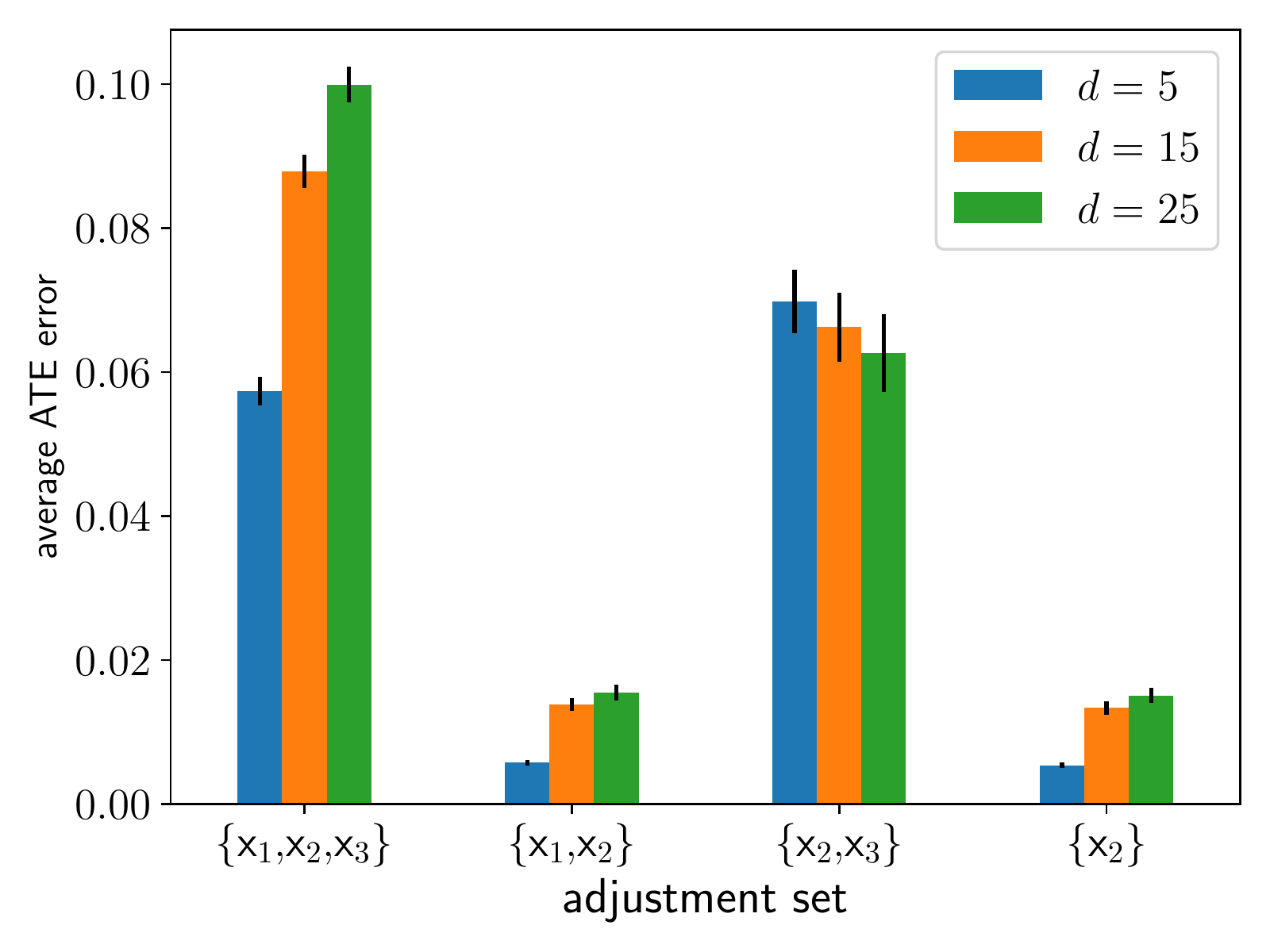}
  \caption{Sets not satisfying back-door ($\{\rvx_1, \rvx_2,\rvx_3\}$, \\ 
  $\{\rvx_2, \rvx_3\}$) result in high ATE error; sets satisfying\\
  back-door ($\{\rvx_1, \rvx_2\}, \{\rvx_2\}$) result in low ATE error.}
  \label{fig:app_toy3_expts}
\end{subfigure}%
\begin{subfigure}{.5\textwidth}
  \centering
  \includegraphics[width=0.9\textwidth]{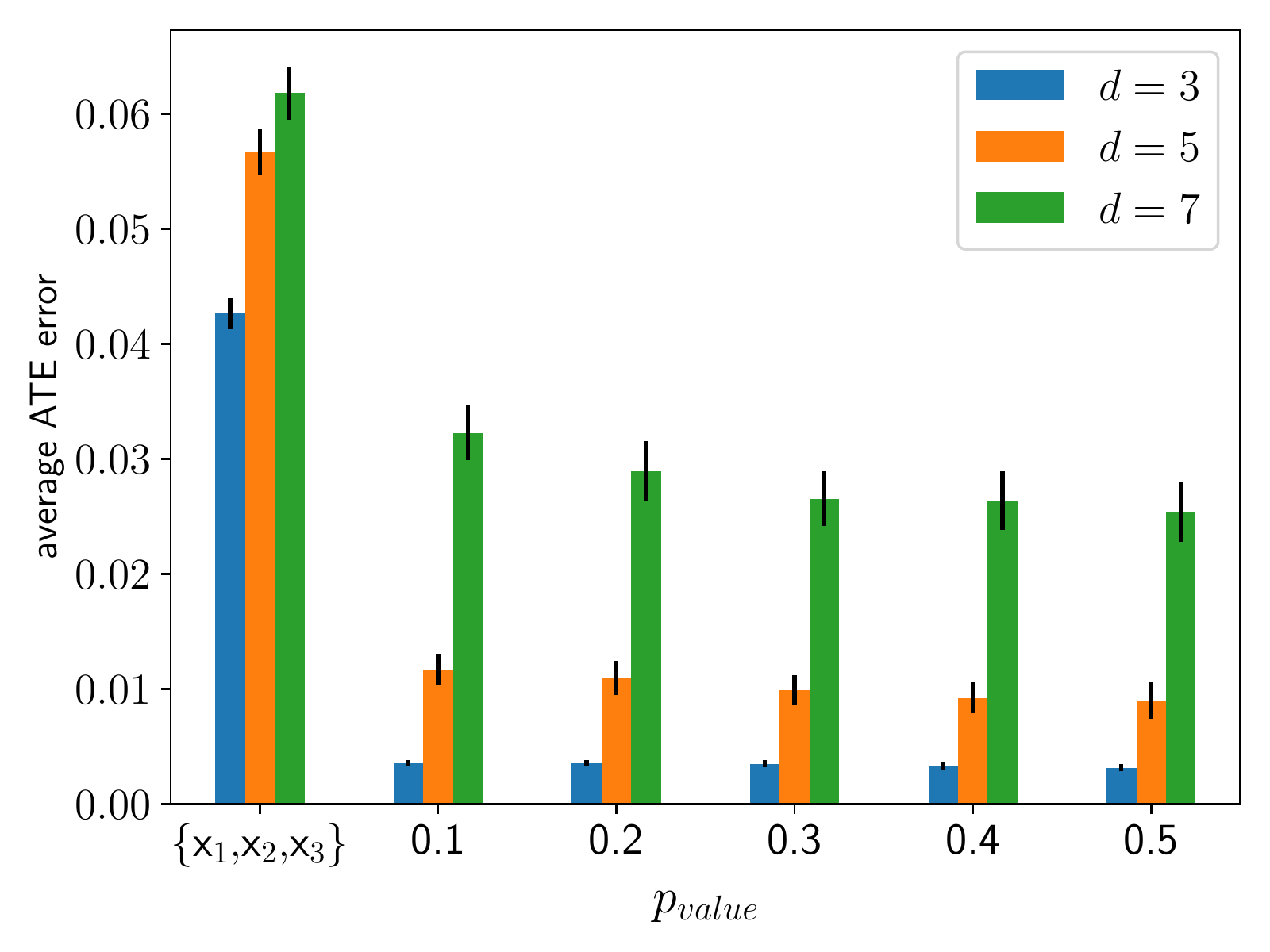}
  \caption{Performance of Algorithm \ref{alg:subset_search} on $\cG^{toy}$.}
  \label{fig:app_toy4_expts}
\end{subfigure}
\qquad
\begin{subfigure}{.5\textwidth}
  \centering
  \includegraphics[width=0.9\textwidth]{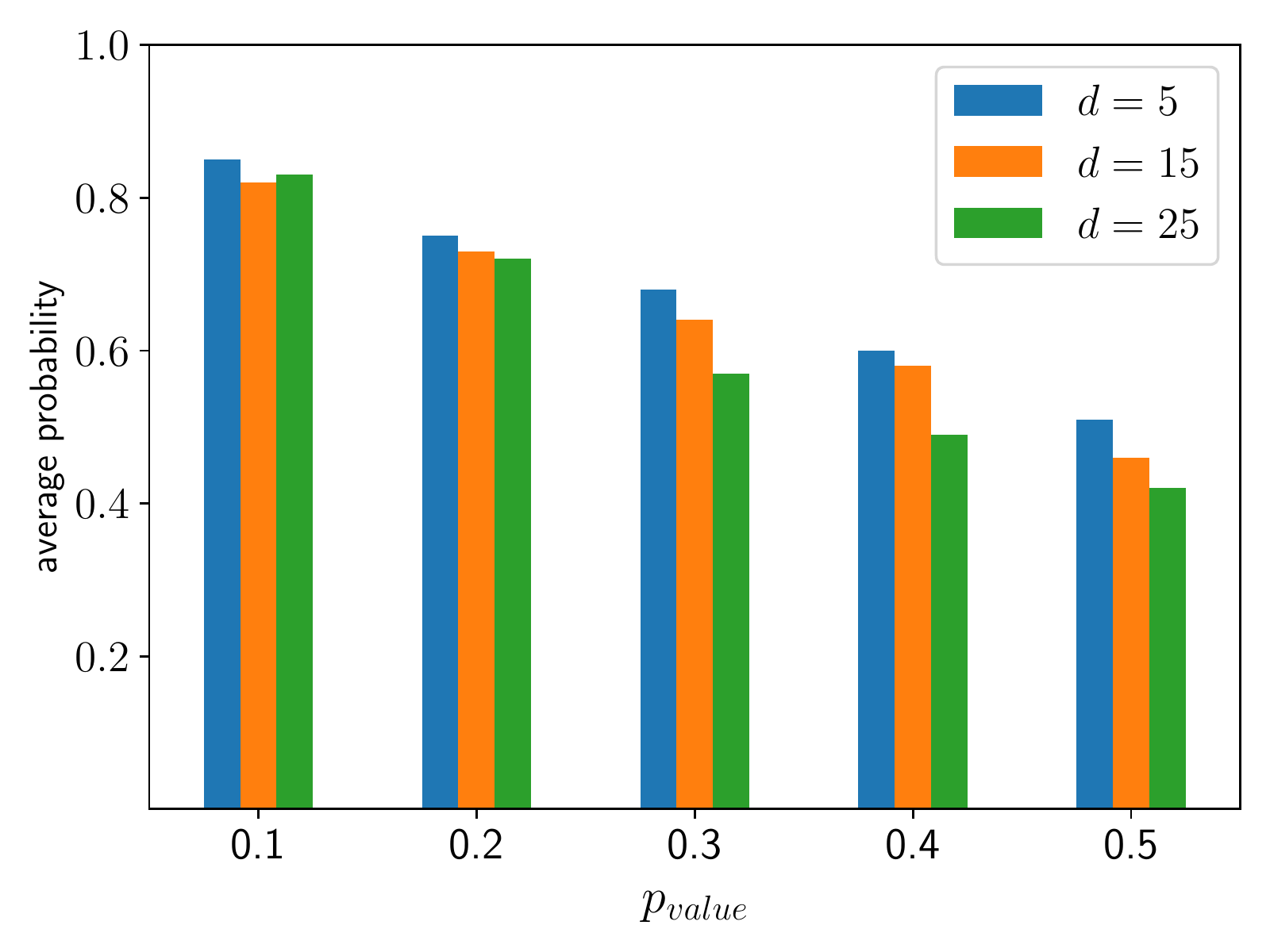}
  \caption{Success probability of the set $\{\rvx_2\}$ in the \\ toy
  example $\cG^{toy}$ for different $p_{value}$ thresholds.}
  \label{fig:app_syn1_sp}
\end{subfigure}%
\begin{subfigure}{.5\textwidth}
  \centering
  \includegraphics[width=0.9\textwidth]{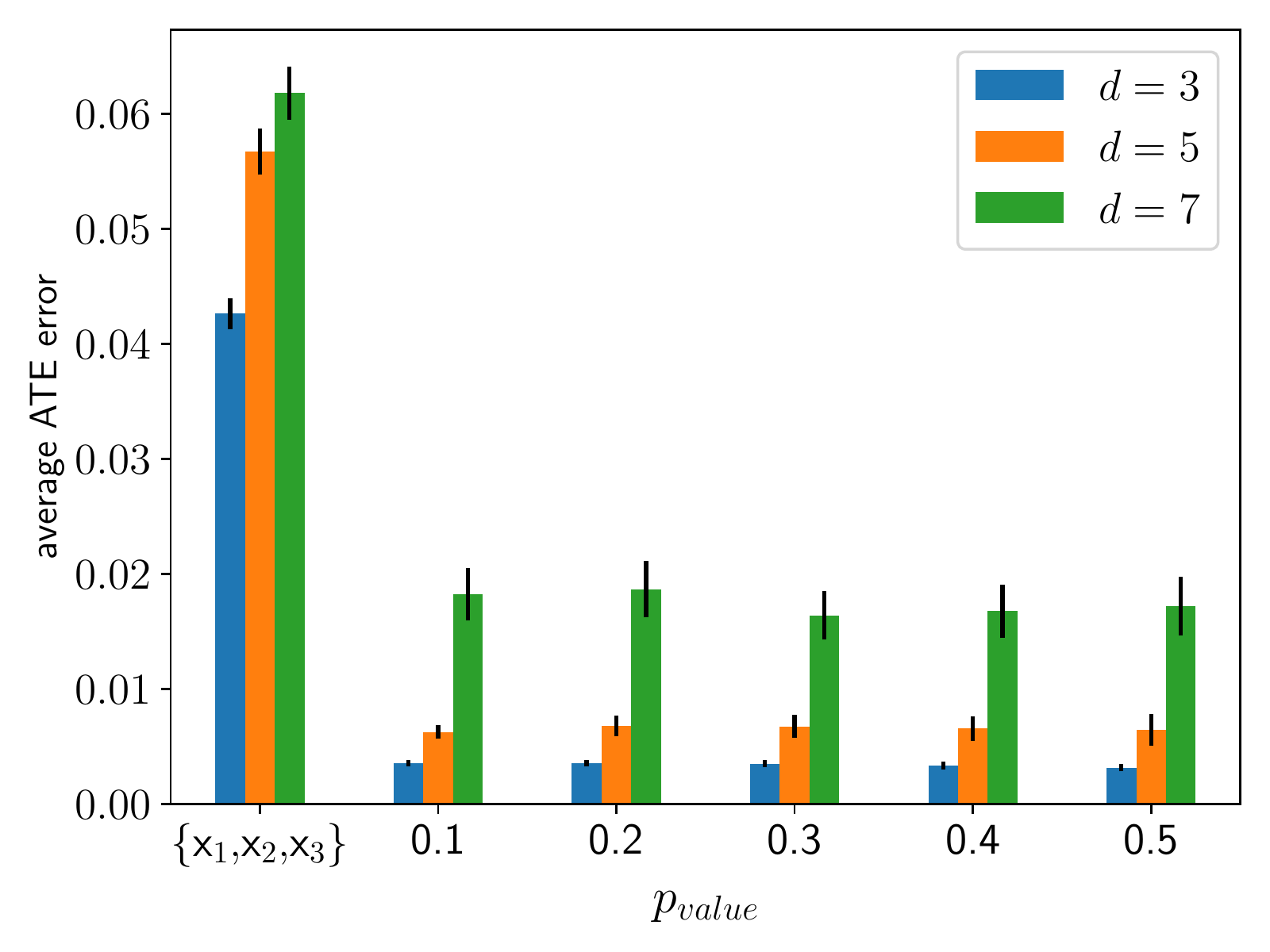}
  \caption{Performance of Algorithm \ref{alg:subset_search} on $\cG^{toy}$ when \\
  the candidate adjustment sets are $\td$-dimensional.}
  \label{fig:app_toy4_sparse}
\end{subfigure}
\caption{Validating our theoretical results and our Algorithm \ref{alg:subset_search} on $\cG^{toy}$ when $\rve = f(\rvx_{\rvt})$.}
\label{fig:app_toy_expts}
\end{figure}

\subsection{IHDP}
\label{appendix:ihdp}
In this section, we provide more details on experiments in Section \ref{subsection:ihdp} on the IHDP\footnote{\url{https://github.com/vdorie/npci/blob/master/examples/ihdp_sim/data/ihdp.RData}} dataset.\\

\noindent{\bf Dataset Description.}
First, we describe various aspects measured by the features available in this dataset.
The feature set comprises of the following attributes (a) 1-dimensional: child's birth-weight, child’s head circumference at birth, number of weeks pre-term that the child was born, birth order, neo-natal health index, mother’s age when she gave birth to the child, child’s gender, indicator for whether the child was a twin, indicator for whether the mother was married when the child born, indicator for whether the child was first born, indicator for whether the mother smoked cigarettes when she was pregnant, indicator for whether the mother consumed alcohol when she was pregnant, indicator for whether the mother used drugs when she was pregnant, indicator for whether the mother worked during her pregnancy, indicator for whether the mom received any prenatal care, (b) 3-dimensional: education level of the mother at the time the child was born, and (c) 7 -dimensional: site indicator.\\ 

\noindent{\bf The set of all observed features satisfies the back-door criterion for IHDP.}
As described in Section \ref{subsection:ihdp}, the outcome simulated by the setting ``A'' of the NPCI package depends on all the observed features. In other words, there is a direct edge from each of the observed feature to the outcome $\rvy$ in this scenario. Also, recall from Section \ref{subsection:ihdp} that the feature set is pre-treatment (i.e., it satisfies Assumption \ref{assumption1}). Therefore, from Remark \ref{remark1}, $\rvbz \subseteq \rvbx$ satisfies the back-door criterion relative to $(\rvt, \rvy)$ in $\cG$ if and only if $\rvt$ and $\rvy$ are d-separated by $\rvbz$ in $\cG_{-\rvt}$. Here, when $\rvbz$ is the set of all observed features, it is easy to see that $\rvt$ and $\rvy$ are d-separated by $\rvbz$ in $\cG_{-\rvt}$. Therefore, the set of all observed features satisfies the back-door criterion.\\

\noindent{\bf Choices of features in $\rvbx^{(o)}$.}
As mentioned in Section \ref{subsection:ihdp}, we keep the feature {child’s birth-weight} in $\rvbx^{(o)}$. In addition to these, we also keep {the number of weeks pre-term that the child}, {child’s head circumference at birth}, {birth order}, {neo-natal health index}, {mother’s age when she gave birth to the child} , { child’s gender}, { indicator for whether the mother used drugs when she was pregnant}, {indicator for whether the mom received any prenatal care}, and  {site indicator} in $\rvbx^{(o)}$.\\

\noindent{\bf Existence of valid adjustment sets of size 5.} 
Since $\rvbx^{(o)}$ comprises of only 10 different features, the set of all subsets of $\rvbx^{(o)} \setminus \{\rvx_{\rvt}\}$ comprises of 512 elements for any $\rvx_{\rvt}$. Therefore, in principle, one could find the set with lowest ATE error amongst these 512 candidate adjustment sets instead of the averaging performed by our algorithm (Algorithm \ref{alg:subset_search}). In an attempt to do this for comparison with our algorithm, we accidentally came across the following subset of features : $\rvbx^{(m)}$ =  $\{${child’s head circumference at birth}, {birth order}, {indicator for whether the mother used drugs when she was pregnant}, {indicator for whether the mom received any prenatal care}, {site indicator}$\}$. The ATE estimated using $\rvbx^{(m)}$ to adjust (termed as `the oracle') significantly outperforms the ATE estimated using $\rvbx^{(o)}$ to adjust (termed as `the baseline' i.e., \texttt{Baseline}) as shown in Figure \ref{fig:ihdp_magic}. 

\begin{figure}[h]
\centering
\begin{subfigure}{.5\textwidth}
  \centering
  \includegraphics[width=0.9\textwidth]{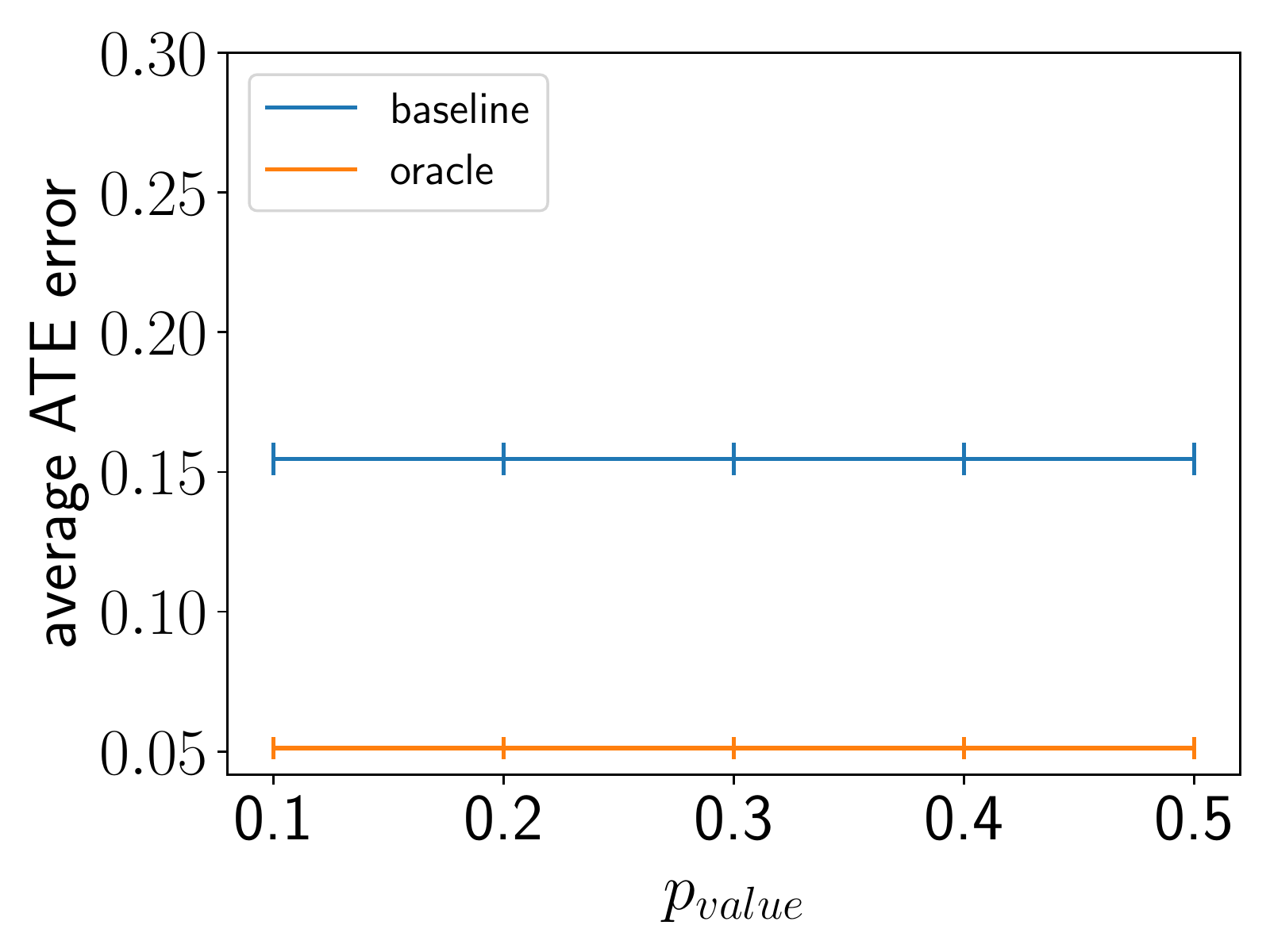}
  \caption{Comparison of \texttt{Baseline} (i.e., adjusting \\for $\rvbx^{(o)}$) with  oracle (i.e., adjusting for $\rvbx^{(m)}$) \\on IHDP}
  \label{fig:ihdp_magic}
\end{subfigure}%
\begin{subfigure}{.5\textwidth}
  \centering
  \includegraphics[width=0.9\textwidth]{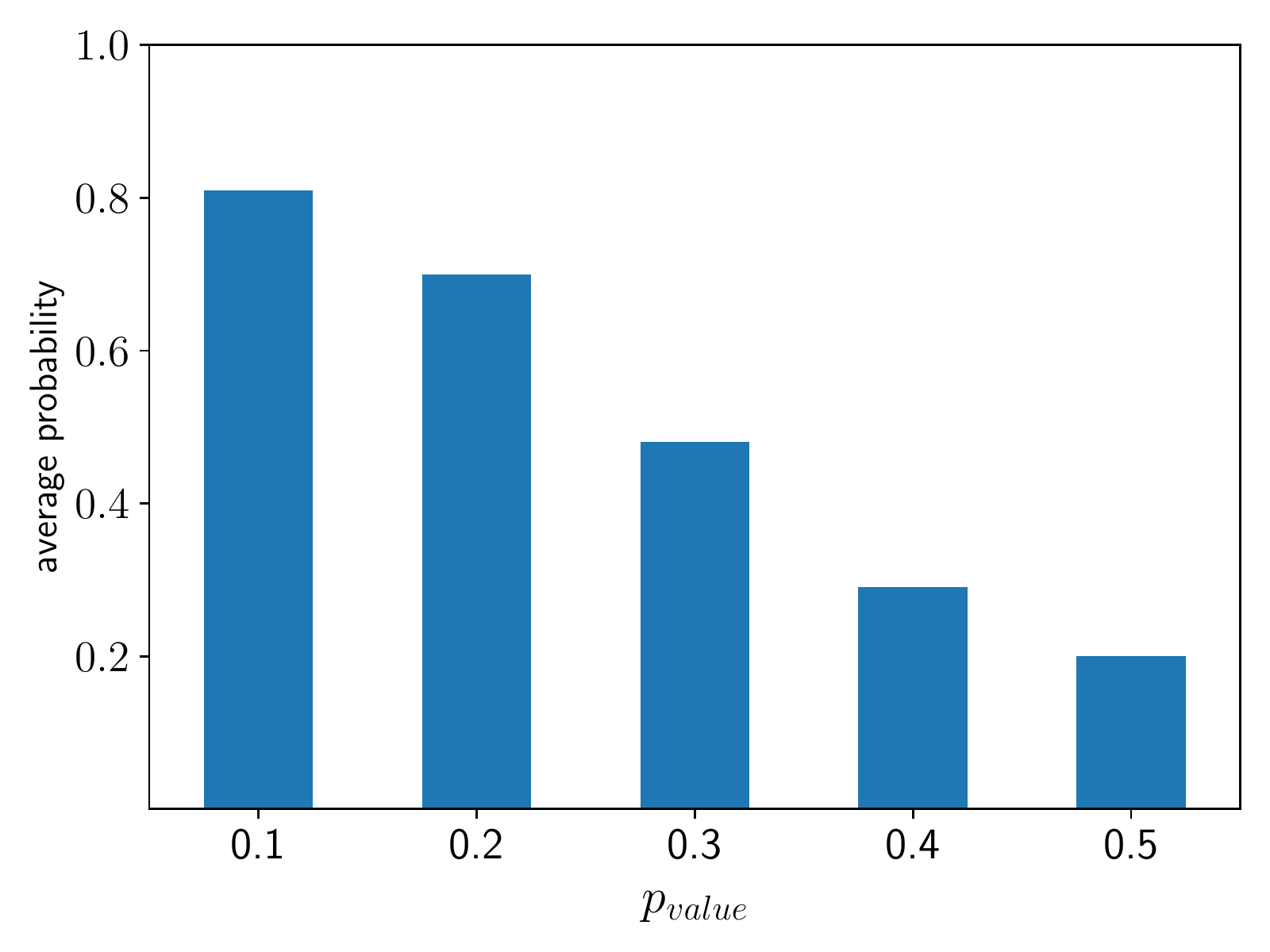}
  \caption{Success probability of the set $\rvbx^{(m)}$ in IHDP\\ for different $p_{value}$ thresholds.}
  \label{fig:ihdp_sp}
\end{subfigure}
\caption{Additional analysis on IHDP.}
\label{fig:ihdp_more_plots}
\end{figure}

Therefore, we believe that there exist valid adjustment sets of size 5 (or adjustment sets better than $\rvbx^{(o)}$) for this dataset. Therefore, to curtail the run-time of \texttt{Exhaustive}, we consider \texttt{Sparse} with $\cX =$ subsets of $\rvbx^{(o)} \setminus \{\rvx_{\rvt}\}$ with size at-most 5 in Section \ref{subsection:ihdp}. However, as mentioned in Section \ref{subsection:ihdp}, the performance of \texttt{Sparse} is similar to that of \texttt{Exhaustive} since (a) \texttt{Sparse} has to perform 382 tests and (b) there is no guarantee that $\rvbx^{(m)}$ will be picked as a valid adjustment set (as explained below). Finally, we point out that the performance of \texttt{IRM-t} is closest to `the oracle' as evident from Figure \ref{fig:ihdp}.\\

\noindent{\bf Success Probability.} Similar to Section \ref{appendix:synthetic_experiments}, we consider the success probability of the set $\rvbx^{(m)}$.
For a given $p_{value}$ threshold, we let the success probability of the set $\rvbx^{(m)}$ be the fraction of times (in $n_r$ runs) the p-value of CI$(\rve \indep \rvy |\rvbx^{(m)}, \rvt )$ is more than $p_{value}$.
In Figure \ref{fig:ihdp_sp}, we show how the success probability of the set $\rvbx^{(m)}$ varies with different $p_{value}$ thresholds i.e., $\{0.1,0.2,0.3,0.4,0.5\}$ for IHDP.
% As we can see in Figure \ref{fig:ihdp_sp}, the success probability of the set $\{\rvx_2\}$, for the same $p_{value}$ threshold, is much lower in high dimensions compared to low dimensions. We believe this happens (a) because of the non-ideal CI tester and (b) because the number of samples are finite. 

\subsection{Cattaneo}
\label{appendix:cattaneo}
In this section, we provide more details on experiments in Section \ref{subsection:cattaneo} on the Cattaneo\footnote{\url{www.stata-press.com/data/r13/cattaneo2.dta}} dataset.\\

\noindent{\bf Dataset Description.}
We describe various aspects measured by the features available in this dataset. The feature set comprises of the following attributes : mother’s marital status, indicator for whether the mother consumed alcohol when she was pregnant, indicator for whether the mother had any previous infant where the newborn died, mother’s age, mother’s education, mother’s race, father’s age, father’s education, father’s race, months since last birth by the mother, birth month, indicator for whether the baby is first-born, total number of prenatal care visits, number of prenatal care visits in the first trimester, and the number of trimesters the mother received any prenatal care.
% indicator for whether the mother is hispanic, indicator for whether the father is hispanic, indicator for whether the mother is foreign, order,
Apart from these, there are also a few other features available in this dataset for which we did not have access to their description.

\subsection{Training details}
\label{appendix:training_details}
For all of our experiments, we split the data randomly into train data and test data in the ratio $0.8:0.2$. We use ridge regression with cross-validation and regularization strengths: $0.001, 0.01, 0.1, 1$ as the regression model.  We mainly relied on the following github repositories --- (a) causallib\footnote{\url{https://github.com/ibm/causallib}} \citep{causalevaluations}, (b) RCoT \citep{strobl2019approximate}, (c) ridgeCV\footnote{\url{https://github.com/scikit-learn/scikit-learn/tree/15a949460/sklearn/linear_model/_ridge.py}}, and (d) IRM\footnote{\url{https://github.com/facebookresearch/InvariantRiskMinimization}}.

For IRM, we use 15000 iterations. We train the IRM framework using 2 environments and perform validation on the remaining environment. For validation, we vary the learning rate (of the Adam optimizer that IRM uses) between $0.01$ and $0.001$ and vary the IRM regularizer between $0.1$ and $0.001$. During training, we use a step learning rate scheduler which decays the initial learning rate by half after every 5000 iterations. 

% \subsection{Existing assets}
% \label{appendix:existing_assets}
% In this work, we mainly relied on the following github repositories --- (a) causallib\footnote[6]{\url{https://github.com/ibm/causallib}} \citep{causalevaluations}, which is under the Apache license version 2.0, (b) RCoT \citep{strobl2019approximate}, which is under the CC BY-NC 4.0 license, and (c) ridgeCV,\footnote[7]{\url{https://github.com/scikit-learn/scikit-learn/tree/15a949460/sklearn/linear_model/_ridge.py}} which is under the BSD 3-clause license.  We did not modify any of the code under licenses; we only installed these repositories as packages. In addition to these, we used two public datasets  (a) IHDP dataset\footnote[8]{\url{https://github.com/vdorie/npci/blob/master/examples/ihdp_sim/data/ihdp.RData}} and (b) Cattaneo dataset.\footnote[9]{\url{www.stata-press.com/data/r13/cattaneo2.dta}} These datasets are commonly used benchmark datasets for treatment effect estimation, which is why we chose them for our comparisons.

\subsection{Comparison with \cite{entner2013data} and \cite{gultchin2020differentiable}}
\label{app:entner}
As described in Section \ref{section:main_results},  \cite{entner2013data} and \cite{gultchin2020differentiable} cannot be used to conclude that $\emptyset, \{\rvx_3\}$, and $\{\rvx_2,\rvx_3\}$ are not admissible i.e., not valid backdoors in $\cG^{toy}$ (because the variable $\rvx_t=\rvx_1$ has an unobserved parent) while our Theorem \ref{thm_necessity} can be used to conclude that. Here, we provide the p-values (averaged over 100 runs) corresponding to these in Table \ref{p-val-e}. As we can see, our invariance test results in a very small p-value for $\emptyset, \{\rvx_3\}$, and $\{\rvx_2,\rvx_3\}$ leading to the conclusion that they are not valid backdoors in $\cG^{toy}$.

\begin{table}[h]
\caption{p-value of CI$(\rve \indep \rvy |\rvbz, \rvt )$ for $\rvbz = \emptyset, \{\rvx_3\},$ or $\{\rvx_2,\rvx_3\}$ in $\cG^{toy}.$}
  \label{p-val-e}
  \centering
	\begin{tabular}{p{16mm}p{40mm}p{40mm}p{40mm}}
	\\
    	\toprule
    % 	\cmidrule(r){1-2}
    	$\rvbz$ & \textbf{d = 3} & \textbf{d = 5}  & \textbf{d = 7} \\
    	\midrule
    	$\emptyset$ & $1.3\times 10^{-15}$ $\pm$ $2.9\times 10^{-16}$ & $1.1\times 10^{-15}$ $\pm$ $1.5\times 10^{-16}$ & $1.9\times 10^{-15}$ $\pm$ $6.7\times 10^{-16}$ \\
    	\midrule
    	$\{\rvx_3\}$ & $1.4\times 10^{-15}$ $\pm$ $2.7\times 10^{-16}$ & $1.2\times 10^{-15}$ $\pm$ $3.1\times 10^{-16}$ & $1.0\times 10^{-15}$ $\pm$ $2.7\times 10^{-16}$ \\
    	\midrule
    	$\{\rvx_2,\rvx_3\}$ & $1.8\times 10^{-4}$ $\pm$ $1.8\times 10^{-4}$ & $5.1\times 10^{-3}$ $\pm$ $3.6\times 10^{-3}$ & $1.9\times 10^{-4}$ $\pm$ $1.3\times 10^{-4}$ \\
    	\bottomrule
  \end{tabular}
\end{table}

% [[1.27675648e-15 1.11133325e-15 1.92401650e-15]]
% [[1.38999923e-15 1.22457600e-15 1.04471987e-15]]
% [[0.00018278 0.00513678 0.00018338]]
% [[2.88258457e-16 1.52343434e-16 6.65217352e-16]]
% [[2.69249021e-16 3.10249058e-16 2.72265509e-16]]
% [[0.00017501 0.00357416 0.00012949]]

\subsection{Comparison with \cite{cheng2020towards}}
\label{app:cheng}
As described in Section \ref{section:main_results},  \cite{cheng2020towards} cannot be used to conclude that $\emptyset, \{\rvx_2\}, \{\rvx_3\}$, and $\{\rvx_2,\rvx_3\}$ are not admissible i.e., not valid backdoors in the DAG obtained by adding the edge $\rvx_1 \rightarrow \rvy $ to $\cG^{toy}$ (because there is no COSO variable) while our Theorem \ref{thm_necessity} can be used to conclude that. Here, we provide the p-values (averaged over 100 runs) corresponding to these in Table \ref{p-val-c}. As we can see, our invariance test results in a very small p-value for $\emptyset, \{\rvx_2\}, \{\rvx_3\}$, and $\{\rvx_2,\rvx_3\}$ leading to the conclusion that they are not valid backdoors.

\begin{table}[h]
\caption{p-value of CI$(\rve \indep \rvy |\rvbz, \rvt )$ for $\rvbz = \emptyset, \{\rvx_2\}, \{\rvx_3\},$ or $\{\rvx_2,\rvx_3\}$ in the DAG obtained by adding the edge $\rvx_1 \rightarrow \rvy $ to $\cG^{toy}$.}
  \label{p-val-c}
  \centering
	\begin{tabular}{p{16mm}p{40mm}p{40mm}p{40mm}}
	\\
    	\toprule
    % 	\cmidrule(r){1-2}
    	$\rvbz$ & \textbf{d = 3} & \textbf{d = 5}  & \textbf{d = 7} \\
    	\midrule
    	$\emptyset$ & $1.2\times 10^{-15}$ $\pm$ $2.0\times 10^{-16}$ & $1.6\times 10^{-15}$ $\pm$ $3.2\times 10^{-16}$ & $1.3\times 10^{-15}$ $\pm$ $2.1\times 10^{-16}$ \\
    	\midrule
    	$\{\rvx_2\}$ & $8.8\times 10^{-7}$ $\pm$ $8.8\times 10^{-7}$ & $5.3\times 10^{-4}$ $\pm$ $4.2\times 10^{-4}$ & $8.1\times 10^{-3}$ $\pm$ $4.8\times 10^{-3}$ \\
    	\midrule
    	$\{\rvx_3\}$ & $1.6\times 10^{-15}$ $\pm$ $3.2\times 10^{-16}$ & $9.7\times 10^{-16}$ $\pm$ $1.6\times 10^{-16}$ & $2.0\times 10^{-15}$ $\pm$ $3.8\times 10^{-16}$ \\
    	\midrule
    	$\{\rvx_2,\rvx_3\}$ & $3.5\times 10^{-7}$ $\pm$ $3.4\times 10^{-7}$ & $1.1\times 10^{-3}$ $\pm$ $7.8\times 10^{-4}$ & $1.4\times 10^{-3}$ $\pm$ $1.3\times 10^{-3}$ \\
    	\bottomrule
  \end{tabular}
\end{table}

% [[1.17905685e-15 1.63646874e-15 1.31561428e-15]]
% [[8.82893113e-07 5.34171820e-04 8.06973815e-03]]
% [[1.56874513e-15 9.74775816e-16 1.97175609e-15]]
% [[3.45560359e-07 1.13876681e-03 1.37753140e-03]]

% [[2.02650101e-16 3.21369277e-16 2.09249062e-16]]
% [[8.78467555e-07 4.24119066e-04 4.82462997e-03]]
% [[3.20570645e-16 1.57859348e-16 3.79303663e-16]]
% [[3.43828212e-07 7.82550188e-04 1.34108072e-03]]

\end{document}